\documentclass[11pt]{article}
\usepackage{setspace}
\usepackage{cprotect}
\usepackage{amsmath,amssymb,amsthm}
\usepackage[noend]{algorithmic}
\usepackage[ruled,vlined]{algorithm2e}
\usepackage{hyperref}
\usepackage{fullpage}
\usepackage{makeidx}
\usepackage{enumerate}
\usepackage{bbm}
\usepackage{graphicx,float,psfrag,epsfig}
\usepackage{epstopdf}
\usepackage{color}
\usepackage{enumitem}
\usepackage{caption}
\usepackage{bigints}
\usepackage{mathtools}
\usepackage[mathscr]{euscript}
\usepackage{subfigure}

\DeclareSymbolFont{rsfs}{U}{rsfs}{m}{n}
\DeclareSymbolFontAlphabet{\mathscrsfs}{rsfs}

\numberwithin{equation}{section}

\newtheoremstyle{myexample} % name
    {\topsep}                    % Space above
    {\topsep}                    % Space below
    {\rm }                   % Body font
    {}                           % Indent amount
    {\bf }                   % Theorem head font
    {.}                          % Punctuation after theorem head
    {.5em}                       % Space after theorem head
    {}  % Theorem head spec (can be left empty, meaning normal)

\newtheoremstyle{myremark} % name
    {\topsep}                    % Space above
    {\topsep}                    % Space below
    {\rm}                        % Body font
    {}                           % Indent amount
    {\bf}                        % Theorem head font
    {.}                          % Punctuation after theorem head
    {.5em}                       % Space after theorem head
    {}  % Theorem head spec (can be left empty, meaning normal)

\newtheorem{claim}{Claim}[section]
\newtheorem{lemma}[claim]{Lemma}

\newtheorem{theorem}{Theorem}
\newtheorem{proposition}[claim]{Proposition}
\newtheorem{corollary}[claim]{Corollary}
\newtheorem{definition}[claim]{Definition}

\theoremstyle{myremark}
\newtheorem{remark}{Remark}[section]

\theoremstyle{myremark}

\theoremstyle{myexample}

\def\hbx{\hat{\boldsymbol x}}

\def\sign{{\rm sign}}

\def\sT{{\sf T}}
\def\<{\langle}
\def\>{\rangle}

\def\prob{{\mathbb P}}

\def\E{{\mathbb E}} %expectation

\newcommand\norm[1]{\left\lVert{#1}\right\rVert}
\def\bs{\boldsymbol{s}}
\def\de{{\rm d}}

\def\reals{\mathbb{R}}

\def\normal{{\sf N}}

\def\sb{{\sf b}}

\def\cS{{\mathcal{S}}}
\def\cT{{\mathcal{T}}}

\def\bM{{\boldsymbol M}}

\def\by{{\boldsymbol y}}

\def\bQ{{\boldsymbol Q}}

\def\bD{{\boldsymbol D}}

\def\bI{{\boldsymbol I}}

\def\be{{\boldsymbol e}}
\def\br{{\boldsymbol r}}

\def\bx{{\boldsymbol x}}
\def\b0{{\boldsymbol 0}}

\def\ba{{\boldsymbol a}}

\def\bphi{{\boldsymbol \varphi}}

\def\bX{{\boldsymbol X}}
\def\bY{{\boldsymbol Y}}
\def\bZ{{\boldsymbol Z}}

\def\bv{{\boldsymbol v}}

\def\bu{{\boldsymbol u}}

\def\hbx{\hat{\boldsymbol x}}

\def\tbx{\tilde{\boldsymbol x}}
\def\bxs{\hat{\boldsymbol x}^{\rm s}}

\def\id{{\boldsymbol I}}

\def\bfone{{\boldsymbol 1}}

\def\bR{{\boldsymbol R}}

\def\bA{{\boldsymbol A}}

\def\U{{\rm U}}
\def\U1{{\rm U}(1)}

\def\tow2{\stackrel{W_2}{\Rightarrow}}
\def\PL{{\rm PL}}

\def\sfp{{\sf p}}

\newcommand{\towa}[1]{\stackrel{W_{#1}}{\Rightarrow}}

\newcommand\abs[1]{\left\lvert{#1}\right\rvert}

\def\sc{{\sf c}}

\newcommand{\beq}{\begin{equation}}
\newcommand{\eeq}{\end{equation}}
\newcommand{\diag}{{\rm diag}}
 
\def \bLambda{{\boldsymbol \Lambda}}
\newcommand{\FP}{{\sf FP}}

\title{Optimal Combination of Linear and Spectral Estimators for Generalized Linear Models}

\author{Marco Mondelli\thanks{Institute of Science and Technology (IST) Austria. Email: \texttt{marco.mondelli@ist.ac.at}.},\;\;\;
Christos Thrampoulidis\thanks{Department of Electrical and Computer Engineering, University of British Columbia (UBC). Email: \texttt{cthrampo@ece.ubc.ca}.} \;\;\;and\;\;\;Ramji Venkataramanan\thanks{Department of Engineering, University of Cambridge. Email: \texttt{ramji.v@eng.cam.ac.uk}.}}

\def\x{\mathbf{x}}
\newcommand{\bxl}{\hat{\boldsymbol x}^{\rm L}}
\newcommand{\bxln}{\bar{\boldsymbol x}^{\rm L}}

\newcommand{\bzl}{{\boldsymbol z}^{\rm L}}

\newcommand{\bxc}{\hat{\boldsymbol x}^{\rm c}}

\def\R{\mathbb{R}}

\DeclarePairedDelimiterX{\inp}[2]{\langle}{\rangle}{#1, #2}

\def\ras{\stackrel{\mathclap{\mbox{\footnotesize a.s.}}}{\longrightarrow}}

\newcommand{\simiid}{\stackrel{\text{iid}}{\sim}}
\def\Nn{\mathcal{N}}
\newcommand{\nn}{\notag}
\newcommand{\eqd}{\stackrel{\mathclap{\text{d}}}{=}}

\def\bh{\mathbf{h}}

\newcommand{\vp}{\vspace{3pt}}

\newcommand{\Tc}{\mathcal{T}}
\newcommand{\Tct}{\widetilde{\Tc}}

\begin{document}

\maketitle

\begin{abstract}
    We study the problem of recovering an unknown signal $\bx$ given measurements obtained from a generalized linear model with a Gaussian sensing matrix. Two popular solutions are based on a linear estimator $\bxl$ and a spectral estimator $\bxs$. The former is a data-dependent linear combination of the columns of the measurement matrix, and its analysis is quite simple. The latter is the principal eigenvector of a data-dependent matrix, and a recent line of work has studied its performance. In this paper, we show how to optimally combine $\bxl$ and $\bxs$. At the heart of our analysis is the exact characterization of the empirical joint distribution of $(\bx, \bxl, \bxs)$ in the high-dimensional limit. This allows us to compute the Bayes-optimal combination of $\bxl$ and $\bxs$, given the limiting distribution of the signal $\bx$. When the distribution of the signal is Gaussian, then the Bayes-optimal combination has the form $\theta\bxl+\bxs$ and we derive the optimal combination coefficient. In order to establish the limiting distribution of $(\bx, \bxl, \bxs)$, we
    design and analyze an Approximate Message Passing (AMP) algorithm whose iterates give $\bxl$ and approach $\bxs$. Numerical simulations demonstrate the improvement of the proposed combination with respect to the two methods considered separately.
\end{abstract}

\section{Introduction}

%GLMs
In a generalized linear model (GLM) \cite{nelder1972generalized, mccullagh2018generalized}, we want to recover a $d$-dimensional signal $\bx \in\mathbb R^d$ given $n$ i.i.d. measurements $\by=(y_1, \ldots, y_n)$ of the form:
\begin{equation}\label{eq:model}
y_i\sim p(y\mid \langle\bx, \ba_i \rangle), \qquad i\in\{1, \ldots, n\},
\end{equation}
where $\langle\cdot, \cdot\rangle$ denotes the scalar product, $\{\ba_i\}_{1\le i\le n}$ are known sensing vectors, and the (stochastic) output function $p(\cdot \mid \langle\bx, \ba_i \rangle)$ is a known probability distribution. GLMs arise in several problems in statistical inference and signal processing. Examples include photon-limited imaging \cite{photonUnser, photonVetterli}, compressed sensing \cite{eldar2012compressed}, signal recovery from quantized measurements \cite{noiseRangan, boufounos20081}, phase retrieval \cite{phFienup, shechtman2015phase}, and neural networks with one hidden layer \cite{lecun2015deep}.

The problem of estimating $\bx$ from $\by$ is, in general, non-convex, and semi-definite programming relaxations have been proposed \cite{candes2015phase,  candes2013phaselift,  waldspurger2015phase,thrampoulidis2019lifting}. However, the computational complexity and memory requirement of these approaches quickly grow with the dimension $d$. For this reason, several non-convex approaches have been developed, e.g., alternating minimization \cite{netrapalli2013phase}, 
approximate message passing (AMP) \cite{DMM09, RanganGAMP, schniter2014compressive}, Wirtinger Flow \cite{candes2015wirt}, Kaczmarz methods \cite{wei2015solving}, and iterative convex-programming relaxations \cite{boufounos20081,bahmani2017phase,goldstein2018phasemax,dhifallah2018phase}. The Bayes-optimal estimation and generalization error have also been studied in \cite{barbier2019optimal}. When the output function $p(\cdot \mid \langle\bx, \ba_i \rangle)$ is unknown, \eqref{eq:model} is called the single-index model in the statistics literature, see e.g. \cite{Bri,li1989regression,kakade2011efficient}. The problem of recovering a structured signal (e.g., sparse, low-rank) from high-dimensional single-index measurements has been an active research topic over the past few years \cite{yi2015optimal,plan2017high,plan2016generalized,thrampoulidis2015lasso,neykov2016l1,genzel2017high,goldstein2018structured,genzel2019recovering,thrampoulidis2019lifting}.

Throughout this paper, the performance of an estimator $\hbx$ will be measured by its normalized correlation (or ``overlap") with $\bx$:
\beq 
\frac{\abs{\< \bx, \hbx \>}}{\| \bx \|_2 \| \hbx \|_2 }, 
\label{eq:est_corr_def}
\eeq
where $\| \cdot \|_2$ denotes the Euclidean norm of a vector. 

Most of the existing techniques require an initial estimate of the signal, which can then be refined via a local algorithm. Here, we focus on two popular methods: a linear estimator and a spectral estimator. The \emph{linear estimator} $\bxl$ has the form:
\begin{equation}\label{eq:deflin}
\frac{1}{n}\sum_{i=1}^n \cT_L(y_i)\ba_i \,,    
\end{equation}
where $\cT_L$ denotes a given preprocessing function. The performance analysis of this estimator is quite simple, see e.g. Proposition 1 in \cite{plan2017high} or Section \ref{sec:lin_est} of this paper. The \emph{spectral estimator} consists in the principal eigenvector $\bxs$ of a matrix of the form:
\begin{equation}\label{eq:defspect}
     \frac{1}{n}\sum_{i=1}^n \cT_s(y_i)\ba_i\ba_i^\sT \,,
\end{equation}
where $\cT_s$ is another preprocessing function. The idea of a spectral method first appeared in \cite{li1992principal} and, for the special case of phase retrieval, a series of works has provided more and more refined performance bounds \cite{netrapalli2013phase,candes2013phaselift,chen2015solving}. Recently,  an exact high-dimensional analysis of the spectral method  for generalized linear models with Gaussian sensing vectors has been carried out in \cite{lu2017phase,mondelli2017fundamental}. These works consider a regime where both $n$ and $d$ grow large at a fixed proportional rate $\delta=n/d>0$. The choice of $\cT_s$ which minimizes the value of $\delta$ (and, consequently, the amount of data) necessary to achieve a strictly positive scalar product \eqref{eq:est_corr_def} was obtained in \cite{mondelli2017fundamental}. Furthermore, the choice of $\cT_s$ which maximizes the correlation between $\bx$ and $\bxs$ for any given value of the sampling ratio $\delta$ was obtained in \cite{luo2019optimal}. The case in which the sensing vectors are obtained by picking columns from a Haar matrix is tackled in \cite{dudeja2020analysis}.

In short, the performance of the linear estimate $\bxl$ and the spectral estimate $\bxs$ is well understood, and there is no clear winner between the two. In fact, the superiority of one method over the other depends on the output function $p(\cdot \mid \langle\bx, \ba_i \rangle)$ and on the sampling ratio $\delta$. For example, for phase retrieval ($y_i=|\langle\bx, \ba_i \rangle|$), the spectral estimate provides positive correlation with the ground-truth signal as long as $\delta>1/2$ \cite{mondelli2017fundamental}, while linear estimators of the form \eqref{eq:deflin} are not effective for any $\delta>0$. On the contrary, for 1-bit compressed sensing ($y_i={\rm sign}(\langle\bx, \ba_i \rangle)$) the situation is the opposite: the spectral estimator is uncorrelated with the signal for any $\delta>0$, while the linear estimate works well. For many cases of practical interest, e.g. neural networks with ReLU activation function ($y_i=\max(\langle\bx, \ba_i \rangle, 0)$), both the linear and the spectral method give estimator with non-zero correlation. Thus, a natural question is the following:
\begin{center}
\emph{What is the optimal way to combine the linear estimator $\bxl$ and the spectral estimator $\bxs$?}
\end{center}

This paper closes the gap and answers the question above for Gaussian sensing vectors $\{\ba_i\}_{1\le i\le n}$. Our main technical contribution is to provide an exact high-dimensional characterization of the joint empirical distribution of $(\bx, \bxl, \bxs)$ in the limit $n, d \to \infty$ with a fixed sampling ratio $\delta=n/d$ (see Theorem \ref{th:W2conv2}). In particular, we prove that the conditional distribution of $(\bxl, \bxs)$ given $\bx$ converges to the law of  a bivariate Gaussian whose mean vector and covariance matrix are specified in terms of the preprocessing functions $\cT_L$ and $\cT_s$. As a consequence, we are able to compute the Bayes-optimal combination of $\bxl$ and $\bxs$ for any given prior distribution on $\bx$ (see Theorem \ref{th:optimality}). In the special case in which the signal prior is  Gaussian, the Bayes-optimal combination has the form $\theta\bxl+\bxs$, with $\theta\in\mathbb R$, and we compute the optimal combination coefficient $\theta_*$ that maximizes the normalized correlation  in  \eqref{eq:est_corr_def}
(see Corollary \ref{lem:fopt}).

The characterization of the joint empirical distribution of $(\bx, \bxl, \bxs)$ is achieved by designing and analyzing a suitable approximate message passing (AMP) algorithm. AMP is a family of iterative algorithms that has been applied to several high-dimensional statistical estimation problems including estimation in linear models \cite{DMM09, BM-MPCS-2011,BayatiMontanariLASSO, krzakala2012}, generalized linear models \cite{RanganGAMP,schniter2014compressive,sur2019modern}, and low-rank matrix estimation \cite{deshpande2014information, RanganFletcherGoyal, lesieur2017constrained,montanari2017estimation}. An appealing feature of AMP algorithms is that under suitable conditions on the model, the empirical joint distribution of the iterates can be exactly characterized in the high-dimensional limit, in  terms of a simple scalar recursion called \emph{state evolution}. 

In this paper,  we design an AMP algorithm that is equivalent to a power method computing the principal eigenvector of the matrix \eqref{eq:defspect}. Then, the state evolution analysis leads to the desired joint empirical distribution of $(\bx, \bxl, \bxs)$. Using the limiting distribution, we reduce the vector problem of estimating $\bx\in\mathbb R^d$ given two (correlated) observations $\bxl, \bxs\in\mathbb R^d$ to the scalar problem of estimating the random variable $X\in\mathbb R$ given two (correlated) observations $X_L, X_s\in \mathbb R$.  We emphasize that the focus of this work is not on using the AMP algorithm as an estimator for the generalized linear model.  Rather, we use AMP as a proof technique to characterize the joint empirical distribution of $(\bx, \bxl, \bxs)$, and thereby understand how to optimally combine the two simple estimators. 

\begin{figure}[t]
    \subfigure[]{
    \centering
    \includegraphics[scale=.4]{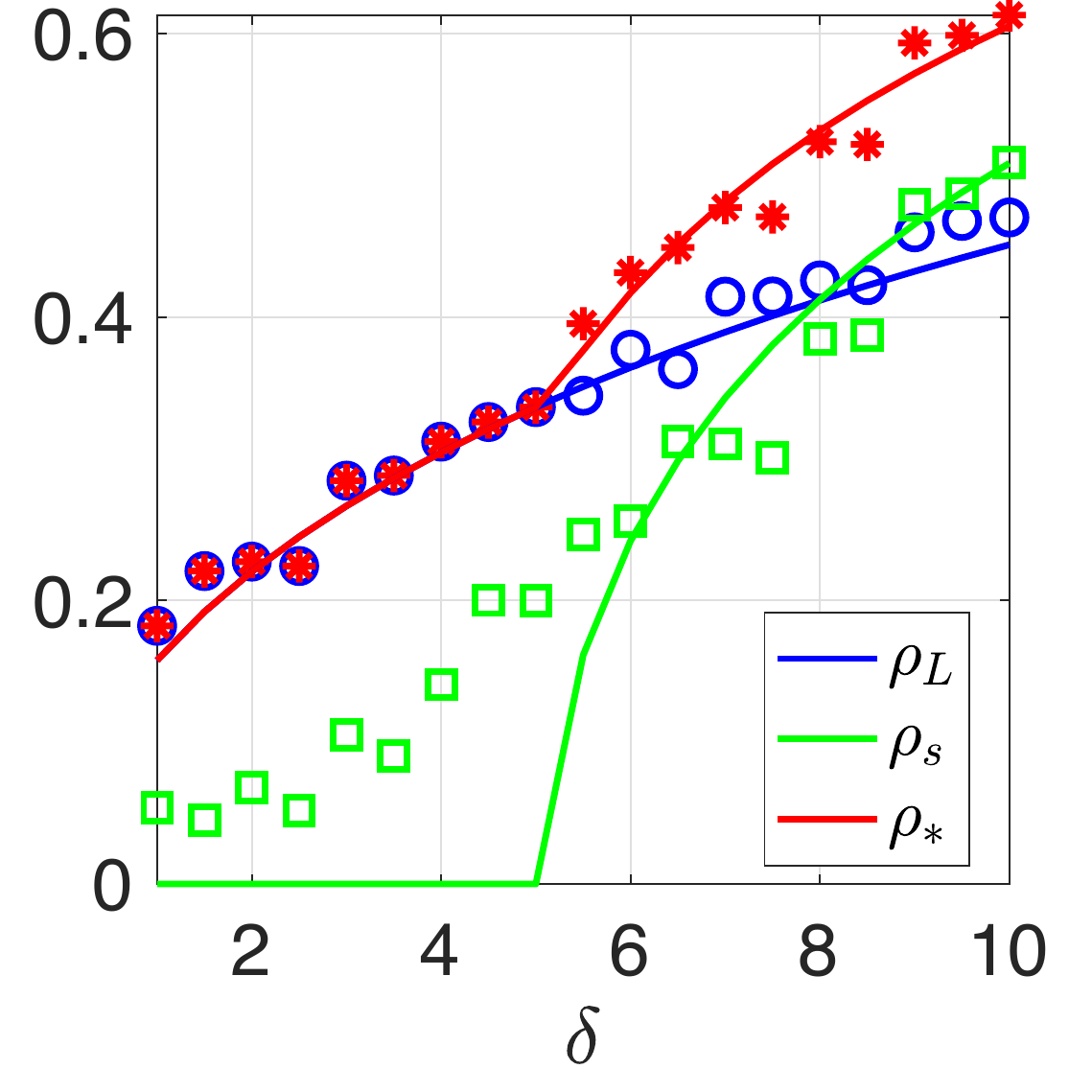}
    }
    \hspace{0.05in}
    \subfigure[]{
    \centering
    \includegraphics[scale=.4]{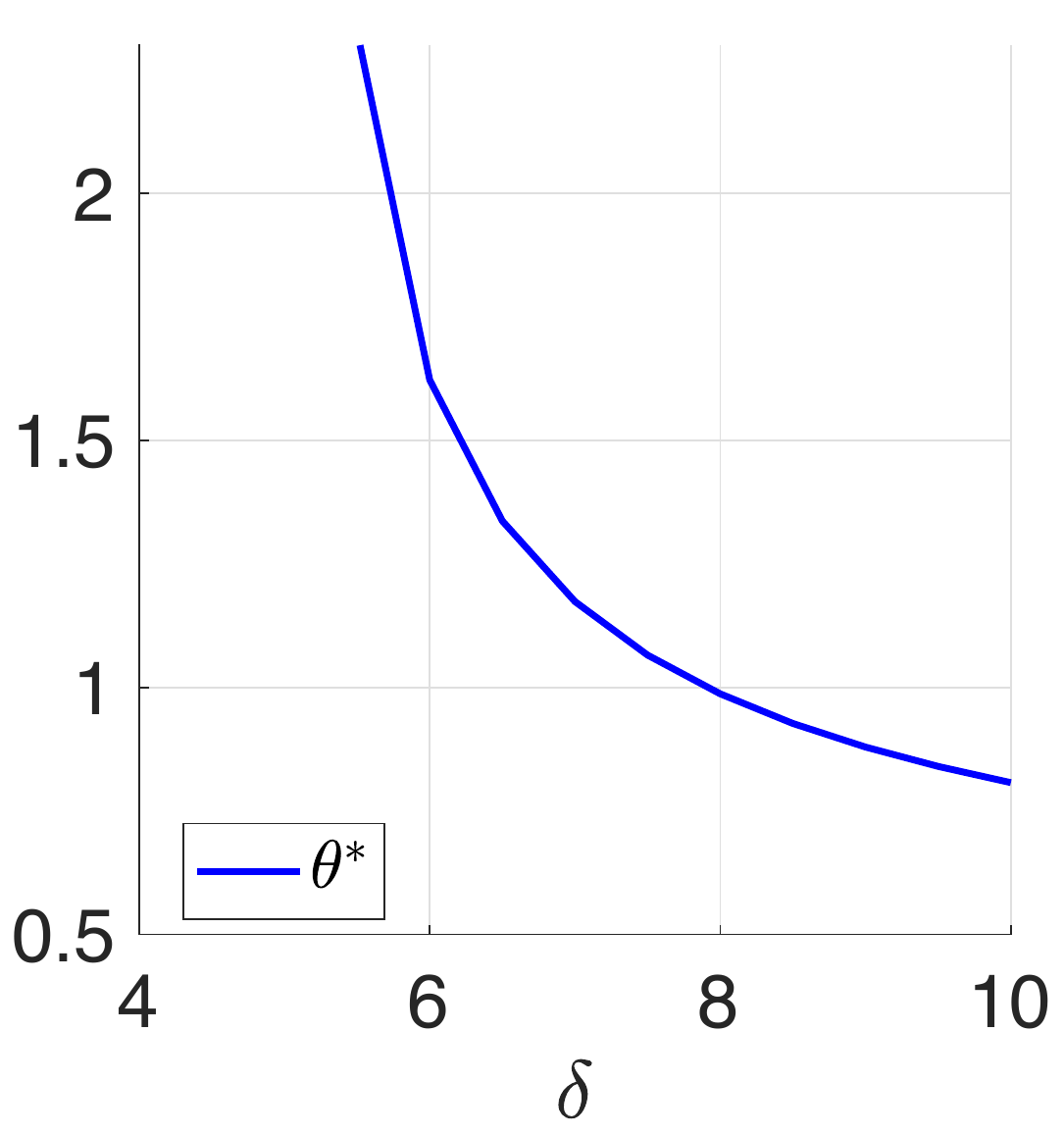}
    }
    \hspace{0.06in}
    \subfigure[]{
    \includegraphics[scale=0.34]{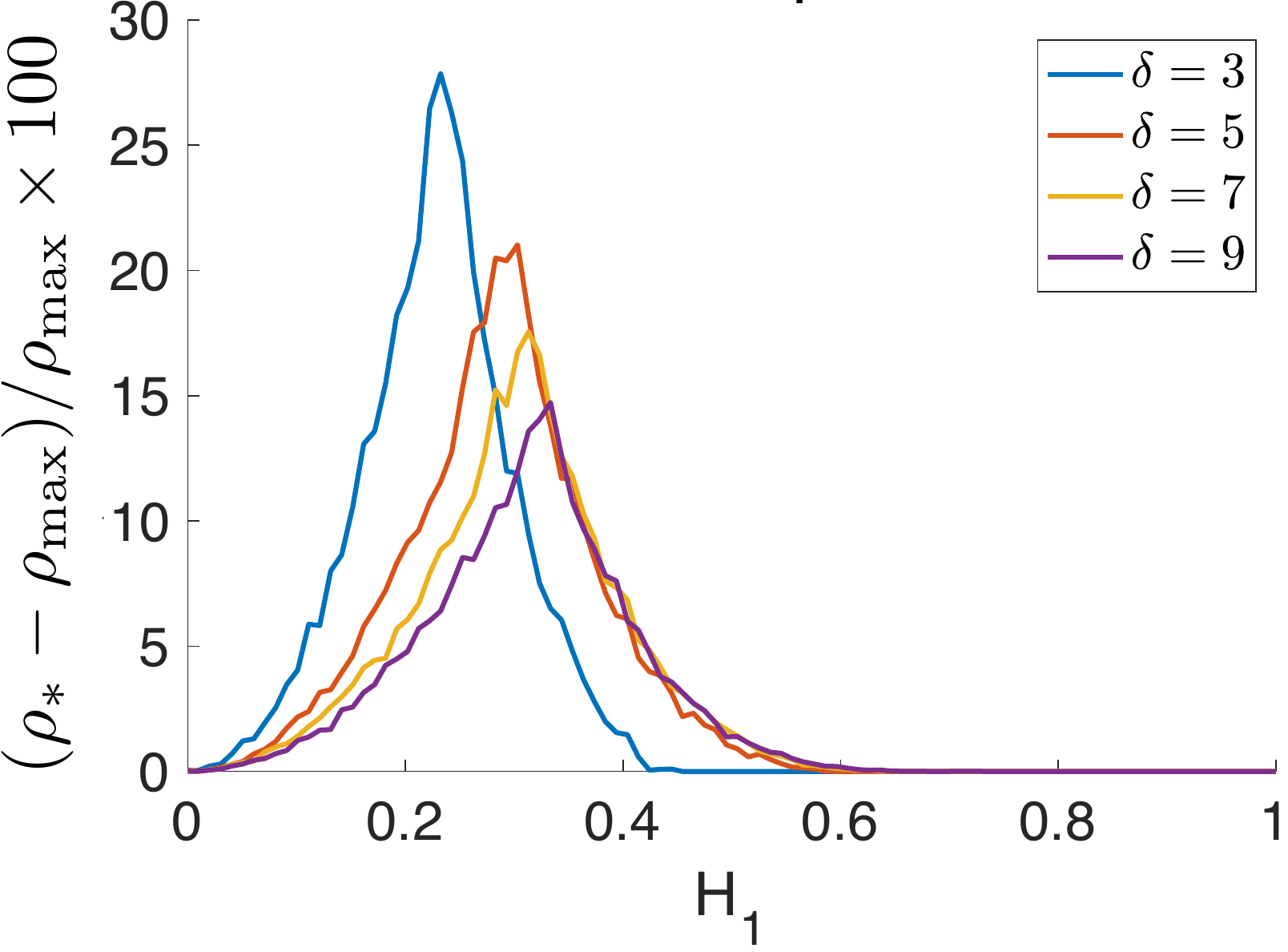}
    }
    \caption{(a) Performance comparison among the linear estimator, the spectral estimator and the proposed optimal combination for a specific output function and ranging values of the sampling ratio $\delta$. The performance is measured in terms of the normalized correlation \eqref{eq:est_corr_def}. (b) Optimal combination coefficient $\theta_*$ as a function of $\delta$ for the same output function as in (a). (c) Percentage performance gain of the combined estimator for different output functions and sampling ratios.}
    \label{fig:intro}
\end{figure}

Our proposed combination of the linear and spectral estimators can significantly boost the correlation with the ground-truth signal \eqref{eq:est_corr_def}. As an illustration, in Figure \ref{fig:intro}(a) we compare against each other the correlations $\rho_L, \rho_s$ and $\rho_*$ of the linear, spectral and optimal combined estimators, respectively, for a range of values of the sampling ratio $\delta=n/d$ and measurements of the form $y_i = 0.3\,\inp{\bx}{\ba_i} + \inp{\bx}{\ba_i}^2 + z_i$. Here, $\bx$ is uniformly distributed on the $d$-dimensional sphere of radius $\sqrt{d}$, $\ba_i \ \sim_{i.i.d.} \ \normal(\b0_d, \bI_d/d)$, $z_i\sim_{i.i.d.}\mathcal{N}(0,0.2)$, and the preprocessing functions are chosen as follows: $\cT_L(y) = y$ and $\cT_s(y) = \min\{y,3.5\}$. The solid lines correspond to analytically derived asymptotic formulae, and they are compared against numerical simulations (cf. markers of the corresponding color) computed for $d=2000$. Specifically, the red line corresponds to the optimal combined estimator $\theta_*\bxl + \bxs$ (in this example, the empirical distribution of $\bx$ is Gaussian). The optimal combination coefficient $\theta_*$ is plotted in Figure \ref{fig:intro}(b) as a function of $\delta$. Note that for values of $\delta$ for which the spectral estimator achieves strictly positive correlation with $\bx$, the combined estimator provides a significant performance improvement. The performance gain depends on: \emph{(i)} the sampling ratio $\delta$ (it can be as large as $\sim30\%$ for $\delta\approx8$), and \emph{(ii)} the output function that defines the measurement. To better visualize this dependence we plot in Figure \ref{fig:intro}(c) the percentage gain
$(\rho_*-\rho_{\max})/\rho_{\max}\times 100$
for various values of $\delta$ and for different output-function parameterizations. Specifically, the $x$-axis in Figure \ref{fig:intro}(c) represents the value of the coefficient $H_1$  of an output function of the form $y_i = 0.5+H_1\,\inp{\bx}{\ba_i} + 0.5\,\inp{\bx}{\ba_i}^2 + z_i$, with $z_i\sim_{i.i.d.}\mathcal{N}(0,0.2)$. Above, 
$\rho_*$ denotes the correlation achieved by our proposed estimator, and $\rho_{\max}$ is the maximum correlation among the linear and spectral estimators.
%To visualize those dependencies, in Figure \ref{fig:intro}(c), for output functions of the form $y_i = 0.5+H_1\,\inp{\bx}{\ba_i} + 0.5\,\inp{\bx}{\ba_i}^2 + z_i$, with $z_i\sim_{i.i.d.}\mathcal{N}(0,0.2)$, we plot this percentage gain as a function of the coefficient $H_1$ and for several values of the sampling ratio $\delta$.

% calculated and plotted the quantity $(\rho_*-\rho_{\max})/\rho_{\max}\times 100$, where $\rho_*=F(\theta_*)$ the correlation of the optimal linear combination of the linear and spectral estimators, and $\rho_{\max}=\max \{|\rho_L|,\rho_s\}$ is the maximum correlation between the two.

% In Figure \ref{fig:intro}(c), we plot this value for output functions of the form $y_i = 0.5+H_1\,\inp{\bx}{\ba_i} + 0.5\,\inp{\bx}{\ba_i}^2 + z_i$, with $z_i\sim_{i.i.d.}\mathcal{N}(0,0.2)$, for ranging values of the coefficient $H_1$ and for several values of the sampling ratio $\delta$.

The rest of the paper is organized as follows. In Section \ref{sec:prel}, we describe the setting and review existing results on the linear and the spectral estimator. In Section \ref{sec:combo}, we present our contributions. The main technical result, Theorem \ref{th:W2conv2}, gives an  exact characterization of the joint empirical distribution of $(\bx, \bxl, \bxs)$. Using this, we derive the Bayes-optimal combination of the estimators $\bxl$ and $\bxs$. In the special case in which the signal prior is Gaussian, the Bayes-optimal combination is linear in $\bxl$ and $\bxs$, and we derive the optimal coefficient. In Section \ref{sec:num}, we demonstrate the effectiveness of our method via numerical simulation. In Section \ref{sec:proof_mainthm}, we describe the generalized AMP algorithm and use it to prove Theorem \ref{th:W2conv2}. 

\section{Preliminaries}\label{sec:prel}

\subsection{Notation and Definitions}

Given $n\in \mathbb N$, we use the shorthand $[n]=\{1, \ldots, n\}$. Given a vector $\bx$, we denote by $\|\bx\|_2$ its Euclidean norm. Given a matrix $\bA$, we denote by $\|\bA\|_{\rm op}$ its operator norm.

The \emph{empirical distribution} of a vector $\bx = (x_1, \ldots, x_d)^{\sT}$ is given by 
$ \frac{1}{d}\sum_{i=1}^d \delta_{x_i}$, where $\delta_{x_i}$ denotes a Dirac delta mass on $x_i$. Similarly, the empirical joint distribution of vectors $\bx, \bx'\in \reals^d$ is $\frac{1}{d} \sum_{i=1}^d \delta_{(x_i, x'_i)}$.

Given two probability measures $\mu$ (on a space $\mathcal{X}$) and $\nu$ (on a space $\mathcal{Y}$), a coupling $\rho$ of $\mu$ and $\nu$ is a probability distribution on $\mathcal{X}\times\mathcal{Y}$ whose marginals coincide with $\mu$ and $\nu$, respectively. 
For $k\ge 1$, the \emph{Wasserstein-$k$} ($W_k$) distance between two probability measures $\mu$, $\nu$ on $\reals^n$ is defined by
\begin{align}
W_k(\mu,\nu) \equiv \inf_{\rho}\E_{(\bX,\bY)\sim \rho}\big\{\|\bX-\bY\|_2^k\}^{1/k}\, ,\label{eq:WassersteinDef}
\end{align}
where the infimum is over all the couplings of $\mu$ and $\nu$.
A sequence of probability distributions $\nu_n$ on $\reals^m$ {converges in $W_k$} to 
$\nu$, written $\nu_n\towa{k} \nu$, if $W_k(\nu_n,\nu) \rightarrow 0$ as $n \rightarrow \infty$.

\subsection{Generalized Linear Model}

Let $\bx \in \reals^d$ be the signal of interest.  We assume that $\|\bx\|^2_2=d$. The signal is observed via inner products with $n$ sensing vectors $(\ba_i)_{i \in [n]}$, with each $\ba_i  \in \reals^d$  having independent Gaussian entries with mean zero and variance $1/d$.  That is,
\beq
( \ba_i) \ \sim_{i.i.d.} \ \normal(\b0_d, \bI_d/d).
\eeq 

Given $g_i = \< \bx, \ba_i \>$, the measurement vector $\by \in \reals^n$ is obtained by drawing each component independently according to a conditional distribution $p_{Y| {G}}$
\beq \label{eq:defy}
y_i \ \sim \ p_{Y | G}(y_i \mid {g}_i), \quad i \in [n].
\eeq
We stack the measurement vectors as rows to define the $n  \times d$ sensing matrix $\bA$. That is,
\beq 
\bA = [\ba_1, \ldots, \ba_n]^{\sT}.
\eeq
We write $\delta_n = \frac{n}{d}$ for the sampling ratio, and assume that $\delta_n \to \delta \in (0, \infty)$. 
Since the entries of the sensing matrix are $\sim_{i.i.d.} \normal(0, 1/d)$, each row of $\bA$  has norm close to $1$.

\subsection{Linear Estimator} \label{sec:lin_est}

Given the measurements $(y_i)_{i \in [n]}$ and a preprocessing function $\cT_L: \reals \to \reals$, define the $n\times 1$ vector
\beq
\label{eq:zdef}
\bzl = [\cT_L(y_1), \ldots, \cT_L(y_n)]^\sT.
\eeq
Consider the following linear estimator that  averages the data as follows:
\beq\label{eq:xlin_def}
\bxl := \frac{\sqrt{d}}{n} \bA^\sT\bzl  = \frac{\sqrt{d}}{n}\sum_{i=1}^n \cT_L(y_i)\ba_i \,.
\eeq

The following lemma characterizes the asymptotic performance of this simple estimator. The proof is rather straightforward, and we include it in Appendix \ref{sec:proof_lin} for completeness. 

\begin{lemma}\label{lemma:pt_lin}
Let $\bx$ be such that $\|\bx\|^2_2=d$, $\{\ba_i\}_{1\le i\le n}\sim_{i.i.d.}\normal({\b0}_d,\id_d/d)$, and $\by$ be distributed according to \eqref{eq:defy}. Let $n/d\to \delta$, $G\sim \normal(0, 1)$ and define $Z_L=\mathcal T_L(Y)$
for $Y\sim p_{Y|G}(\,\cdot\,|\,G)$ such that $\E\{|GZ_L|\}<\infty$. Let $\bxl$ be the linear estimator defined as in \eqref{eq:xlin_def}. Then, 
as $n \to \infty$,
	\begin{equation}\label{eq:lin_cor}
\| \bxl \|_2^2 \ras \left(\E\left\{G Z_L\right\}\right)^2 + \frac{\E\{Z_L^2\}}{\delta},	\quad \text{ and } \quad
\frac{\langle\bxl, \bx\rangle}{\|{\bxl}\|_2 \, \norm{\bx}_2} \ras\frac{\E\left\{G Z_L\right\}}{\sqrt{\left(\E\left\{G Z_L\right\}\right)^2 + {\E\left\{Z_L^2\right\}}\big/{\delta}}}\,.
	\end{equation}		
\end{lemma}

\subsection{Spectral Estimator}

Given the measurements $(y_i)_{i \in [n]}$, consider the $n \times n$ diagonal matrix
\beq
\label{eq:Zdef}
\bZ_s = \diag(\cT_s(y_1), \ldots, \cT_s(y_n)),
\eeq
where $\cT_s: \reals \to \reals$ is a preprocessing function. Consider the $d \times d$ matrix
\beq
\label{eq:Dn_def}
\bD_n =   \bA^\sT \bZ_s \bA.
\eeq
Let $G\sim \normal(0, 1)$, $Y\sim p(\cdot \mid G)$, and $Z_s=\mathcal T_s(Y)$. We will make the following assumptions on $Z_s$. 

\textbf{(A1)} $\mathbb P(Z_s=0)<1$. \label{assump:A1}

\textbf{(A2)} $Z_s$ has bounded support and $\tau$ is the supremum of this support, i.e., 
\begin{equation}\label{eq:deftau}
\tau = \inf\{ z : \mathbb P (Z_s\le z) = 1\}.
\end{equation}

\textbf{(A3)} As $\lambda$ approaches $\tau$ from the right, we have
\begin{equation}\label{eq:hplemmaub1}
\lim_{\lambda\to \tau^+}{\mathbb E}\left\{\frac{Z_s}{(\lambda-Z_s)^2}\right\}=\lim_{\lambda\to \tau^+}{\mathbb E}\left\{\frac{Z_s\cdot G^2}{\lambda-Z_s}\right\}=\infty.
\end{equation}

Let us comment on these assumptions. First, the condition \textbf{(A1)} simply avoids the degenerate case in which the measurement vector, after passing through the preprocessing function, is $0$ with high probability. Second, the condition \textbf{(A2)} requires that the support of $Z_s$ is bounded both from above and below. This assumption appears in the papers that have recently analyzed the performance of spectral estimators \cite{lu2017phase,mondelli2017fundamental,luo2019optimal}, and it is also required for Lemma \ref{lemma:pt} below. Requiring that the support of $Z_s$ is bounded from above is rather natural, since the argument relies on the matrix $\bD_n$ having a spectral gap. It is not clear whether having the support of $Z_s$ bounded from both sides (rather than only from above) is necessary, and investigating this aspect is an interesting avenue for future research.   Let us also point out that the condition \textbf{(A2)} is purely technical and rather mild. In fact, if the desired preprocessing function is not bounded\footnote{This is the case e.g. in noiseless phase retrieval, where $y_i=\langle \bx, \ba_i\rangle^2$ and the optimal preprocessing function is $\cT_s^*(y)=1-1/y$.}, then one can construct a sequence of bounded approximations that approach its performance, as done e.g. in \cite{luo2019optimal}. Finally, the condition \textbf{(A3)} essentially
requires that $Z_s$ has sufficient probability mass near the supremum of the support $\tau$. One sufficient condition is that the law of $Z_s$ has a point mass at $\tau$. If this is not the case, the argument in (115)-(118) of \cite{mondelli2017fundamental} shows how to modify the preprocessing function $\cT_s$ so that \emph{(i)} condition \textbf{(A3)} holds, and \emph{(ii)} the spectral estimator suffers no performance loss.

For $\lambda\in (\tau, \infty)$ and $\delta\in (0, \infty)$, define
\begin{equation}\label{eq:defphi2}
\phi(\lambda) = \lambda \cdot {\mathbb E}\left\{\frac{Z_s\cdot G^2}{\lambda-Z_s}\right\},
\end{equation}
and 
\begin{equation}\label{eq:defpsi}
\psi_{\delta}(\lambda) = \lambda\left(\frac{1}{\delta}+{\mathbb E}\left\{\frac{Z_s}{\lambda-Z_s}\right\}\right).
\end{equation}
Note that $\phi(\lambda)$ is a monotone non-increasing function and that $\psi_{\delta}(\lambda)$ is a convex function. Let $\bar{\lambda}_{\delta}$ be the point at which $\psi_{\delta}$ attains its minimum, i.e.,
\begin{equation}\label{eq:minpsi}
\bar{\lambda}_\delta = \arg\min_{\lambda\ge \tau} \psi_{\delta}(\lambda).
\end{equation}
For $\lambda\in (\tau, \infty)$, define also
\begin{equation}\label{eq:defzeta}
\zeta_{\delta}(\lambda) = \psi_{\delta}(\max(\lambda, \bar{\lambda}_\delta)).
\end{equation}

The spectrum of $\bD_n$ exhibits a phase transition as $\delta$ increases. The most basic phenomenon of this kind was unveiled for low-rank perturbations of a Wigner matrix: the well-known BBAP phase transition, first discovered in the physics literature \cite{hoyle2004principal}, and named after the authors of \cite{baik2005phase}. Here, the model for the random matrix $\bD_n$ is quite different from that considered in \cite{hoyle2004principal,baik2005phase}, and the phase transition is formalized by the following result.

\begin{lemma}\label{lemma:pt}
Let $\bx$ be such that $\|\bx\|^2_2=d$, $\{\ba_i\}_{1\le i\le n}\sim_{i.i.d.}\normal({\b0}_d,\id_d/d)$, and $\by$ be distributed according to \eqref{eq:defy}. Let $n/d\to \delta$, $G\sim \normal(0, 1)$ and define $Z_s=\mathcal T_s(Y)$
for $Y\sim p_{Y|G}(\,\cdot\,|\,G)$. Assume that $Z_s$ satisfies the assumptions \textbf{(A1)}-\textbf{(A2)}-\textbf{(A3)}. Let $\bxs$ be the principal eigenvector of the matrix $\bD_n$, defined as in \eqref{eq:Dn_def}. Then, the following results hold:
\begin{enumerate}

	\item The equation
	\begin{equation}\label{eq:uniqsol}
\zeta_{\delta}(\lambda) = \phi(\lambda) 	
	\end{equation}
	admits a unique solution, call it $\lambda_{\delta}^*$, for $\lambda > \tau$. 

	\item As $n \to \infty$,
	\begin{equation}\label{eq:numpred}
\frac{|\langle\bxs, \bx\rangle|^2}{\norm{\bxs}_2^2 \, \norm{\bx}_2^2} \stackrel{\mathclap{\mbox{\footnotesize a.s.}}}{\longrightarrow} \left\{\begin{array}{ll}
\vspace{1em}
0, & \mbox{ if }\psi_{\delta}'(\lambda_{\delta}^*)\le 0,\\
\displaystyle\frac{\psi_{\delta}'(\lambda_{\delta}^*)}{\psi_{\delta}'(\lambda_{\delta}^*)-\phi'(\lambda_{\delta}^*)}, & \mbox{ if }\psi_{\delta}'(\lambda_{\delta}^*)> 0,\\
\end{array}\right.
	\end{equation}		
where $\psi_{\delta}'$ and $\phi'$ denote the derivatives of these two functions.

	\item Let $\lambda_1^{\bD_n}\ge \lambda_2^{\bD_n}$ denote the two largest eigenvalues of $\bD_n$. Then, as $n\to \infty$,
	\begin{equation}
	\begin{split}
	\lambda_1^{\bD_n} &\stackrel{\mathclap{\mbox{\footnotesize a.s.}}}{\longrightarrow}\delta \,\zeta_{\delta}(\lambda_{\delta}^*),\\ 
	\lambda_2^{\bD_n} &\stackrel{\mathclap{\mbox{\footnotesize a.s.}}}{\longrightarrow}\delta\,\zeta_{\delta}(\bar{\lambda}_{\delta}).\\ 
	\end{split}	
	\end{equation}
\end{enumerate}
\end{lemma}

The proof immediately follows from Lemma 2 of \cite{mondelli2017fundamental}. In that statement, it is assumed that $\bx$ is uniformly distributed on the  $d$-dimensional sphere, but this assumption is actually never used. In fact, since the sensing vectors $\{\ba_i\}_{1\le i\le n}$ are i.i.d. standard Gaussian, to prove the result above,  without loss of generality we can assume that $\bx= \sqrt{d}\be_1$, where $\be_1$ is the first element of the canonical basis of $\mathbb R^d$. We also note that  the signal $\bx$ and the measurement matrix $\bA$ differ from Lemma 2 in \cite{mondelli2017fundamental} for a scaling factor. This accounts for an extra term $\delta$ in the expression of the eigenvalues of $\bD_n$.

%%%%%%%%%%%%%%%%%%%%%%%%%%%%%%%%%%%%%%%%%%%%%%%%%%%
\section{Main Results}\label{sec:combo}

Throughout this section, we will make the following additional assumptions on the signal $\bx$, the output function of the GLM, and the preprocessing functions $\cT_L$ and $\cT_s$ used for the linear and spectral estimators, respectively.

\textbf{(B1)} Let $\hat{P}_{X,d}$ denote the empirical distribution of $\bx \in \reals^d$, with $\| \bx \|^2_2 =d$. As $d \to \infty$, $\hat{P}_{X,d}$ converges weakly to a distribution $P_X$ such that, for some $k \geq 2$, the following hold: \emph{(i)} $\E_{P_X}\{ \abs{X}^{2k -2} \} < \infty$, \,  and \emph{(ii)} $\lim_{d \to \infty} \E_{\hat{P}_{X,d}}\{ \abs{X}^{2k -2} \} = \E_{P_X}\{ \abs{X}^{2k -2} \}$. Furthermore, $\E\{ \abs{Y}^{2k -2} \} < \infty$, with $Y\sim p_{Y|G}(\,\cdot\,|\,G)$ and $G \sim \normal(0,1)$.
\label{Assump:B1}

\textbf{(B2)} The function $\cT_L: \reals \to \reals$ is Lipschitz and 
$|\E\{ \cT_L(Y) G  \}|>0$; the function $\cT_s: \reals \to \reals$ is bounded and Lipschitz.

The assumption \textbf{(B2)} is mainly technical and rather mild, since one can construct a sequence of approximations of the desired $\cT_L, \cT_s$ that satisfy \textbf{(B2)}.

Lemmas \ref{lemma:pt_lin} and \ref{lemma:pt} in the previous sections derive formulae for the asymptotic correlation of $\bx$ with the linear estimator $\bxl$ and the spectral estimator $\bxs$. For convenience, let us denote these as follows:
\begin{align}\label{eq:rhos_def}
\rho_L:=\frac{\E\left\{Z_L G\right\}}{\sqrt{\left(\E\left\{Z_L G\right\}\right)^2 + {\E\left\{Z_L^2\right\}}\big/{\delta}}},\quad\text{and}\quad
\rho_s:= \sqrt{\frac{\psi_{\delta}'(\lambda_{\delta}^*)}{\psi_{\delta}'(\lambda_{\delta}^*)-\phi'(\lambda_{\delta}^*)}}.
\end{align}
We also denote by $n_L$ the high-dimensional limit of $\|\bxl\|_2$, 
\begin{equation}
    n_L = \sqrt{\left(\E\left\{G Z_L\right\}\right)^2 + {\E\left\{Z_L^2\right\}}\big/{\delta}},
\end{equation}
and we define
\begin{equation}\label{eq:crosscorr}
   q := \displaystyle\sqrt{\frac{\psi_{\delta}'(\lambda_{\delta}^*)}{\psi_{\delta}'(\lambda_{\delta}^*)-\phi'(\lambda_{\delta}^*)}}\cdot\frac{\E\left\{ \frac{Z_L \cdot G}{1-\frac{1}{\lambda_{\delta}^*} Z_s} \right\}}{\sqrt{\left(\E\left\{Z_L\cdot G\right\}\right)^2 + \frac{\E\left\{Z_L^2\right\}}{\delta}}}. 
\end{equation}

\subsection{Joint Distribution of Linear and Spectral Estimators}

The key technical challenge is to compute the limiting joint empirical distribution of the signal $\bx$, the linear estimator $\bxl$, and the spectral estimator $\bxs$. This result is stated in terms of \emph{pseudo-Lipschitz} test functions.
\begin{definition}[Pseudo-Lipschitz test function]\label{def:PLk}
We say that a function $\psi: \reals^m \to \reals$ is pseudo-Lipschitz of order $k \geq 1$, denoted $\psi \in \PL(k)$, if there is a constant $C > 0$  such that 
\beq
\norm{\psi(\bx)-\psi(\by)}_2 \le C (1 + \| \bx \|_2^{k-1} + \| \by \|_2^{k-1} )\norm{\bx - \by}_2,
\label{eq:PL2_prop}
\eeq
for all $\bx,\by \in \mathbb{R}^m$.
\end{definition}
Examples of test functions in $\PL(2)$ with $m=2$ include $\psi(a,b) = (a-b)^2$, $\psi(a,b) =ab$, and $\psi(a,b) = \abs{a-b}$. We note that if $\psi \in \PL(k)$, then $\psi(\bx) \leq C'(1+ \| \bx\|_2^k)$ for some constant $C' >0$. Also note that if $\psi \in \PL(k)$ for $k \geq 2$, then $\psi \in \PL(k')$ for $1 \leq k' \leq (k-1)$.

\begin{theorem}[Joint distribution]\label{th:W2conv2}
Let $\bx$ be such that $\|\bx\|^2_2=d$, $\{\ba_i\}_{1\le i\le n}\sim_{i.i.d.}\normal({\b0}_d,\id_d/d)$, and $\by$ be distributed according to \eqref{eq:defy}. Let $n/d\to \delta$, $G\sim \normal(0, 1)$, $Z_L=\mathcal T_L(Y)$, and $Z_s=\mathcal T_s(Y)$ for $Y\sim p_{Y|G}(\,\cdot\,|\,G)$. Assume that \textbf{(A1)}-\textbf{(A2)}-\textbf{(A3)} and \textbf{(B1)}-\textbf{(B2)} hold. Assume further that $\psi_{\delta}'(\lambda_{\delta}^*)> 0$, and let $\bxs$ be the principal eigenvector of $\bD_n$, defined as in \eqref{eq:Dn_def} with preprocessing function $\cT_s$, with the sign of $\bxs$ chosen so that $\langle\bxs, \bx\rangle\ge 0$. Let $\bxl$ be the linear estimator defined as in \eqref{eq:xlin_def} with preprocessing function $\cT_L$. 

Consider the rescalings $\bx^{\rm s}=\sqrt{d}\bxs$ and $\bx^{\rm L} = \sqrt{d} \bxl/n_L$. Then, the following holds almost surely for any PL(k) function $\psi: \reals^3 \to \reals$:
\begin{align}
& \lim_{d \to \infty} \,   \frac{1}{d} \sum_{i=1}^d \psi(x_i, x^{\rm L}_i,  x^{\rm s}_i) = \E\{ \psi(X, 
\, \rho_L X + W_L
,\,  \rho_s X + W_s)) \} \label{eq:psiX1new}. 
\end{align}
Here, $X \sim P_X$, and $(W_L, W_s)$ are independent of $X$ and jointly Gaussian with zero mean and covariance given by $\E\{W_L^2\}=1-\rho_L^2$, $\E\{W_s^2\}=1-\rho_s^2$ and $\E\{W_L W_s\} = q-\rho_L\rho_s$.
\end{theorem}
%\rv{changed the notation $\bx^{\rm L}$ to $\bx^{\rm L}$ as the scaling was inconsistent with the $\bx^{\rm L}$ defined in Lemma \ref{lem:evec_emp_dist}.}

Note that the order $k$ of the pseudo-Lipschitz test function appearing in \eqref{eq:psiX1new} is the same as the integer $k$ appearing in assumption \textbf{(B1)}. In particular, the order of pseudo-Lipschitz functions for which  \eqref{eq:psiX1new} holds is only constrained by the   fact that the random variables $X$ and $Y$ should have finite moments of order $2k-2$.  The proof of the theorem is given in Section \ref{sec:proof_mainthm}.

\begin{remark}[What happens if either linear or spectral are ineffective?]\label{rmk:special} From Lemma \ref{lemma:pt_lin}, the assumption $|\E\{ \cT_L(Y) G  \}|>0$ contained in \textbf{(B2)} implies that $|\rho_L|>0$. Similarly, from Lemma \ref{lemma:pt}, the assumption $\psi_{\delta}'(\lambda_{\delta}^*)> 0$ implies that $\rho_s>0$. Thus, Theorem \ref{th:W2conv2} assumes that both the linear and the spectral estimators are effective. We note that a similar result also holds  when only one of the two estimators achieves strictly positive correlation. In that case, $\psi: \reals^2 \to \reals$ takes as input the components of the signal $\bx$ and of the estimator that is effective (as well as the corresponding random variables), and a formula analogous to \eqref{eq:psiX1new} holds. The proof of this claim is easily deduced from the argument of Theorem \ref{th:W2conv2}. A simpler proof using a rotational invariance argument along the lines of \eqref{eq:rot_inv_app}-\eqref{eq:rotinvfin} also leads to the same result.
\end{remark}

\begin{remark}[Convergence in $W_k$]\label{rmk:equiv}
The result in Eq. \eqref{eq:psiX1new} is equivalent to the statement that the empirical joint distribution of  $(\bx, \bx^{\rm L}, \bx^{\rm s})$ converges almost surely in $W_k$ distance to the joint law of
\begin{equation}
 (X, 
\, \rho_L X + W_L
,\,  \rho_s X + W_s).   
\end{equation}
This follows from the fact that a sequence of distributions $P_n$ with finite $k$-th moment  converges in $W_k$  to $P$ if and only if   $P_n$ converges weakly to $P$ and
 $\int \| a \|^k \, \de P_n(a)  \to \int \| a \|^k \, \de P(a)$,  see
 \cite[Definition 6.7, Theorem 6.8]{villani2008optimal}. 
\end{remark}

\subsection{Optimal Combination}

Equipped with the result of Theorem \ref{th:W2conv2}, we now reduce the vector problem of estimating $\bx$ given $(\bxl, \bxs)$ to an estimation problem over scalar random variables, i.e., how to optimally estimate $X$ from observations $X_L:=\rho_L X + W_L$ and $X_s:=\rho_s X + W_s$, where $W_L$ and $W_s$ are jointly Gaussian. The Bayes-optimal estimator for this scalar problem is given by
\begin{equation}\label{eq:Bayesopt}
    F_*(x_L, x_s) := \mathbb E\{X \mid X_L=x_L, X_s=x_s\}.
\end{equation}
This is formalized in the following result. 
%Its proof is rather straightforward and it is contained in Appendix \ref{app:optf} for completeness.

\begin{lemma}\label{lemma:optf}
Let  $(X,X_L,X_s)$ be jointly distributed random variables such that
\begin{align}\label{eq:meas}
    X_L = \rho_L X + W_L\qquad\text{and}\qquad X_s = \rho_s X + W_s\,,
\end{align}
 where $(W_L,W_s)$ are jointly Gaussian independent of $X$ with zero mean and covariance given by  $\E\{W_L^2\}=1-\rho_L^2$, $\E\{W_s^2\}=1-\rho_s^2$ and $\E\{W_L W_s\} = q-\rho_L\rho_s$. Assume that $\rho_L\neq 0$ or $\rho_s\neq 0$. Let 
\begin{equation}\label{eq:calV}
    \mathcal{V}:=\left\{ f(X_L,X_s)~:~0<\E\{f^2(X_L,X_s)\}<\infty\right\},
\end{equation}
and consider the following optimal estimation problem of $X$ given $X_L$ and $X_s$ over all measurable estimators $f:\R^2\rightarrow\R$ in $\mathcal{V}$:
\begin{align}
\max_{f\in\mathcal{V}} \, \frac{|\E\left\{X\cdot f(X_L,X_s)\right\}|}{\sqrt{\E\{X^2\} \cdot \E\left\{f^2(X_L,X_s)\right\}}}.\label{eq:opt_est}
\end{align}
Then, for any $c\neq 0$, $\hat X = cF_*(X_L, X_s)$ attains the maximum in \eqref{eq:opt_est}, where $F_*$ is defined in \eqref{eq:Bayesopt}.
\end{lemma}
\begin{proof}%[Proof of Lemma \ref{lemma:optf}]
By the tower property of conditional expectation, for any $f \in \mathcal{V}$ we have
\beq
\begin{split}
\frac{\abs{\E\left\{X\cdot f(X_L,X_s)\right\}}}{\sqrt{ \E\{ f(X_L,X_s)^2 \}}}
= \frac{\abs{\E\left\{ \E\{ X \mid X_L, X_s \} \cdot f(X_L,X_s)\right\}}}{\sqrt{ \E\{ f(X_L,X_s)^2 \}}}
& \leq \sqrt{\E\left\{ \E\{ X \, | \,  X_L, X_s \}^2 \right\}} ,
\end{split}
\label{eq:towerCS}
\eeq
where we have used the Cauchy-Schwarz inequality. Moreover, the inequality in \eqref{eq:towerCS} becomes an equality if and only if 
$f(X_L,X_s) = c \, \E\{ X \, | \, X_L, X_s \}$, for some $c \neq 0$, which proves the result. 
\end{proof}

%\rv{Included a proof  that is a bit shorter than the one that was previously in the appendix. Also removed the assumption that $X$ is zero mean and unit variance, which I don't think we need. Also added $\sqrt{\E X^2}$ in the denominator of \eqref{eq:opt_est} to make it a normalized correlation. Doesn't affect the maximization. }

At this point, we are ready to show how to optimally combine the linear estimator $\bxl$ and the spectral estimator $\bxs$.

\begin{theorem}[Optimal combination]\label{th:optimality}
Consider the setting of Theorem \ref{th:W2conv2}.
%, and assume that $\rho_L\neq 0$ or $\rho_s\neq 0$.
Let $F_*$ be defined in \eqref{eq:Bayesopt} and assume that $F_*\in PL(\lfloor k/2\rfloor)$. Then, as $n\to\infty$,
\begin{equation}\label{eq:optres0}
\frac{|\langle F_*(\bx^{\rm L}, \bx^{\rm s}), \bx\rangle|}{\|F_*(\bx^{\rm L}, \bx^{\rm s})\|_2 \|\bx\|_2}\ras \rho_*:=\frac{|\E\left\{X\cdot F_*(X_L,X_s)\right\}|}{\sqrt{\E\left\{F_*^2(X_L,X_s)\right\}}},
\end{equation}
where $F_*$ acts component-wise on $\bx^{\rm L}$ and $\bx^{\rm s}$, i.e., $F_*(\bx^{\rm L}, \bx^{\rm s})=(F_*(x^{\rm L}_1, x^{\rm s}_1), \ldots, F_*(x^{\rm L}_d, x^{\rm s}_d))$. Furthermore, for any $f\in PL(\lfloor k/2\rfloor)$ acting component-wise on $\bx^{\rm L}$ and $\bx^{\rm s}$, almost surely,
\begin{equation}\label{eq:optres}
    \lim_{n\to\infty}\frac{|\langle f(\bx^{\rm L}, \bx^{\rm s}), \bx\rangle|}{\|f(\bx^{\rm L}, \bx^{\rm s})\|_2 \|\bx\|_2}\le\rho_*.
\end{equation}
\end{theorem}

\begin{proof}[Proof of Theorem \ref{th:optimality}]
If $f(b, c)\in$ PL($\lfloor k/2\rfloor$), then one can immediately verify that \emph{(i)} $\psi(a, b, c)=af(b, c)\in$ PL($k$), and \emph{(ii)} $\psi(a, b, c)=(f(b, c))^2\in$ PL($k$). Thus, by applying Theorem \ref{th:W2conv2}, we have that, for any $f(b, c)\in$ PL($\lfloor k/2\rfloor$), as $n\to\infty$,
\begin{equation}
\begin{split}
    \frac{|\langle f(\bx^{\rm L}, \bx^{\rm s}), \bx\rangle|}{\|f(\bx^{\rm L}, \bx^{\rm s})\|_2 \|\bx\|_2} &\ras \frac{\mathbb E\{Xf(\rho_LX+W_L, \rho_sX+W_s)\}}{\sqrt{\mathbb E\{f^2(\rho_LX+W_L, \rho_sX+W_s)\}}}.
\end{split}
\end{equation}
By taking $f=F_*$, the result \eqref{eq:optres0} immediately follows. By applying Lemma \ref{lemma:optf}, \eqref{eq:optres} also follows and the proof is complete. 
\end{proof}

The integer $k$ appearing in assumption \textbf{(B1)} is the same one defining the order $\lfloor k/2\rfloor$ of the pseudo-Lipschitz functions $F_*$ and $f$ in \eqref{eq:optres0}-\eqref{eq:optres}. 

\begin{remark}[What happens if either linear or spectral are ineffective?]\label{rmk:special2} Theorem \ref{th:optimality} considers the same setting of Theorem \ref{th:W2conv2}, and therefore it assumes that $\rho_L\neq 0$ and $\rho_s\neq 0$. The results in \eqref{eq:optres0}-\eqref{eq:optres} still hold if either $\rho_L= 0$ or $\rho_s= 0$ (and even in the case $\rho_L= \rho_s= 0$). For the sake of simplicity, suppose that $\rho_L= 0$ (the argument for $\rho_s= 0$ is analogous). Then, $X_L=W_L$ is independent of $X$ and therefore the conditional expectation in \eqref{eq:Bayesopt} does not depend on $x_L$. Recall from Remark \ref{rmk:special} that if $\rho_L=0$, then a formula analogous to \eqref{eq:psiX1new} holds where $\psi: \reals^2 \to \reals$ takes as input the components of $\bx$ and $\bx^{\rm s}$ on the LHS, and the corresponding random variables on the RHS. Hence, \eqref{eq:optres0}-\eqref{eq:optres} are obtained by following the same argument in the proof of Theorem \ref{th:optimality}.

We highlight that, even if one of the two estimators is ineffective, the proposed optimal combination can still improve on the performance of the other one. This is due to the fact that the function $F_*$ takes advantage of the knowledge of the signal prior. We showcase an example of this behavior for a binary prior in Figure \ref{fig:Bayes} discussed in Section \ref{sec:num_bayes}. We also note that if the signal prior is Gaussian, then no improvement is possible when one of the two estimators has vanishing correlation with the signal, see Figures \ref{fig:relu}-\ref{fig:Hermite} in Section \ref{subsec:numoptlin}. In fact, as detailed in Section \ref{sec:opt_lin}, in the Gaussian case $F_*(\bx^{\rm L}, \bx^{\rm s})$ is a linear combination of $\bx^{\rm L}$ and $\bx^{\rm s}$. Thus, if $\rho_L=0$ (resp. $\rho_s=0$), then $F_*(\bx^{\rm L}, \bx^{\rm s})$ is aligned with $\bx^{\rm s}$ (resp. $\bx^{\rm L}$).
\end{remark}

\begin{remark}[Sign of $\bxs$] The spectral estimator $\bxs$ is defined up to a change of sign, since it is the principal eigenvector of a suitable matrix. In Theorem \ref{th:W2conv2} and \ref{th:optimality}, we pick the sign of $\bxs$ such that $\langle \bxs, \bx\rangle\ge 0$. In practice, there is a simple way to resolve the sign ambiguity: one can match the sign of $q$ as defined in \eqref{eq:crosscorr} with the sign of the scalar product $\langle \bxl, \bxs\rangle$ (see also \eqref{eq:ascorrlim}). 
\end{remark}

\begin{remark}[Sufficient condition for pseudo-Lipschitz $F_*$]
The assumption that the Bayes-optimal estimator $F_*$ in \eqref{eq:Bayesopt} is pseudo-Lipschitz is fairly mild. In fact, $F_*$ is Lipschitz if either of the following two conditions on $X$ hold \cite[Lemma 3.8]{feng2021unifying}: \emph{(i)} $X$ has a log-concave distribution, or   \emph{(ii)} there exist independent random variables $U,V$ such that $U$ is Gaussian, $V$ is compactly supported and $X \stackrel{\rm d}{=} U+V$.
\end{remark}

\begin{remark}[Non-separable combinations]
The optimality of $F_*$ in Theorem \ref{th:optimality} can be extended to a class of  combined estimators of the form $f_d(\bx^{\rm L}, \bx^{\rm s})$, where $f_d: \reals^d \times \reals^d \to \reals^d$ may not act component-wise  on $(\bx^{\rm L}, \bx^{\rm s})$. Given $f_d$, we define the function $S_{f_d}(\bx^{\rm L}, \bx^{\rm s})= \frac{1}{d}\|f_d(\bx^{\rm L}, \bx^{\rm s}) \|^2$. Let $f_d: \reals^d \to \reals^d$ be any  sequence of functions (indexed by $d$) such that $S_{f_d}: \reals^d \times \reals^d \to \reals$ is \emph{uniformly} pseudo-Lipschitz of order $k$. That is, for each $d$, the property \eqref{eq:PL2_prop} holds for $S_{f_d}$ with a pseudo-Lipschitz constant $C$ that does not depend on $d$. Then, the state evolution result in \cite[Theorem 1]{berthier2020state} for non-separable test functions implies that 
\beq
\label{eq:result_nonsep}
\lim_{d \to \infty} \prob\left( \frac{|\langle f_d(\bx^{\rm L}, \bx^{\rm s}), \bx\rangle|}{\|f_d(\bx^{\rm L}, \bx^{\rm s})\|_2 \|\bx\|_2} \leq \rho_*\right) =1.
\eeq
The result above is in terms of convergence in probability, while the limiting statement in \eqref{eq:optres} holds almost surely.  This is because the state evolution result  for AMP with non-separable functions \cite[Theorem 1]{berthier2020state}  is obtained in terms of convergence in probability.
\end{remark}

\subsection{A Special Case: Optimal Linear Combination}\label{sec:opt_lin}

Theorem \ref{th:optimality} shows that the optimal way to combine $\bxl$ and $\bxs$ is via the Bayes-optimal estimator $F_*$ for the corresponding scalar problem. If the signal prior $X$ is standard Gaussian, then $F_*(\bx^{\rm L}, \bx^{\rm s})$ is a linear combination of $\bx^{\rm L}$ and $\bx^{\rm s}$. In this section, we provide closed-form expressions for the performance of such optimal \emph{linear} combination. 

For convenience, let us denote the normalized linear estimator as $\bxln$, i.e., $\bxln=\bxl/\|\bxl\|_2$. (Recall that the spectral estimator $\bxs$ is already normalized, i.e., $\|\bxs\|_2=1$.) We consider an estimator $\bxc(\theta)$ of $\bx$, parameterized by $\theta\in\R\cup\{\pm\infty\}$, defined as follows:
\begin{align}\label{eq:combo_def}
\bxc(\theta):= \theta \bxln + \bxs,\quad \theta\in\R\cup\{\pm\infty\},
\end{align}
where we use the convention, $\bxc(\theta)=\pm\bxln$ for $\theta=\pm\infty$.
%Also, let $\bxcn(\theta)=\frac{\bxc(\theta)}{\norm{\bxc(\theta)}_2}.$

Let us now compute the asymptotic performance of the proposed estimator $\bxc(\theta)$ in \eqref{eq:combo_def}. 
Specifically, using Lemmas \ref{lemma:pt_lin} and \ref{lemma:pt}, it follows immediately that 
\beq
\left\< \bxc(\theta), \, \frac{\bx}{\norm{\bx}_2} \right\> \ras \theta \cdot\rho_L + \rho_s \,.
\label{eq:xcs_corr}
\eeq
In order to conclude with the limit of the normalized correlation $\frac{\inp{\bxc(\theta)}{\x}}{\norm{\bxc(\theta)}_2\norm{\bx}_2}$, it still remains to compute the magnitude of the new estimator:
\beq
\norm{\bxc(\theta)}_2^2 = \theta^2 \|\bxln\|_2^2 + \norm{\bxs}_2^2 + 2 \theta \inp{\bxln}{\bxs} = \theta^2+1+ 2 \theta \inp{\bxln}{\bxs}.
\label{eq:xc_mag}
\eeq

This is possible thanks to the following result, which gives the correlation between the linear and the spectral estimator as well as the asymptotic performance of the linear combination $\bxc(\theta)$.
\begin{corollary}[Performance of linear combination]\label{cor:crosscor}
Consider the setting of Theorem \ref{th:W2conv2}.
%, and assume that \ct{$\E\{GZ_L\}<\infty$} $\E[|GZ|]<\infty$.
Then, as $n\rightarrow\infty$, 
	\begin{equation}\label{eq:ascorrlim}
\frac{\langle\bxl, \bxs\rangle}{\|{\bxl}\|_2 \, \norm{\bxs}_2} \ras q,
\end{equation}
where $q$ is given by \eqref{eq:crosscorr}.
 Furthermore, let $\bxc(\theta)$ be the estimator defined in \eqref{eq:combo_def} parameterized by $\theta\in\R$.  Then, as $n\rightarrow\infty$,
\begin{align}
\frac{\inp{\bxc(\theta)}{\bx}}{\norm{\bxc(\theta)}_2\norm{\bx}_2} \ras \frac{\theta\rho_L + {\rho_s}}{\sqrt{1+\theta^2+2\theta q}}=:F(\theta).
\label{eq:Ftheta_def}
\end{align}
\end{corollary}
\begin{proof}
The limit of the correlation \eqref{eq:ascorrlim} follows by applying Theorem \ref{th:W2conv2} with the PL(2) function $\psi(a, b, c)=bc$ and using that $\|\bxl\|_2 \ras n_L$. The result \eqref{eq:Ftheta_def} then follows from \eqref{eq:xc_mag} and \eqref{eq:ascorrlim}, recalling that 
$\inp{\bxln}{\bxs} = \frac{\langle\bxl, \bxs\rangle}{\|{\bxl}\|_2 \, \norm{\bxs}_2}$ and 
$\| \bxs \|_2=1$.
\end{proof}

Using \eqref{eq:rhos_def}, the parameter $q$ can be alternatively expressed in terms of $\rho_L$ and $\rho_s$ in the following compact form:
\begin{align}\label{eq:q_def_compact}
q = \rho_L \cdot \rho_s \cdot {\E\left\{ \frac{Z_L  G}{1-\frac{1}{\lambda_\delta^*} Z_s} \right\}}\Big/{\E\left\{Z_L G\right\}}\,.
\end{align}
Observe also that  $F(0)=\rho_s$ and $F(\infty):=\lim_{\theta\rightarrow\pm\infty}F(\theta)=\pm\rho_L$.

Using the characterization of Corollary \ref{cor:crosscor}, we can compute the combination coefficient $\theta_*$ that leads to the asymptotically optimal linear combination of the form \eqref{eq:combo_def}. 

\begin{corollary}[Optimal linear combination]\label{lem:fopt}
Recall the notation in Corollary \ref{cor:crosscor}
and define 
\begin{align}\label{eq:theta_star_def}
\theta_*=\frac{\rho_L-\rho_s q}{\rho_s-\rho_L q}\in\R\cup\{\pm\infty\}.
\end{align}
Assume $|q|<1$. Then, for all $\theta\in\R\cup\{\pm\infty\}$, it holds that
\begin{align}\label{eq:rho_star_def}
|F(\theta)| \leq F(\theta_*)= \sqrt{\frac{\rho_s^2+\rho_L^2-2q\rho_L\rho_s}{1-q^2}}\,.
\end{align}
\end{corollary}

The proof of Corollary \ref{cor:crosscor} is deferred to Appendix \ref{app:fopt}. Let us now comment on the assumption $|q|<1$. If $\bxl$ and $\bxs$ are perfectly correlated (i.e., $|q|=1$), then it is clear that the combined estimator $\bxc(\theta)$ cannot improve the performance for any value of $\theta$. On the contrary, when $|q|<1$,  Corollary \ref{lem:fopt} characterizes when the linear combination $\bxc$ strictly improves upon the performance of the individual estimators $\bxl$ and $\bxs$. Specifically, by denoting
\begin{align}\label{eq:rhomax_def}
\rho_{\max} := \max\{|\rho_L|,\rho_s\}\qquad\text{and}\qquad
p := \begin{cases}
\rho_s/\rho_L, &~\text{if}~ |\rho_L|\geq \rho_s, \\
\rho_L/\rho_s, &~\text{else},
\end{cases} 
%= \min\left\{\frac{\rho_s}{\rho_L},\frac{\rho_L}{\rho_s}\right\},
\end{align}
such that the right-hand side of \eqref{eq:rho_star_def} becomes
\begin{align}\label{eq:rho_*}
%\theta_*=\frac{1-pq}{p-q}\in\R\cup\{\pm\infty\}\qquad\text{and}\qquad
F(\theta_*) = \rho_{\max}\sqrt{1+\frac{(p-q)^2}{1-q^2}}\,,
\end{align}
it can be readily checked that $F(\theta_*)>\rho_{\max}$ provided that $|q|<1$ and $q\neq p$.

\begin{remark}[Optimization of preprocessing functions]
The linear estimator $\bxl$ and the spectral estimator $\bxs$ use the preprocessing functions $\Tc_L$ (cf. \eqref{eq:deflin}) and $\Tc_s$ (cf. \eqref{eq:defspect}), respectively. The choice of these functions naturally affects the performance of the two estimators, as well as, that of the combined estimator $F_*(\bx^{\rm L}, \bx^{\rm s})$. Lemmas \ref{lemma:pt_lin}-\ref{lemma:pt}, Theorem \ref{th:optimality} and Corollary \ref{lem:fopt} derive sharp asymptotics on the estimation performance that hold for \emph{any} choice of the preprocessing functions $\Tc_L$ and $\Tc_s$ satisfying our technical assumptions \textbf{(A1)}, \textbf{(A2)}, \textbf{(A3)}, \textbf{(B2)}. In Appendix \ref{app:optpre}, we briefly discuss how these results can be used to yield optimal such choices. Specifically, the optimal preprocessing function that maximizes the normalized correlation of the spectral estimator is derived in \cite{luo2019optimal}, see Appendix \ref{sec:spec_opt}. The optimal choice for the linear estimator is much easier to obtain and we outline the process in Appendix \ref{sec:lin_opt}. In Appendix \ref{sec:lin_vs_spec}, we combine these two results to derive a precise characterization of the sampling regimes in which the linear estimator surpasses the spectral estimator, and vice-versa. Finally, in Appendix \ref{sec:comb_opt}, we use Corollary \ref{lem:fopt} to cast the problem of optimally choosing $\Tc_L$ and $\Tc_s$ to maximize the correlation of the combined estimator $\bxc(\theta_*)$ as a function optimization problem. Solving the latter is beyond the scope of this paper, and it represents an intriguing future research direction.
\end{remark}
\section{Numerical Simulations}\label{sec:num}

This section validates our theory via numerical experiments and provides further insights on the benefits of the proposed combined estimator. 

First, we consider a setting in which the signal $\bx$ is uniformly distributed on the $d$-dimensional sphere. In this case, the limiting empirical distribution $P_X$ is Gaussian. Thus, the Bayes-optimal estimator $F_*(\bx^{\rm L},\bx^{\rm s})$ in \eqref{eq:Bayesopt} is linear and is given by $\hat\bx^{\rm c}(\theta_*)$, where $\theta_*$ is determined in Corollary \ref{lem:fopt}. For this scenario, we study in Figures \ref{fig:relu}, \ref{fig:eg2} and \ref{fig:Hermite}  the performance gain of $\hat\bx^{\rm c}(\theta_*)$ for three different measurement models and for various noise levels. 

Second, in Figure \ref{fig:Bayes} we consider a setting in which the the entries of $\bx$ are binary, drawn i.i.d. from the set $\{1, -1 \}$, such that the  empirical distribution  is of the form $P_X(1) = 1 - P_X(-1) = \sfp$, for some $\sfp \in (0,1)$. For this case, we compute the Bayes-optimal estimator $F_*(\bx_L,\bx_s)$ and compare it with the optimal linear combination $\hat\bx^c(\theta_*)$ for various choices of output functions.

%%%%%%%%%%%%%%%%%%%%%%%%%%%%%%%%%%%%%%%%%%

%%%%%%%%%%%%%%%%%%%%%%%%%%%%%%%%%%%%%%%%%%%%%%%%%%%%%%%%%%%%%%%

\subsection{Optimal Linear Combination}\label{subsec:numoptlin}

In Figure \ref{fig:relu}, we fix the input-output relation as $y_i = f(\inp{\ba_i}{\bx}) + \sigma z_i$, with $z_i\sim_{i.i.d.} \mathcal{N}(0,1)$ and $f(x)=\max\{x,-0.4x\}$ (cf. Figure \ref{fig:relu}(a)), and we investigate the performance gain of the proposed combined estimator for different values of the noise variance $\sigma^2$. Here, $\bx$ is generated via a standard Gaussian vector which is normalized such that $\|\bx\|_2=\sqrt{d}$. 
Also, $\ba_i \ \sim_{i.i.d.} \ \normal(\b0_d, \bI_d/d)$ and the pre-processing functions are chosen: $\cT_L(y) = y$ and $\cT_s(y) = \min\{y,3.5\}$.  Note that the empirical distribution of $\bx$ tends to a standard Gaussian distribution in the high-dimensional limit. Thus, following Section \ref{sec:opt_lin}, the optimal combined estimator is (asymptotically) linear and is given by \eqref{eq:combo_def} for $\theta=\theta_*$ chosen as in \eqref{eq:theta_star_def}. 
In Figure \ref{fig:relu}(b), we plot the percentage improvement $\frac{\rho_*-\rho_{\max}}{\rho_{\max}}\times 100$ as a function of the sampling ratio $\delta$, for three values of the noise variance $\sigma^2=0,0.4$ and $0.8$. Here, $\rho_*=F(\theta_*)$ defined in \eqref{eq:rho_star_def} and $\rho_{\max}=\max\{|\rho_L|,\rho_s\}.$ We observe that, as $\sigma^2$ increases, larger values of the sampling ratio $\delta$ are needed for the combined estimator to improve upon the linear and spectral estimators. However, for sufficiently large $\delta$, the benefit of the combined estimator is more pronounced for larger values of the noise variance. For instance, for $\sigma^2=0.8$ and large values of $\delta$, the percentage gain is larger than $10\%$. In Figure \ref{fig:relu}(c), we fix $\sigma^2=0.4$ and plot the correlations $\rho_L, \rho_s$ and $\rho_*$. The solid lines correspond to the theoretical predictions obtained by Lemma \ref{lemma:pt_lin}, Lemma \ref{lemma:pt} and Corollary \ref{lem:fopt}, respectively. The theoretical predictions are compared against the results of Monte Carlo simulations. For the simulations, we used $d=1000$ and averaged over $15$ independent problem realizations. In Figure \ref{fig:relu}(c)(Middle), we also plot the limit $q$ (cf. \eqref{eq:crosscorr}) of the cross-correlation $\frac{\inp{\bx^{\rm L}}{\bx^{\rm s}}}{\| \bx^{\rm L} \| \| \bx^{\rm s} \|}$  and the ratio $p$ in \eqref{eq:rhomax_def}. The corresponding values of the optimal combination coefficient $\theta_*$ are plotted in Figure \ref{fig:relu}(c)(Right). For values of $\delta$ smaller than the threshold for weak-recovery of the spectral method (where $\rho_s=0$), we observe that $\rho_*=\rho_L$ and $\theta_*=\infty$. However, for larger values of $\delta$, the value of the optimal coefficient $\theta_*$ is non trivial. Finally, Figure \ref{fig:relu}(d) shows the same plots as in Figure \ref{fig:relu}(c), but for $\sigma^2=0.8$.

\begin{figure}[p]
    \hspace{30pt}\subfigure[
    ]{
    \centering
    \includegraphics[scale=.33]{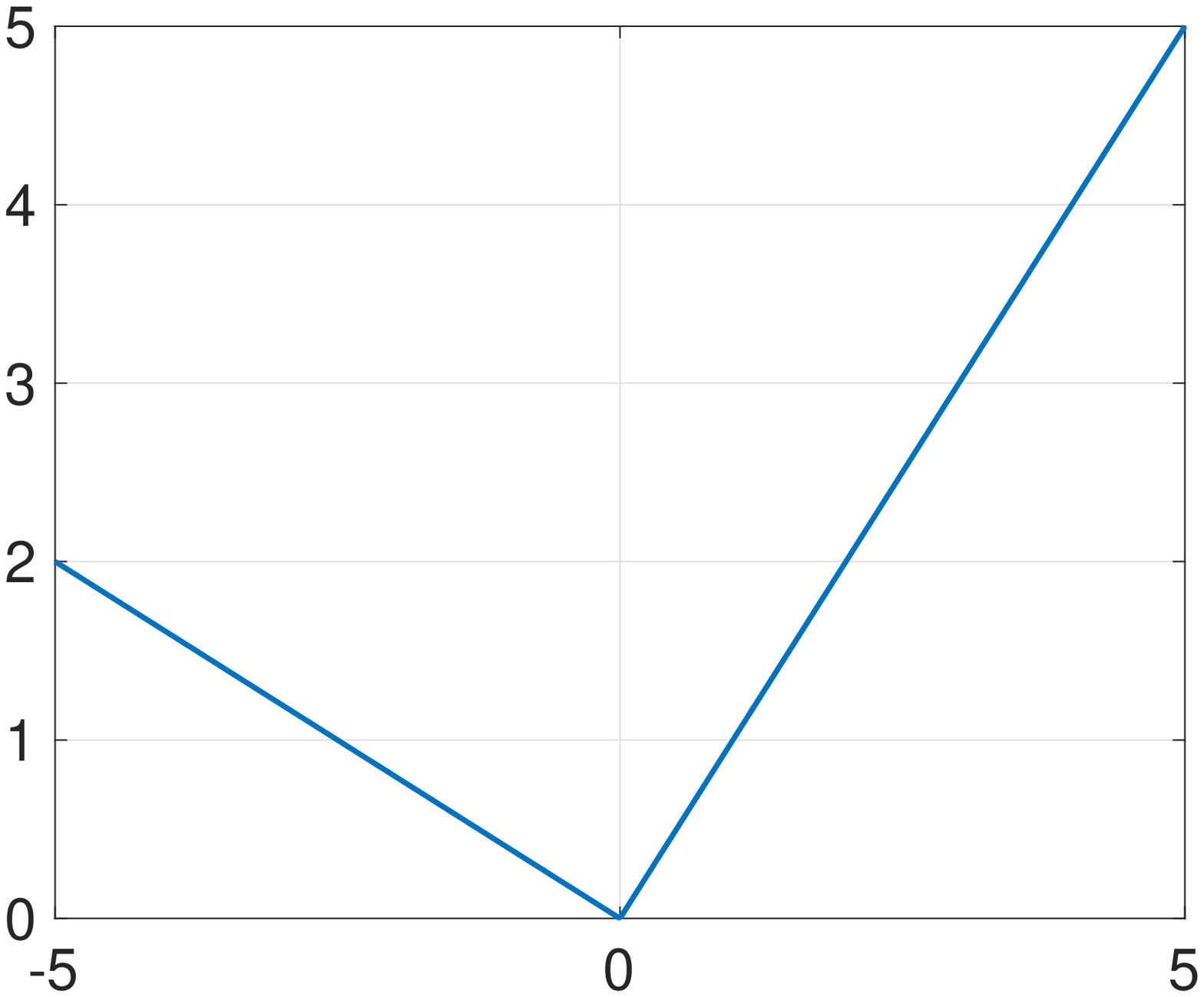}
    }
    \subfigure[]{
    \centering
    \includegraphics[scale=.37]{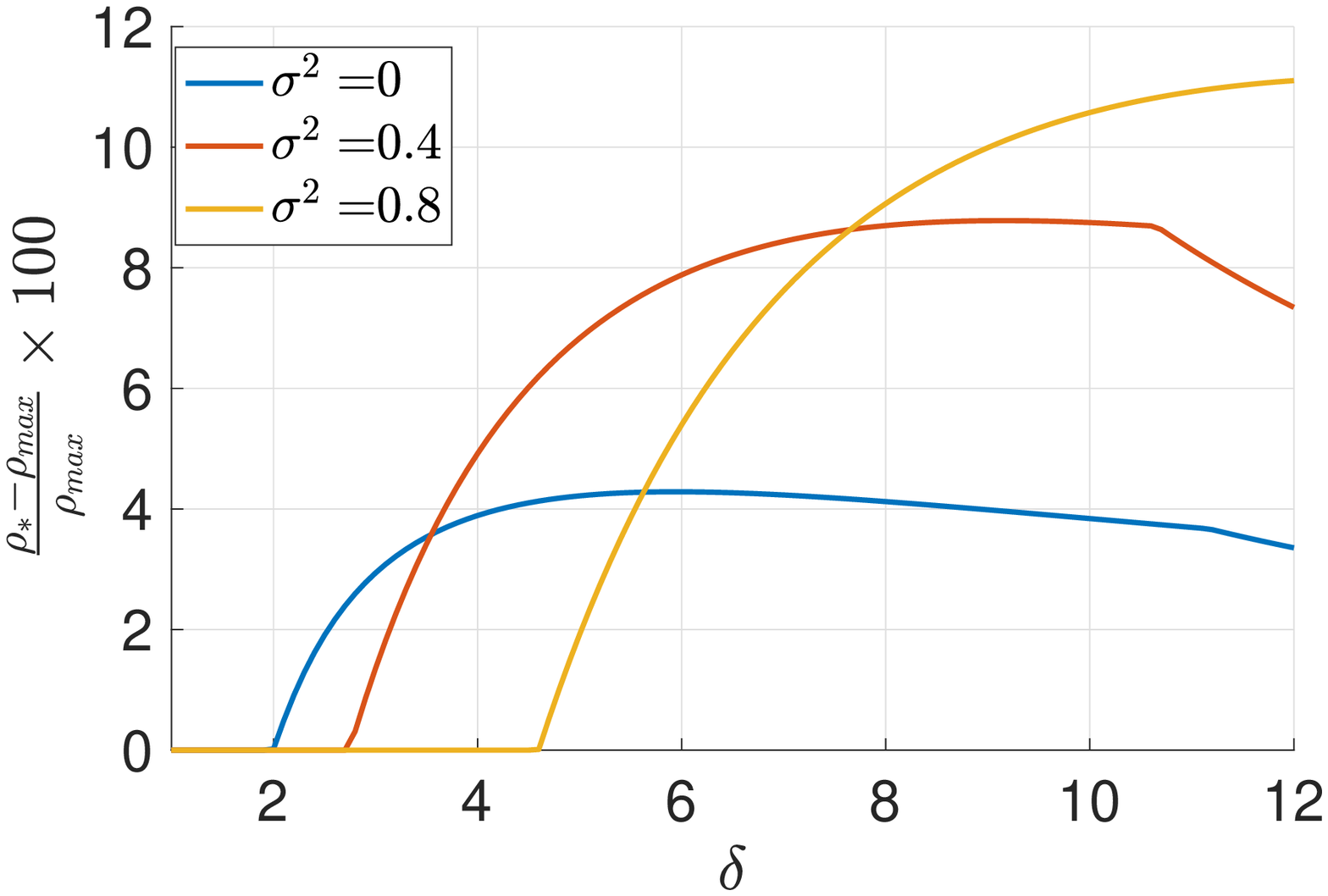}
    }
    \\
 \subfigure[$\sigma^2=0.4$]{
  \hspace{-18pt}  \includegraphics[scale=.42]{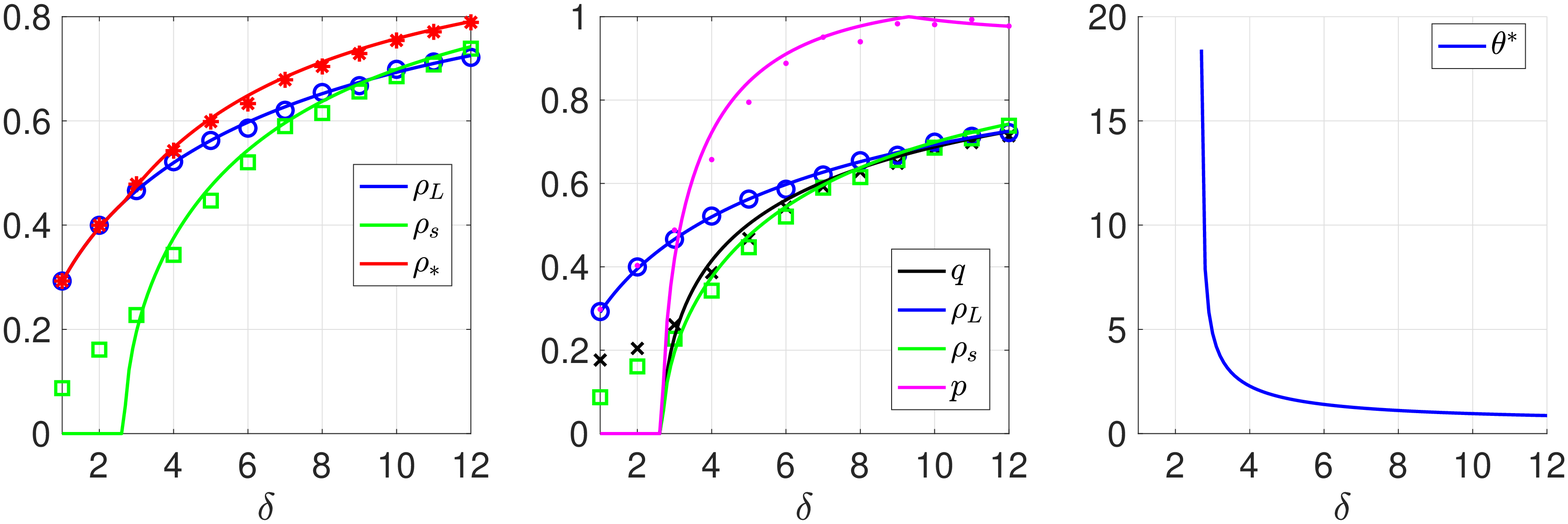}
    % \caption{$\sigma^2=0.4$, $d=250$, $N=15$}
    % \label{fig:eg1}
    }
    \\
    \subfigure[$\sigma^2=0.8$]{
  \hspace{-18pt}  \includegraphics[scale=.42]{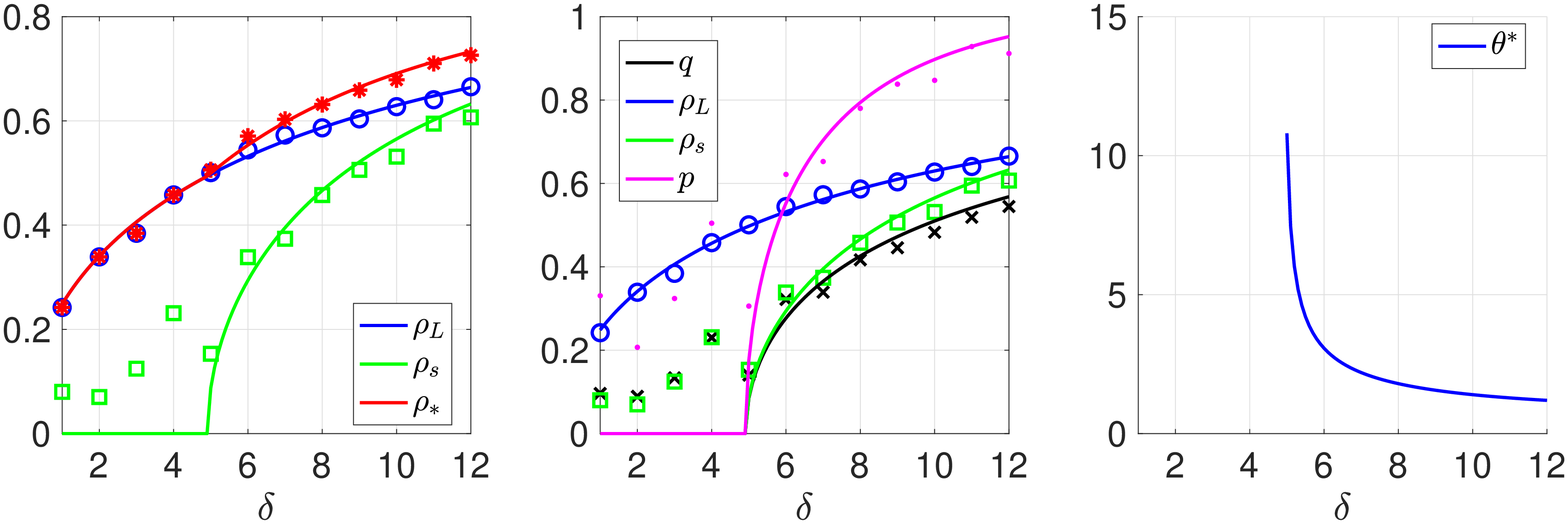}
    % \caption{$\sigma^2=0.4$, $d=250$, $N=15$}
    % \label{fig:eg1}
    }
    \caption{Monte-Carlo simulations and theoretical predictions for an i.i.d. Gaussian signal and  measurements model $y_i = f(\inp{\ba_i}{\bx}) + \sigma z_i, z_i\sim_{i.i.d.} \mathcal{N}(0,1)$. Here, $f(x)=\max\{x,-0.4x\}$ (cf. Fig. \ref{fig:relu}(a)).}
    \label{fig:relu}
\end{figure}

%%%%%%%%%%%%%%%%%%%%%%%%%%%%%%%%%%%%%%%%%%%%%%%%%%%%%%%%%%%%%%%%%%%
The setting of Figure \ref{fig:eg2} is the same as in  Figure \ref{fig:relu}, only now the input-output function is chosen as $f(x)=|x|\cdot\mathbf{1}_{\{|x|\geq 1.5\}} + x\cdot\mathbf{1}_{\{|x|<1.5\}}$. Comparing Figure \ref{fig:eg2}(b) to Figure \ref{fig:relu}(b), note that the benefit of the combined estimator is more significant for the link function studied here. Moreover, in contrast to  Figure \ref{fig:relu}(b), here, the percentage gain of the combined estimator takes its maximum value in the noiseless case: $\sigma^2=0$.

\begin{figure}[p]
    \hspace{30pt}
    \subfigure[
    %$f(x)=|x|\cdot\mathbf{1}_{\{|x|>=1.5\}} + x\cdot\mathbf{1}_{\{x<1.5\}}$
    ]{
    \centering
    \includegraphics[scale=.35]{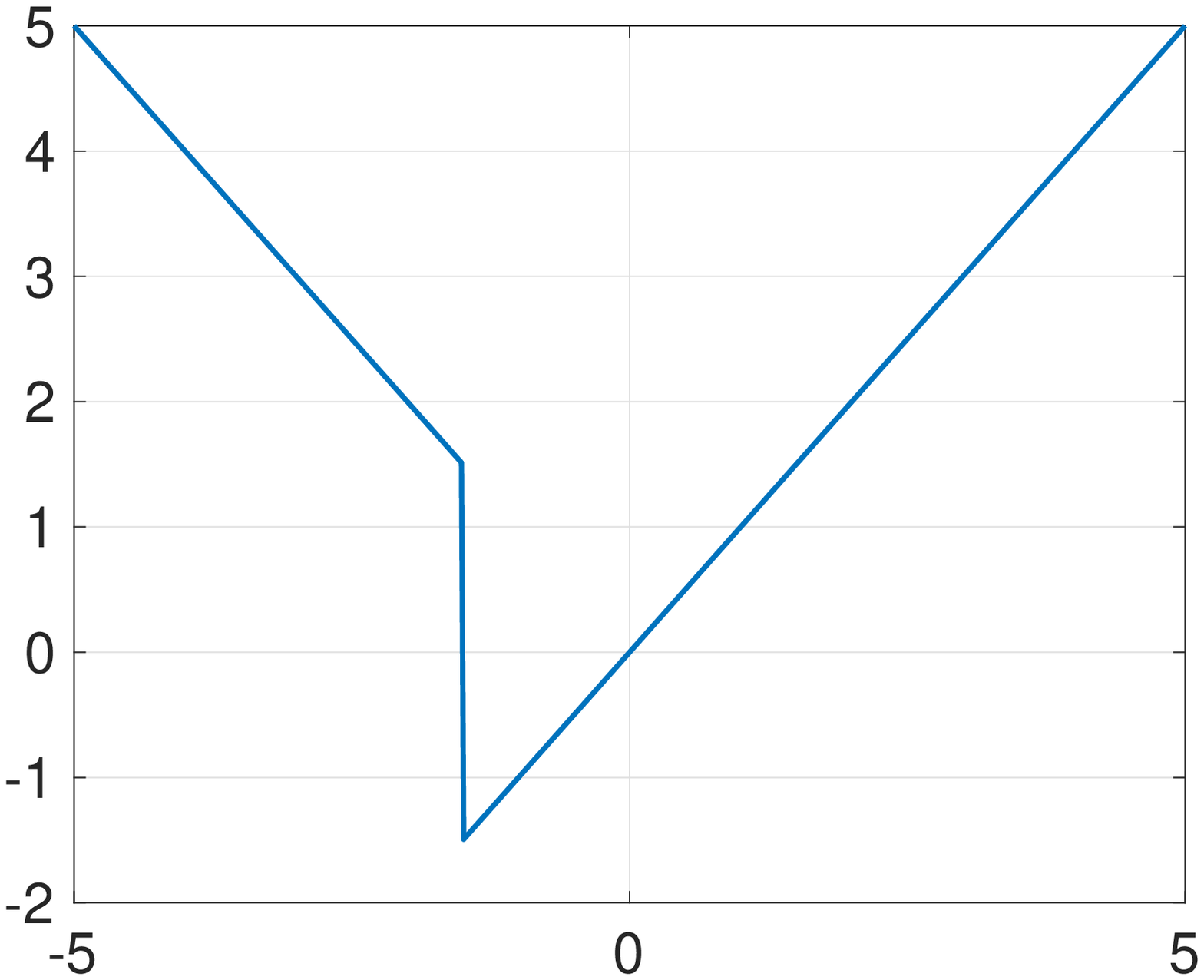}
    }
    \subfigure[]{
    \centering
    \includegraphics[scale=.35]{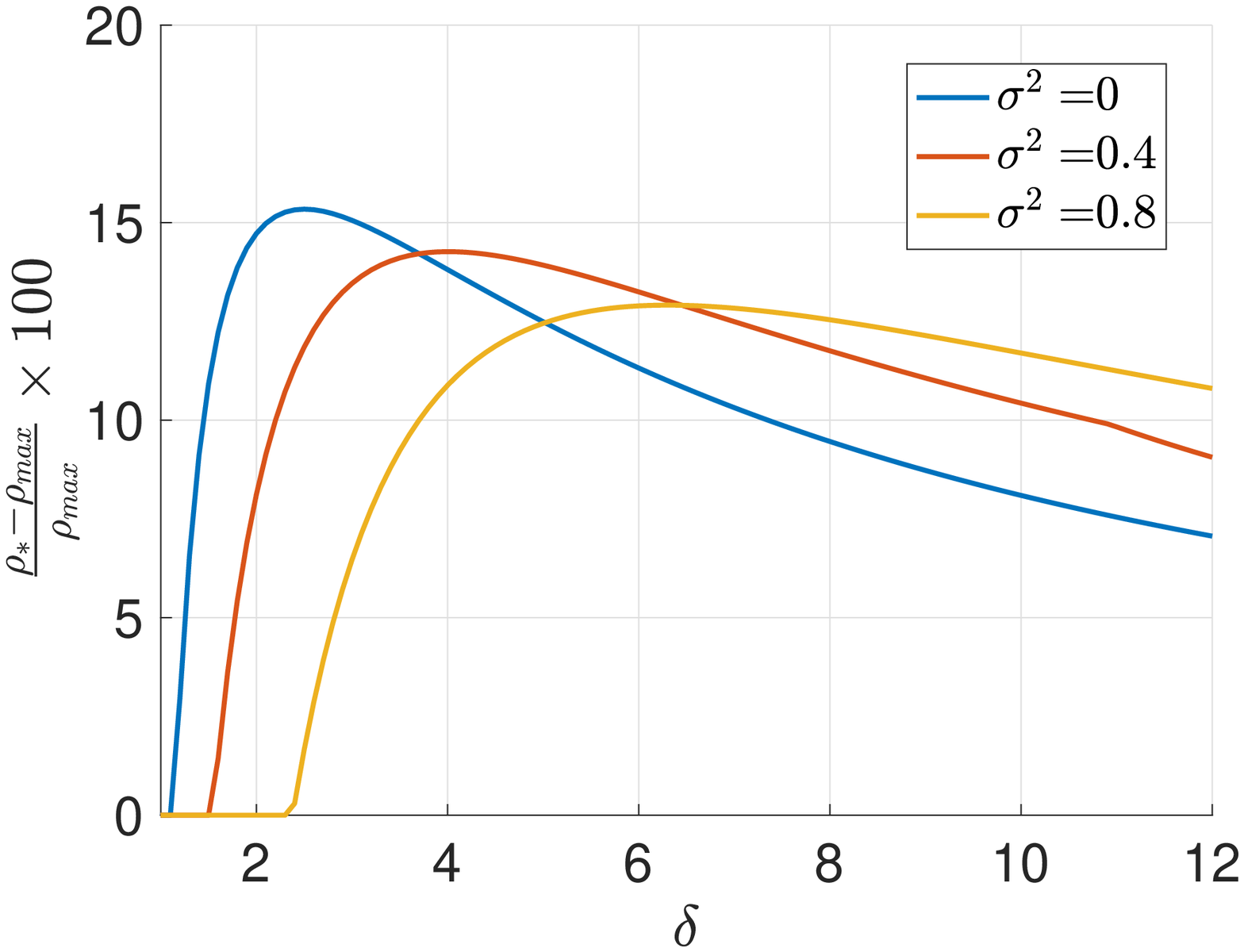}
    }
    \\
    \subfigure[$\sigma^2=0.4$]{
  \hspace{-18pt}  \includegraphics[scale=.42]{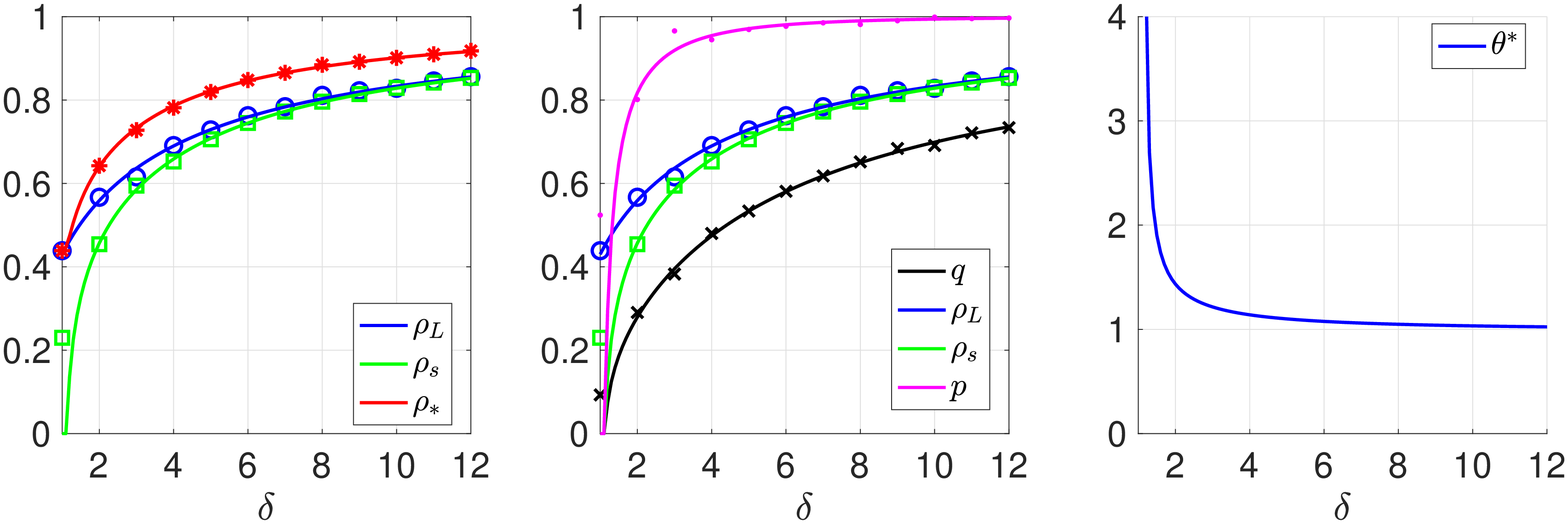}
    % \caption{$\sigma^2=0.4$, $d=250$, $N=15$}
    % \label{fig:eg1}
    }
    \\
    \subfigure[$\sigma^2=0.8$]{
  \hspace{-18pt}  \includegraphics[scale=.42]{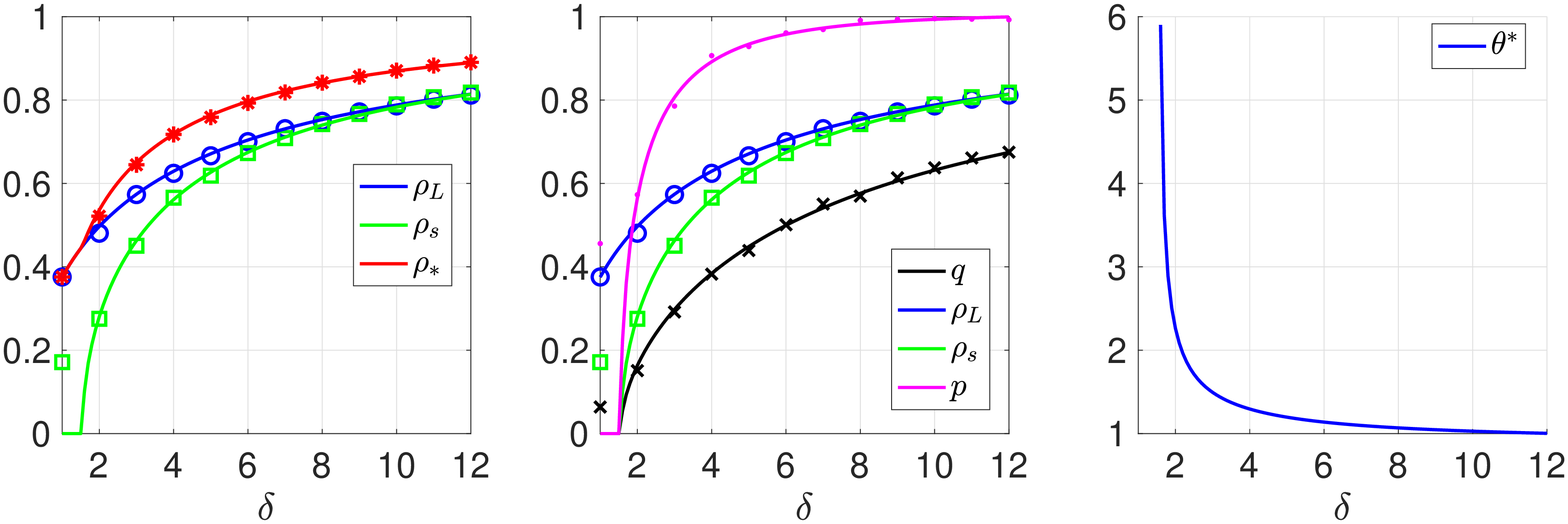}
    % \caption{$\sigma^2=0.4$, $d=250$, $N=15$}
    % \label{fig:eg1}
    }
    \caption{Monte-Carlo simulations and theoretical predictions for an i.i.d. Gaussian signal and measurements model $y_i = f(\inp{\ba_i}{\bx}) + \sigma z_i, z_i\sim_{i.i.d.} \mathcal{N}(0,1)$.  Here, $f(x)=|x|\cdot\mathbf{1}_{\{|x|\ge 1.5\}} + x\cdot\mathbf{1}_{\{|x|<1.5\}}$ (cf. Fig. \ref{fig:eg2}(a)).}
    \label{fig:eg2}
\end{figure}

% \begin{figure}[t]
%     \subfigure[]{
%     \centering
%     \includegraphics[scale=.4]{Figures/fine_link.eps}
%     }
%     \subfigure[]{
%     \centering
%     \includegraphics[scale=.33]{Figures/fine_gain0to1.eps}
%     }
%     \\
%     \subfigure[$\sigma^2=0$, $d=250$, $N=15$]{
%     \includegraphics[scale=.4]{Figures/fine_all3_sig2=00.eps}
%     % \caption{$\sigma^2=0.4$, $d=250$, $N=15$}
%     % \label{fig:eg1}
%     }
%     \\
%     \subfigure[$\sigma^2=0.4$, $d=250$, $N=15$]{
%     \includegraphics[scale=.4]{Figures/fine_all3_sig2=040.eps}
%     % \caption{$\sigma^2=0.4$, $d=250$, $N=15$}
%     % \label{fig:eg1}
%     }
%     \caption{\ct{Here $y=f(g)+\sigma\mathcal{N}(0,1)$, for the depicted link function $f$. See (b) on how performance gain depends on noise.\mm{Why is the plot cut at $\delta=4$? Can we also add smaller $\delta$ to see what happens there? Also, how about even larger noise? This seems to give good improvement at large $\delta$.}\mm{What is $N$?}}\ct{Move to 5.}}
%     \label{fig:eg2}
% \end{figure}

%%%%%%%%%%%%%%%%%%%%%%%%%%%%%%%%%%%%%%%%%%%%%%%%%%%%%%%%%%%%%%%%%%%
\begin{figure}[p]
    \hspace{30pt}\subfigure[]{
    \centering
    \includegraphics[scale=.35]{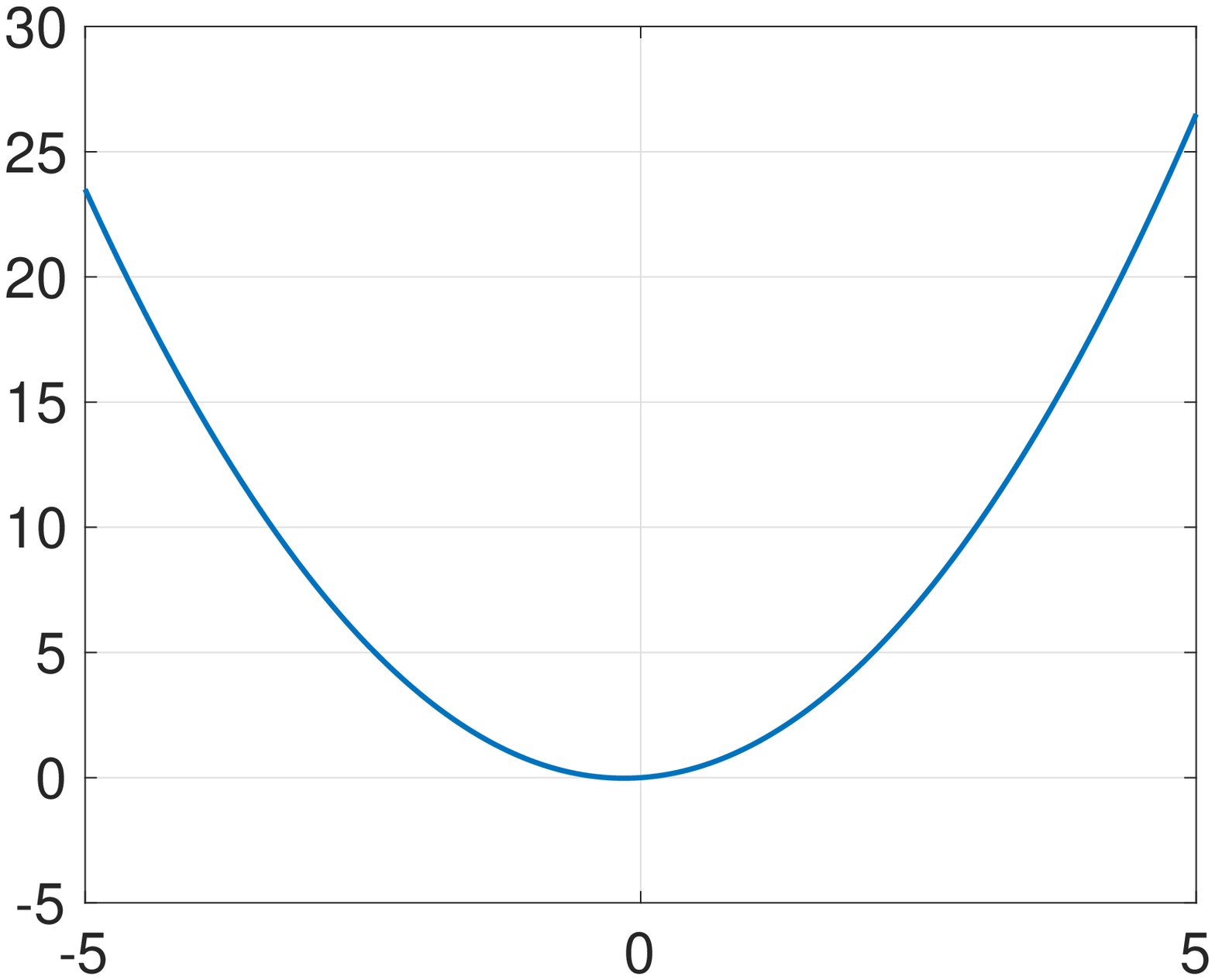}
    }
    \subfigure[]{
    \centering
    \includegraphics[scale=.35]{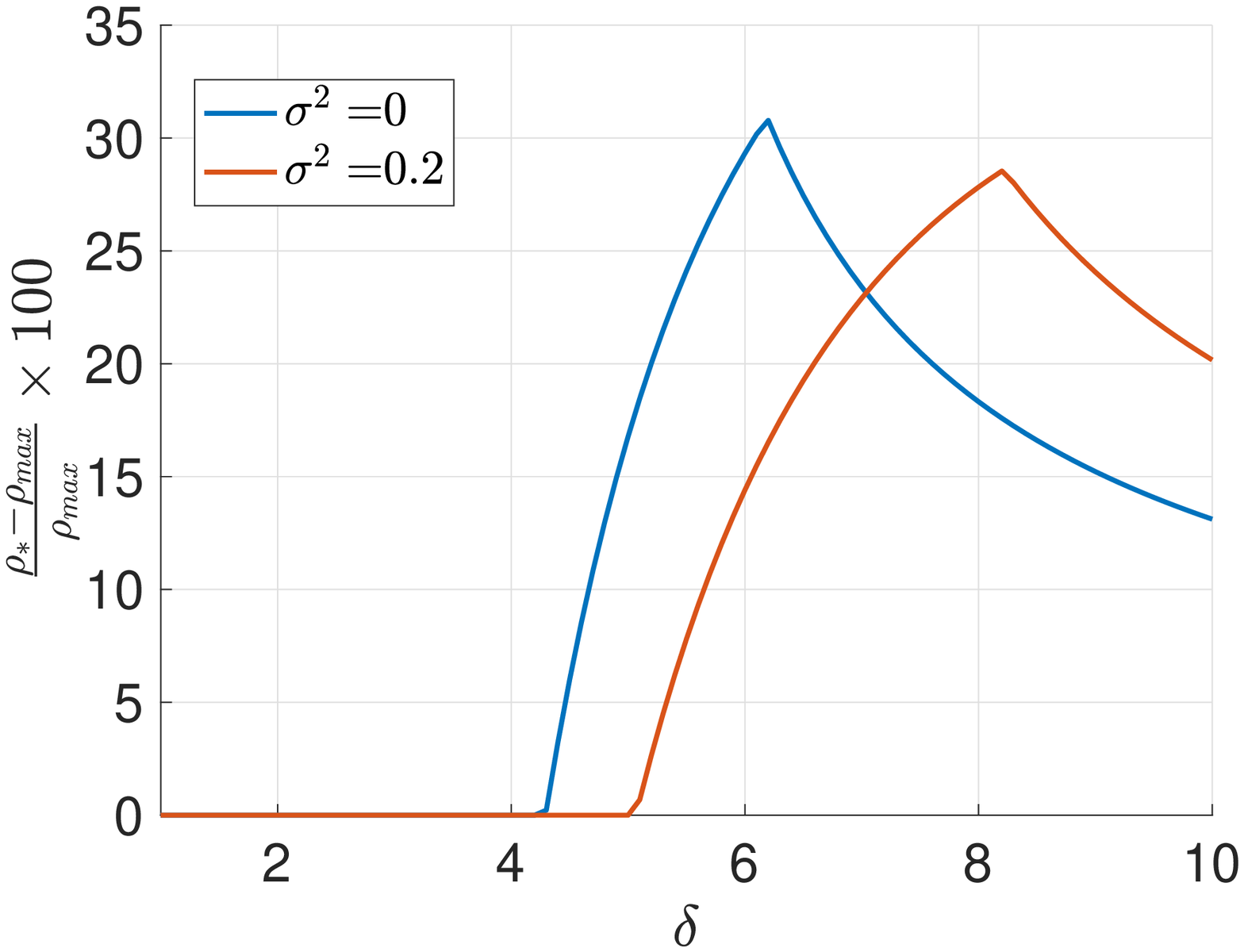}
    }
    \\
    \subfigure[$\sigma^2=0$]{
 \hspace{-18pt}   \includegraphics[scale=.42]{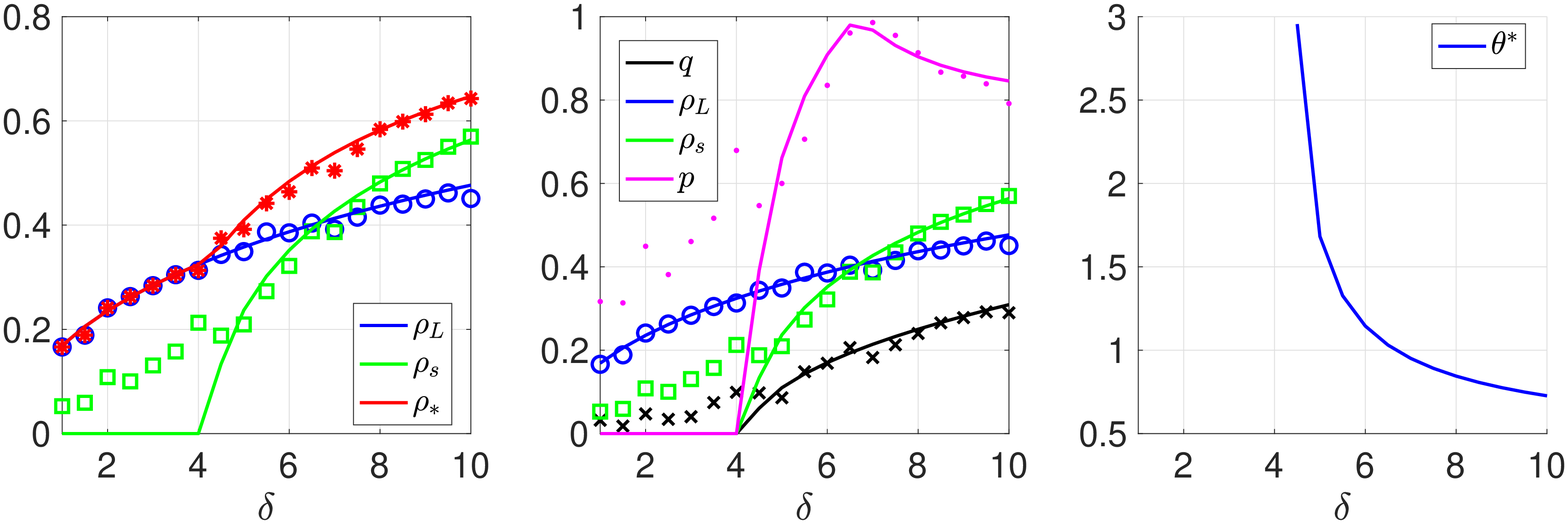}
    }
    \subfigure[$\sigma^2=0.2$]{
 \hspace{-18pt}   \includegraphics[scale=.42]{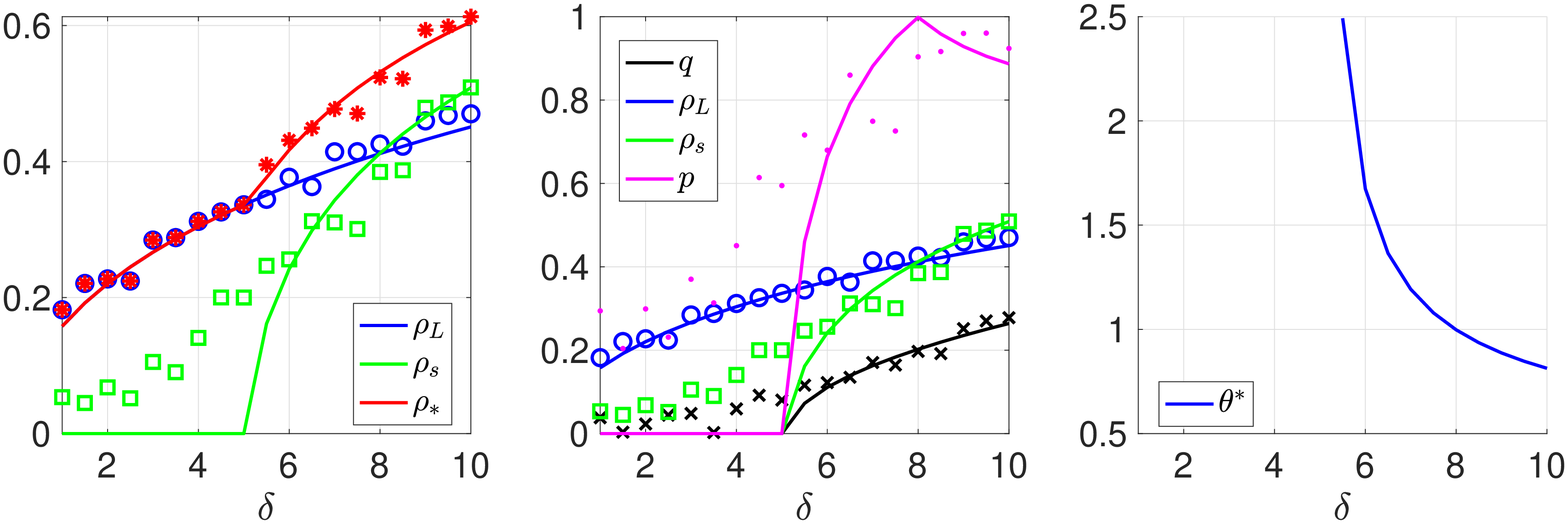}
    }
    \caption{{Monte-Carlo simulations and theoretical predictions for an i.i.d. Gaussian signal and   measurements model $y_i = f(\inp{\ba_i}{\bx}) + \sigma z_i, z_i\sim_{i.i.d.} \mathcal{N}(0,1)$. Here, $f(x)=0.3x+x^2$ (cf. Fig. \ref{fig:Hermite}(a)).}}
    \label{fig:Hermite}
\end{figure}

In Figure \ref{fig:Hermite}, we repeat the experiments of Figures \ref{fig:relu} and \ref{fig:eg2}, but for $f(x)= 1+ 0.3x + (x^2-1)$. Compared to the two functions studied in Figures \ref{fig:relu} and \ref{fig:eg2}, in Figure \ref{fig:Hermite} we observe that the performance gain is significantly larger and reaches values up to $~30\%$. This can be argued by considering the expansion of the input-output functions on the basis of the Hermite polynomials, i.e., $f(x) = \sum_{i=0}^\infty H_i h_i(x)$, where $h_i(x)$ is the $i$th-order Hermite polynomial with leading coefficient $1$ and $H_i = \frac{1}{i!}\E_{G\sim\mathcal{N}(0,1)}\{f(G)h_i(G)\}$. Specifically, recall that the first three Hermite polynomials are as follows: $h_0(x)=1$, $h_1(x)=x$ and $h_2(x)=x^2-1$. Thus, for 
$f(x)= 1+ 0.3x + (x^2-1)$, only the first three  coefficients $\{H_i\}, i=0,1,2$, are  non-zero. To see the relevance of these coefficients to the linear and spectral estimators note that for identity pre-processing functions it holds that
$
\E\{GZ_L\} = \E\{GY\} = \E\{Gf(G)\} = H_1
$ 
and 
$
\E\{(G^2-1)Z_s\} =\E\{(G^2-1)f(G)\} = 2H_2. 
$
Thus, it follows directly from Lemma \ref{lemma:pt_lin} that $\rho_L=0$ if the first Hermite coefficient $H_1$ is zero. Similarly, it can be shown using Lemma \ref{lemma:pt} that $\rho_s=0$ if the second Hermite coefficient $H_2$ is zero; in fact, the threshold of weak recovery of the spectral method is infinity in this case (see \eqref{eq:delta_star} in Appendix \ref{app:optpre}). Intuitively, the linear and spectral estimators exploit the energy of the output function corresponding to the Hermite polynomials of first- and second-order, respectively; see also \cite{dudeja2018learning,thrampoulidis2019lifting}. In this example, we have chosen $f(x)$ such that all of its energy is concentrated on the Hermite polynomials of order up to two. 

As a final note, from the numerical results in Figures \ref{fig:relu}, \ref{fig:eg2} and \ref{fig:Hermite}, we observe that the proposed optimal combination leads to a performance improvement only if both the linear and the spectral estimators are asymptotically correlated with the signal. This is because the signal prior is Gaussian (see Remark \ref{rmk:special2}). In contrast, as we will see in the next section, when the signal prior is binary, the combined estimator provides an improvement even when only the linear estimator is effective.

\subsection{Bayes-optimal Combination}\label{sec:num_bayes}
In Figures \ref{fig:Bayes}(a,b) and \ref{fig:Bayes}(c,d) we consider the same setting as in Figures \ref{fig:eg2} and \ref{fig:Hermite}, respectively. However, here each entry of $\bx$ takes value either $+1$ or $-1$. Each entry is  chosen independently according to the distribution 
$P_X(1) = 1 -P_X(-1) =\sfp$, for $\sfp \in (0,1)$. Thus, the Bayes optimal combination $\hat\bx^{\rm mmse}=F_*(\bx^{\rm L},\bx^{\rm s})$ is not necessarily linear as in the Gaussian case. In Appendix \ref{sec:Bayes_Bern}, we compute the Bayes-optimal estimator $F_*(x_L,x_s)$ (cf.  \eqref{eq:Bayesopt}) for the setting considered here. Then, we use the prediction of Theorem \ref{th:optimality}  to plot in solid black lines the normalized correlation of $\hat\bx^{\rm mmse}$ with $\bx$ (i.e., $\rho_*^{\rm mmse} = \rho_*$ in \eqref{eq:optres0}). The theoretical predictions (solid lines) are compared against the results of Monte Carlo simulations (markers). Moreover, we compare the optimal performance against those of the linear estimator (cf. $\rho_L$), the spectral estimator  (cf. $\rho_s$) and the optimal linear combination $\hat\bx^{\rm c}(\theta_*)$ (cf. $\rho_*^{\rm linear})$. We have chosen $\sfp=0.3$ in Figures \ref{fig:Bayes}(a,c) and $\sfp=0.5$ in Figures \ref{fig:Bayes}(b,d). Note that the optimal \emph{linear} combination provides a performance improvement only for the values of $\delta$ s.t. $\rho_L>0$ and $\rho_s>0$. On the contrary, $\rho_*^{\rm mmse}$ is strictly larger than $\rho_L$ even when $\rho_s=0$.

\begin{figure}[p]
    \subfigure[$\sfp=0.3$]{
    \centering
    \includegraphics[scale=.4]{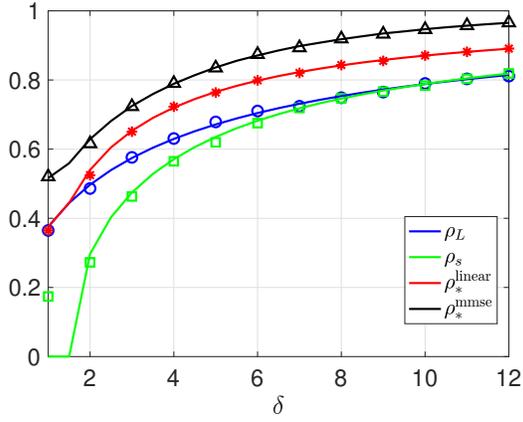}
    }
    \subfigure[$\sfp=0.5$]{
    \centering
    \includegraphics[scale=.4]{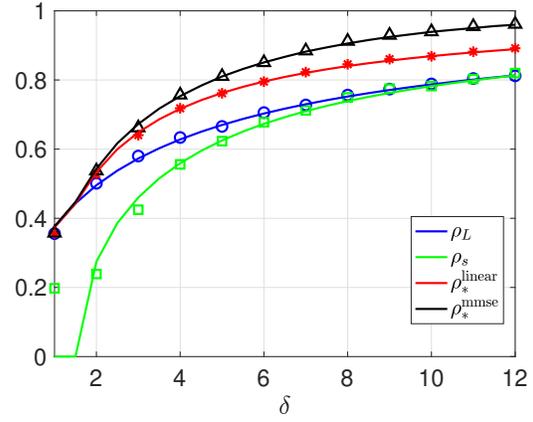}
    }
    \\
    \subfigure[$\sfp=0.3$]{
    \includegraphics[scale=.4]{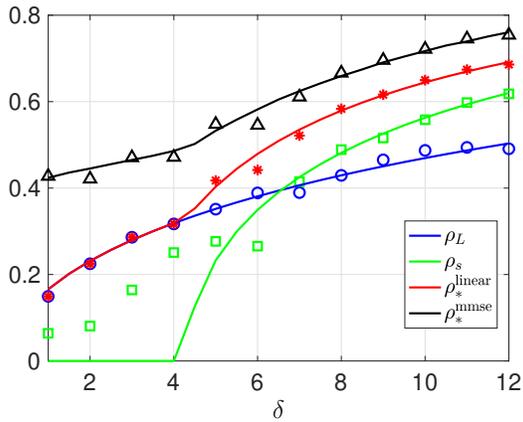}
    % \caption{$\sigma^2=0.4$, $d=250$, $N=15$}
    % \label{fig:eg1}
    }
    \subfigure[$\sfp=0.5$]{
    \includegraphics[scale=.4]{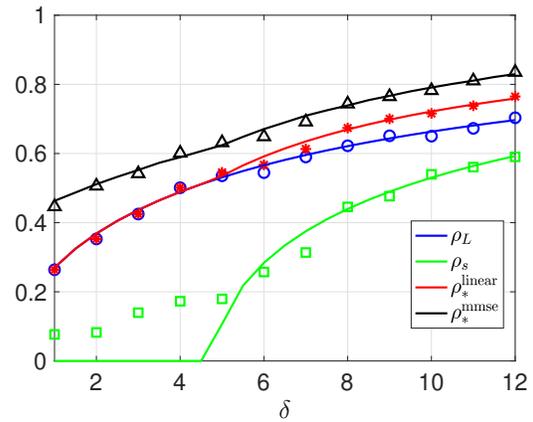}
    % \caption{$\sigma^2=0.4$, $d=250$, $N=15$}
    % \label{fig:eg1}
    }
    \caption{Bayes-optimal combination vs optimal linear combination for a binary prior $P_X(1)= 1- P_X(-1)=\sfp$. In (a,b) (resp. (c,d)) the setting is otherwise the same as in Figure \ref{fig:eg2} (resp. Figure \ref{fig:Hermite}).}
    \label{fig:Bayes}
\end{figure}

\begin{figure}[p]
	\begin{center}
		\includegraphics[width=0.85\textwidth]{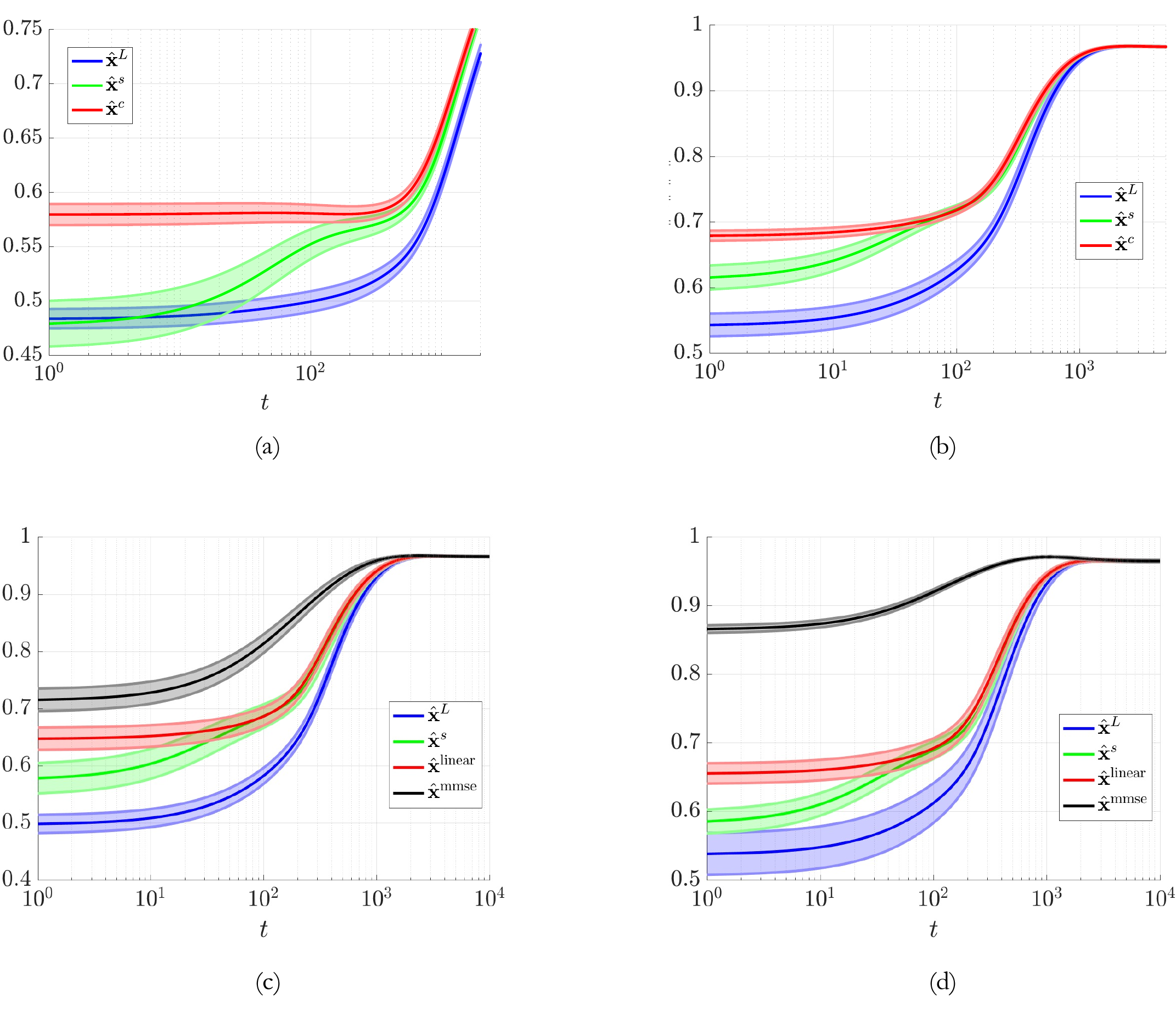}
	\end{center}
	\caption{Normalized correlation of gradient descent iterates for different initializations: \emph{(i)} linear $\hat{\bx}^{\rm L}$, \emph{(ii)} spectral $\hat{\bx}^{\rm s}$, \emph{(iii)} optimal linear combination $\hat{\bx}^{\rm linear}$ (Bayes-optimal for Gaussian signal prior), and \emph{(iv)} Bayes-optimal combination $\hat{\bx}^{\rm mmse}$. Depicted are means (solid lines) and standard errors (shaded regions) over $10$ Monte Carlo realizations. In all cases, $y_i=f( \langle \ba_i , \bx \rangle)+z_i, i\in[n]$, $z_i\sim\mathcal{N}(0,0.2)$ and $d=250$. (a) $f(x)=0.3 x + 0.5(x^2-1),$ $\delta=3$, Gaussian prior. (b)  $f(x)=0.3 x + 0.5x^2,$ $\delta=5$, Gaussian prior. (c,d) $f(x)=0.3 x + 0.5x^2,$ $\delta=5$,  prior over $\{+1, -1 \}$ with $P_X(1)=\sfp=0.3$ for (c) and $P_X(1)=\sfp=0.1$ for (d).
	%\textbf{(a)} Here, $y_i= 0.3 \langle \mathbf{a}_i , \mathbf{x} \rangle + 0.5(\langle \mathbf{a}_i , \mathbf{x} \rangle^2-1)+z_i, i\in[n], z_i\sim\mathcal{N}(0,0.2)$, $d=1000$ and $\delta=3$. \textbf{(b)} Here, $y_i= 0.3 \langle \mathbf{a}_i , \mathbf{x} \rangle + 0.5\langle \mathbf{a}_i , \mathbf{x} \rangle^2+z_i, i\in[n], z_i\sim\mathcal{N}(0,0.2)$, $d=250$ and $\delta=5$.
	}
	\label{fig:GD}
\end{figure}

% \begin{figure}
%     \centering
%     \includegraphics[scale=.4]{Figures/gains_vs_H1.eps}
%     \caption{\ct{Performance gain as a function of the first hermite coefficient}}
%     \label{fig:my_label}
% \end{figure}

\subsection{The Combined Estimator as Initialization for Local Algorithms}
As mentioned in the introduction, the initial estimates of the signal obtained by either the linear/spectral methods or our proposed combined estimator, can then be refined via  local algorithms such as gradient descent (GD). The theory and numerical simulations in the previous sections showed the superiority of our proposed combined estimator $\hat\bx^{\rm mmse}=F_*(\bx^{\rm L},\bx^{\rm s})$ over $\bx^{\rm L}$ and $\bx^{\rm s}$ in terms of correlation with the true signal. In Figure \ref{fig:GD}, by plotting the correlation of GD iterates $\bx_t, t\geq 1$ for different initializations, we show numerically how this improvement translates to improved performance of gradient descent refinements. Specifically,  we ran GD on squared error loss with step size $1/2$, that is $\bx_{t+1}=\bx_{t}-\sum_{i\in[n]}f^\prime(\langle \ba_i , \bx_t \rangle)\left(y_i-f(\langle \ba_i , \bx_t \rangle)\right)$. Here, $f$ is the output function described in the caption of the figure.

In Figures \ref{fig:GD}(a,b), the true signal has  i.i.d. Gaussian entries, thus the linear combined estimator $\hat\bx^{\rm c}$ is the optimal combination in terms of correlation performance. We observe that for two different choices of  output function and sampling ratio (see caption), GD with linear or spectral initialization requires hundreds of iterations to reach the performance of the combined estimator. In Figures \ref{fig:GD}(c,d), the true signal has entries in $\{+1, -1\}$, chosen independently  with $P_X(1)=\sfp=0.3$ and $P_X(-1)=\sfp=0.1$, respectively. Here, the linear combined estimator is sub-optimal (but still improves upon linear/spectral), thus we also compute and study the Bayes-optimal estimator. Interestingly, while for both priors, GD converges to the same correlation as $t$ increases, for $\sfp=0.1$, GD achieves higher correlation if stopped early. It is a fascinating, albeit challenging, question to better understand the evolution of the GD trajectory as a function of the initialization, signal prior and output function.

% \begin{figure}[p]
%     \hspace{30pt}
%     \subfigure[
%     ]{
%     \centering
%     \includegraphics[scale=1]{GD1.pdf}
%     }
%     \subfigure[]{
%     \centering
%     \includegraphics[scale=1]{GD2.pdf}
%     }
% \end{figure}

%%%%%%%%%%%%%%%%%%%%%%%%%%%%%%%%%%%%%%%%%%%%%%%%%%%

\section{Proof of Theorem \ref{th:W2conv2}} \label{sec:proof_mainthm}

\subsection{Proof Sketch}\label{subsec:sketch}

The proof of Theorem \ref{th:W2conv2} is based on the design and analysis of a  \emph{generalized approximate message passing} (GAMP) algorithm. GAMP is a  class of iterative algorithms proposed by Rangan \cite{RanganGAMP} for estimation in generalized linear models. A GAMP algorithm is defined in terms of a sequence of Lipschitz functions $f_t:\mathbb R\to \mathbb R$ and $g_t:\mathbb R \times \reals \to \mathbb R$, for $t \geq 0$. For $t\ge 0$, the GAMP iteration computes:
\begin{equation}
\begin{split}
\bu^t &= \frac{1}{\sqrt{\delta}}\bA f_t(\bv^t)-\sb_t g_{t-1}(\bu^{t-1}; \by),\\
\bv^{t+1} &= \frac{1}{\sqrt{\delta}}\bA^\sT g_t(\bu^t; \by) - \sc_t f_t(\bv^t).
\end{split}
\label{eq:GAMP}
\end{equation}
 Here, the functions $f_t$ and $g_t$ are understood to be applied component-wise, i.e., $f_t(\bv^t)=(f_t(v^t_1)$, $\ldots, f_t(v^t_d))$ and $g_t(\bu^t; \by)=(g_t(u^t_1; y_1), \ldots, g_t(u^t_n; y_n))$. The scalars $\sb_t, \sc_t$ are defined as
\begin{equation}
\sb_t =\frac{1}{n}\sum_{i=1}^d f_t'(v_i^t), \qquad
\sc_t = \frac{1}{n}\sum_{i=1}^n g_t'(u_i^t; y_i),
\label{eq:GAMP_onsager}
\end{equation}
where $g_t'(\cdot, \cdot)$ denotes the derivative with respect to the first argument. The iteration \eqref{eq:GAMP} is initialized with 
\beq 
\bu^{0}= c \bfone_n, \qquad   \bv^1 = \frac{1}{\sqrt{\delta}} \bA^{\sT} g_{0}(\bu^{0} ; \by),
\label{eq:GAMP_init}
\eeq
for some constant $c>0$. Here, $\bfone_n \in \reals^n$ denotes the all-ones vector.

A key feature of the GAMP algorithm \eqref{eq:GAMP} is that the asymptotic empirical distribution of its iterates can be succinctly characterized via a deterministic recursion, called \emph{state evolution}. Hence, the performance of the high-dimensional problem involving the iterates $\bu^t, \bv^t$ is captured by a scalar recursion. Specifically, this result  gives that for $t \geq 1$, the empirical distributions of $\bu^t$ and $\bv^t$ converge in $W_k$ distance to the laws of the random variables $U_t$ and $V_t$, respectively, with 
\begin{align}
    U_t & \equiv \mu_{U,t} G + \sigma_{U,t} W_{U,t},  \label{eq:Ut_def}\\
    V_t & \equiv \mu_{V,t} X + \sigma_{V,t} W_{V,t},    \label{eq:Vt_def}
\end{align}
where $(G, W_{U,t}) \sim_{\text{i.i.d.}} \normal(0,1)$. Similarly, $X \sim P_X$ and $W_{V,t} \sim \normal(0,1)$ are independent. The deterministic parameters $(\mu_{U,t}, \sigma_{U,t}$, $\mu_{V,t}, \sigma_{V,t} )$ are computed via the recursion \eqref{eq:SE_def} detailed in Section \ref{subsec:GAMP}, and the formal statement of the state evolution result is contained in Proposition \ref{prop:GAMP_SE} (again in Section \ref{subsec:GAMP}). 

Next, in Section \ref{eq:GAMP_power}, we show that a suitable choice of the functions $f_t, g_t$ leads to a GAMP algorithm such that \emph{(i)} $\bv^1 = \sqrt{d} \,  \bxl$, and \emph{(ii)} $\bv^t$ is aligned with $\sqrt{d} \, \bxs$ as $t$ grows large. 
Specifically, choosing $g_0(u; y)=\mathcal T_L(y)/\sqrt{\delta}$ in \eqref{eq:GAMP_init} immediately gives $\bv^1 = \sqrt{d} \,  \bxl$. 
%, see \eqref{eq:g0def}-\eqref{eq:v1_sp}. 
In order to approach the spectral estimator as $t\to\infty$, we pick $f_t, g_t$ to be linear functions; see \eqref{eq:ft_gt_choice}. The idea is that, with this choice of $f_t$ and $g_t$, the GAMP iteration is effectively a power method. Let us now briefly discuss why this is the case. With the choice of $f_t, g_t$ in \eqref{eq:ft_gt_choice}, the GAMP iteration can be expressed as 
\begin{equation}\label{eq:newGAMP}
    \begin{split}
    \bu^t &= 
    \frac{1}{\sqrt{\delta} \, \beta_t} \big[ \bA \bv^t \, -  \, \bZ \bu^{t-1} \big], \\
\bv^{t+1} &=  \bA^\sT \bZ \bu^t - \frac{\sqrt{\delta}}{\beta_t} \, \E\{ Z \} \,  \bv^t,
    \end{split}
\end{equation}
where $\bZ = \diag(\cT(y_1), \ldots, \cT(y_n))$, $Z=\mathbb E\{\cT(Y)\}$, and the function $\mathcal T:\mathbb R\to\mathbb R$ is defined later in terms of the spectral preprocessing function $\cT_s$; see \eqref{eq:deftildeZ}. Then, Lemma \ref{lem:fixed_pts} analyzes the fixed points of the state evolution of the GAMP algorithm \eqref{eq:newGAMP}, and Lemma \ref{lem:evec_emp_dist} proves that in the high-dimensional limit, the vector $\bv^t$ tends to align with the principal eigenvector of the matrix $\bM_n=  \bA^\sT \bZ(\bZ+\delta\mathbb E\{Z(G^2-1)\}\bI_n)^{-1} \bA$ as $t \to \infty$. Here, we provide a heuristic sanity check for this last claim. Assume that the iterates $\bv^t$ and $\bu^t$ converge to the limits $\bv^\infty$ and $\bu^\infty$, respectively, in the sense that $\lim_{t\to\infty}\frac{1}{d}\|\bv^t-\bv^\infty\|^2=0$ and $\lim_{t\to\infty}\frac{1}{n}\|\bu^t-\bu^\infty\|^2=0$. Then, from \eqref{eq:newGAMP}, the limits $\bv^\infty$ and $\bu^\infty$ satisfy
\begin{equation}\label{eq:newGAMP2}
    \begin{split}
    \bu^\infty &= 
    \frac{1}{\sqrt{\delta} \, \beta_\infty} \big[ \bA \bv^\infty \, -  \, \bZ \bu^{\infty} \big], \\
\bv^{\infty} &=  \bA^\sT \bZ \bu^\infty - \frac{\sqrt{\delta}}{\beta_\infty} \, \E\{ Z \} \,  \bv^\infty.
    \end{split}
\end{equation}
Furthermore, from the analysis of the fixed points of state evolution of Lemma \ref{lem:fixed_pts}, we obtain that $\beta_\infty=\sqrt{\delta}\mathbb E\{Z(G^2-1)\}$. Thus, after some manipulations, \eqref{eq:newGAMP2} can be rewritten as
\begin{equation}\label{eq:newGAMP3}
\begin{split}
    \bu^\infty &=
     (\bZ+\delta\mathbb E\{Z(G^2-1)\}\bI_n)^{-1} \bA \bv^\infty\\
\left(1+\frac{ \E\{ Z \}}{\mathbb E\{Z(G^2-1)\}} \,\right)\bv^{\infty} &=  \bA^\sT \bZ (\bZ+\delta\mathbb E\{Z(G^2-1)\}\bI_n)^{-1} \bA \bv^\infty.
\end{split}    
\end{equation}
Hence, $\bv^\infty$ is an eigenvector of the matrix $\bA^\sT \bZ (\bZ+\delta\mathbb E\{Z(G^2-1)\}\bI_n)^{-1} \bA$, and the GAMP algorithm is effectively performing a power method. Finally, we choose $\mathcal T$ (and consequently $\bZ$) so that $\bZ (\bZ+\delta\mathbb E\{Z(G^2-1)\}\bI_n)^{-1}=\bZ_s$, with $\bZ_s$ given by \eqref{eq:Zdef}. In this way, the matrix $\bA^\sT \bZ (\bZ+\delta\mathbb E\{Z(G^2-1)\}\bI_n)^{-1} \bA$ coincides with the spectral matrix $\bD_n$, as defined in \eqref{eq:Dn_def}, and therefore $\bv^t$ approaches the spectral estimator $\bxs$. 

In conclusion, as $\bv^1 = \sqrt{d} \bxl$ and $\bv^t$ tends to be aligned with $\sqrt{d} \bxs$ for large $t$, we can use the state evolution result of Proposition \ref{prop:GAMP_SE} to analyze the joint empirical distribution of $(\bx, \sqrt{d} \bxl, \sqrt{d} \bxs)$.  At this point, the proof of Theorem \ref{th:W2conv2} becomes a straightforward application of Lemma \ref{lem:evec_emp_dist}, and is presented at the end of Section \ref{eq:GAMP_power}. The proof of Lemma \ref{lem:evec_emp_dist} is quite long, so it is provided separately  in Section \ref{subsec:proof_lem_evec}.

%which iteratively produces vectors $(\bu^t, \bv^t)$ for $t \geq 1$, where $\bu^t \in \reals^n, \bv^t \in \reals^d$. We  design a GAMP algorithm such that  $\bv^1 = \sqrt{d} \bxl$, and $\bv^t$ is aligned with $\sqrt{d} \bxs$ as $t$ grows large. More precisely, we show that, almost surely,
%\beq
%\label{eq:vtxs_diff}
%\lim_{t \to \infty } \lim_{d \to \infty}  \frac{\abs{\< \bv^t, \, \bxs \>}}{ \|\bv^t\| \|\bxs \|} =1.
%\eeq
% Then, we use the GAMP algorithm to analyze the empirical joint distribution of $(\bx, \sqrt{d} \bxl, \sqrt{d} \bxs)$ via the empirical joint distribution of $(\bx, \bv^1, \bv^t)$ as $t \to \infty$. The latter can be succinctly characterized in the high-dimensional limit (as $d \to \infty$, with $n/d \to \delta$) using the state evolution recursion of the GAMP algorithm.  

%Then, we describe the generic GAMP algorithm in Section \ref{subsec:GAMP}, and present the state evolution result characterizing the empirical joint distribution of the iterates.  Next, in Section \ref{eq:GAMP_power}, we analyze a specific  GAMP algorithm whose iterates approach the spectral estimator. The key technical result for proving the convergence in \eqref{eq:vtxs_diff} (and thereby obtaining the desired empirical joint distribution) is Lemma \ref{lem:evec_emp_dist}. 

\subsection{State Evolution for Generalized Approximate Message Passing (GAMP)} \label{subsec:GAMP}

For $t \ge 1$, let us define the following deterministic recursion for the parameters $(\mu_{U,t}, \sigma_{U,t}$, $\mu_{V,t}, \sigma_{V,t} )$ appearing in \eqref{eq:Ut_def}-\eqref{eq:Vt_def}: 
\begin{equation}
\label{eq:SE_def}
\begin{split}
    \mu_{U,t} & = \frac{1}{\sqrt{\delta}} \E\{ X f_t(V_t) \}, \\
     \sigma_{U,t}^2 & = \frac{1}{\delta} \E\Big\{ \big( f_t(V_t)  -  X \E\{ X f_t(V_t)\} \big)^2 \Big\} = 
     \frac{1}{\delta} \E\Big\{ \big( f_t(V_t)  -  \sqrt{\delta}\,  \mu_{U,t} X  \big)^2 \Big\}, \\
     \mu_{V,t+1}& =\sqrt{\delta} \E\{ G g_t(U_t; Y) \} \, - \, \E\{ g_t'(U_t;Y) \} \E\{ X f_t(V_t) \}, \\
     \sigma_{V,t+1}^2 & = \E\{ g_t(U_t; Y)^2 \}.
\end{split}
\end{equation}
Recalling from \eqref{eq:GAMP_init} that $\bu^0 = c \bfone_n$, the state evolution recursion is initialized with
\beq
\begin{split}
% & \mu_{U,0} = \lim_{n \to \infty} \frac{1}{n} \< \bu^0, \bg \> \stackrel{\text{a.s.}}{=} 0, \qquad 
% \sigma_{U,0}^2 = \frac{1}{n} \| \bu^0 \|^2 - \mu_{U,0}^2 \stackrel{\text{a.s.}}{=} c^2, \\
%
& \mu_{V,1} = \sqrt{\delta} \, \E\{ g_0(c ; Y) G \},
% \lim_{n \to \infty} \frac{1}{n}\< g_0(\bu^0; \by), \bg \>,
\qquad 
\sigma_{V,1}^2 =  \E\{ g_0(c; Y)^2 \}.
%\lim_{n \to \infty}  \frac{1}{n}\| g_0(\bu^0; \by)\|^2.
\label{eq:SEinit0}
\end{split}
\eeq

Furthermore, for $t \geq 0$, let the sequences of random variables $(W_{U,t})_{t \geq 0}$ and 
$(W_{V,t})_{t \geq 0}$ be each jointly Gaussian with zero mean and covariance defined as follows
\cite{rush2018finite, berthier2020state}.  
% We let $\E\{ W_{U,0}^2 \} =1$, and for $t \ge 1$:
% \beq
% \E\{W_{U,t} W_{U,0}\} = \frac{1}{ \sqrt{\delta} \sigma_{U,t}} \E\{ f_t(V_t) - X \sqrt{\delta} \mu_{U,t} \}.
% \eeq
 First, we have:
\beq
\E\{W_{V,1} W_{V,t} \}  = \frac{1}{\sigma_{V,1} \, \sigma_{V,t}}   
\E\left\{ g_{0}(c; Y) \, g_{t-1}(\mu_{U,t-1} G + \sigma_{U,t-1} W_{U,t-1}; Y) \right\}, \qquad t \geq 2. 
\label{eq:WV_corr_init}
\eeq
Then, for $r, t \geq 1$, we iteratively compute:
\begin{align}
& \E\{W_{U,r} W_{U,t} \} \nonumber \\
& \quad  = \frac{1}{\sigma_{U,r} \, \sigma_{U,t}}   
\cdot \frac{1}{\delta}\E\Big\{ \big( f_r(\mu_{V,r} X + \sigma_{V,r} W_{V,r})  -  X \sqrt{\delta}\,  \mu_{U,r}\big)
\big( f_t(\mu_{V,t} X + \sigma_{V,t} W_{V,t})  -  X \sqrt{\delta}\,  \mu_{U,t}\}\big) \Big\},
\label{eq:WU_corr}  \\
\label{eq:WV_corr}
& \E\{W_{V,r+1} W_{V,t+1} \}  = \frac{1}{\sigma_{V,r+1} \, \sigma_{V,t+1}}   
\E\left\{ g_{r}(\mu_{U,r} G + \sigma_{U,r} W_{U,r}; Y) \, g_{t}(\mu_{U,t} G + \sigma_{U,t} W_{U,t}; Y) \right\}.  
\end{align}
Note that for $r=t$, by \eqref{eq:SE_def} we have $\E\{W_{U,t}^2 \} = \E\{W_{V,t}^2 \}=1$. 

At this point, we are ready to present the state evolution result \cite{RanganGAMP,javanmard2013state} for the GAMP algorithm \eqref{eq:GAMP}-\eqref{eq:GAMP_onsager}.

\begin{proposition}[State Evolution]
Consider the GAMP iteration in Eqs. \eqref{eq:GAMP}-\eqref{eq:GAMP_onsager}, with initialization $\bu_0 = c \bfone_n \in \reals^n$, for any constant $c>0$. Assume that for $t \ge 0$, the functions $g_t: \reals \times \reals \to \reals$ and $f_t: \reals  \to \reals$ are Lipschitz, and that Assumption \textbf{(B1)} on p.\pageref{Assump:B1} holds for some $k \geq 2$.

Then, the following holds almost surely for any PL(k) function $\psi: 
\reals^{t+1}  \to \reals$, for $t \geq 1$:
%\mm{need to check that this works for  PL(k)}
\begin{align}
& \lim_{n \to \infty}\frac{1}{n} \sum_{i=1}^n \psi(g_i, u^t_i, u^{t-1}_i, \ldots, u^1_{i} ) = \E\{ \psi(G, \, U_t, \, U_{t-1}, \, \ldots, U_1) \}, \label{eq:psiG} \\
& \lim_{d \to \infty} \frac{1}{d} \sum_{i=1}^d \psi(x_i, v^{t}_i, v^{t-1}_i, \ldots, v^1_i) = \E\{ \psi(X, \, V_{t}, \, V_{t-1}, \, \ldots, V_1) \}, \label{eq:psiX}
\end{align}
where the distributions of the random vectors $(G, U_t, \ldots, U_1)$ and $(X, V_t, \ldots, V_1)$ are given by the state evolution recursion in Eqs. \eqref{eq:Ut_def}-\eqref{eq:WV_corr}.
\label{prop:GAMP_SE}
\end{proposition}

\begin{remark}
%
% The analysis in \cite{rush2018finite} shows that the asymptotic results in \eqref{eq:psiG} and \eqref{eq:psiX} can be sharpened to obtain non-asymptotic concentration inequalities. Indeed, for any $\epsilon \in (0,1)$, \cite[Lemma 6]{rush2018finite} implies that, for $t > 0$,
% \begin{align}
%     \prob \left( \abs{ \frac{1}{d} \sum_{i=1}^d \psi(x_i, v^{t}_i,  \ldots, v^1_i) -  \E\{ \psi(X,  \, V_{t}, \, \ldots, V_1) \}} \geq \epsilon \right) \leq 
%     K_t \exp\big(-\kappa_t n \epsilon^2 \big),  \label{eq:xv_conc} \\
%     %
%      \prob \left( \abs{ \frac{1}{n} \sum_{i=1}^n \psi(g_i, u^{t}_i,  \ldots, u^1_i) -  \E\{ \psi(G,  \, U_{t}, \, \ldots, U_1) \}} \geq \epsilon \right) \leq 
%       K_t \exp\big(-\kappa_t n \epsilon^2 \big).
%      \label{eq:gu_conc}
% \end{align}
% Here, $K_t = (C_0)^t (t!)^C$ and $\kappa_t = 1/((c_0)^t (t!)^c)$, where $C, C_0, c, c_0$ are  universal positive constants that do not depend on $t$ or $n$. The above concentration inequalities and the Borel-Cantelli lemma imply that the following limits hold almost surely:
Suppose that we have a sequence of $PL(k)$ functions $\psi_t: 
\reals^{t+1}  \to \reals$ (indexed by $t$) such that 
\beq
\begin{split} 
& \lim_{t \to \infty} \E\{ \psi_t(G, \, U_{t}, \, U_{t-1}, \, \ldots, U_1)\} = c_U, \\
& \lim_{t \to \infty} \E\{ \psi_t(X, \, V_{t}, \, V_{t-1}, \, \ldots, V_1)\} = c_V,
\end{split}
\eeq
for some constants $c_U, c_V \in \reals$. Then, since the statements \eqref{eq:psiG} and \eqref{eq:psiX} hold with probability $1$ for each fixed $t \geq 1$, we have that, almost surely,
\begin{align}
   & \lim_{t\to \infty} \lim_{n \to \infty} \frac{1}{n} \sum_{i=1}^n \psi(g_i, u^{t}_i,  \ldots, u^1_i) = c_U,  \label{eq:tn_GU_lim} \\
    & \lim_{t\to \infty} \lim_{d \to \infty}   \frac{1}{d} \sum_{i=1}^d \psi(x_i, v^{t}_i,  \ldots, v^1_i)  =  c_V. \label{eq:td_XV_lim}
\end{align}
%The statements in \eqref{eq:td_XV_lim} and \eqref{eq:tn_GU_lim}  can be obtained by taking, say, $n =(1/\kappa_t)^{2}$, which makes the Borel-Cantelli sum (over $t \geq 0$) finite. 
\label{rem:tn_infty}
\end{remark}

\subsection{GAMP as a Power Method to Compute the Spectral Estimator} \label{eq:GAMP_power}

Consider the GAMP iteration in \eqref{eq:GAMP} initialized with $\bu^0 = \frac{1}{\delta} \bfone_n$, and the function $g_0:\reals \times \reals \to \reals$ chosen as
\beq
g_0(u; \, y) =    \frac{\cT_L(y)}{\sqrt{\delta}},
\label{eq:g0def}
\eeq
so that 
\beq
\bv^1 = \frac{1}{\delta}\bA^{\sT} \cT_L(\by).
\label{eq:v1_sp}
\eeq
(The function $f_0$ is assumed to be zero.)  From \eqref{eq:xlin_def}, we note that $\bv^1 = \sqrt{d} \, \bxl$.

For $t \ge 1$, let the functions $g_t: \reals \times \reals \to \reals$ and $f_t:\reals \to \reals$ be chosen as
\begin{equation}
   g_t(u; \, y) = \sqrt{\delta} \, u \cT(y), \qquad  f_{t}(v) = \frac{v}{\beta_{t}},
   \label{eq:ft_gt_choice}
\end{equation}
where the function $\cT: \reals \to \reals$ is bounded and Lipschitz, and $\beta_t$ is a constant, defined iteratively for $t \ge 1$ via the state evolution equations below (Eqs. \eqref{eq:SE_lin1}-\eqref{eq:defbetat1}). To prove Theorem \ref{th:W2conv2}, we will choose $\cT$ as a suitable function of $\cT_s$ (see \eqref{eq:deftildeZ}). Note that the functions $g_t$ and $f_t$ are required to be Lipschitz for $t\ge 0$, and this  will be ensured choosing $\cT$ to be bounded and Lipschitz  (see Lemma \ref{lem:evec_emp_dist} below).

With this choice of $f_t, g_t$, for $t\ge 1$, the scalars in \eqref{eq:GAMP_onsager} are given by
\beq
\sb_t =\frac{1}{\delta \beta_t}, \qquad
\sc_t = \sqrt{\delta} \cdot \frac{1}{n}\sum_{i=1}^n \cT(y_i).
\eeq
In the GAMP iteration below, we will replace the parameter $\sc_t$  by its almost sure limit $\bar{\sc}_t = \sqrt{\delta} \, \E\{Z\}$, where $Z \triangleq \cT(Y)$. The state evolution result in Proposition \ref{prop:GAMP_SE} still holds when $\sc_t$ is replaced with $\bar{\sc}_t$ in the GAMP iteration \cite{RanganGAMP, javanmard2013state}. This can be shown using the pseudo-Lipschitz property of the test functions $\psi$ in Eqs. \eqref{eq:psiG}-\eqref{eq:psiX} and the fact that  $ \lim_{n \to \infty} \frac{1}{n}\sum_{i=1}^n \cT(y_i) = \E\{ Z \}$ almost surely, due to the strong law of large numbers.

%$$ \sc_t \lim_{d \to \infty} \, \frac{\| f_t(\bv^t) \|^2}{d} = \bar{\sc}_t \lim_{d \to \infty} \frac{\| f_t(\bv^t) \|^2}{d} = \bar{\sc}_t \E\{ f_t(V_t)^2 \}.$$

With these choices, the GAMP iteration in \eqref{eq:GAMP} is as follows.  Initializing with 
\beq
\bu^0 = \frac{1}{\delta} \bfone_n, \qquad \bv^1 = \frac{1}{\delta} \bA^{\sT} \cT_L(\by),
\label{eq:power_gamp_init}
\eeq
we have for $t \ge 1$: %
\begin{equation}
\label{eq:GAMP_choice_fg}
\begin{split}
    \bu^t &= 
    \begin{cases}
    \frac{1}{\sqrt{\delta} \, \beta_t} \big[ \bA \bv^t \, -  \, \bZ_L \bu^{t-1} \big], & \text{ for } t=1, \\
    \frac{1}{\sqrt{\delta} \, \beta_t} \big[ \bA \bv^t \, -  \, \bZ \bu^{t-1} \big], & \text{ for } t >1,  
    \end{cases}\\
\bv^{t+1} &=  \bA^\sT \bZ \bu^t - \frac{\sqrt{\delta}}{\beta_t} \, \E\{ Z \} \,  \bv^t,
\end{split}    
\end{equation}
where $\bZ_L = \diag(\cT_L(y_1), \ldots, \cT_L(y_n))$ and   $\bZ = \diag(\cT(y_1), \ldots, \cT(y_n))$.
The state evolution equations in \eqref{eq:SE_def} become:
\begin{align}
    & \mu_{U,t} = \frac{\mu_{V,t}}{\sqrt{\delta} \beta_t}, \qquad \sigma_{U, t}^2 = \frac{\sigma_{V,t}^2}{ \delta \, \beta^2_t},  \label{eq:SE_lin1} \\
    & \mu_{V,t+1} = \sqrt{\delta}  \,\frac{\mu_{V,t}}{\beta_t} \big[ \E\{ZG^2\} - \E\{Z\} \big], \qquad
    \sigma_{V,t+1}^2 = \frac{1}{\beta_t^2}\big[ \mu_{V,t}^2 \E\{Z^2G^2\} + 
    \sigma_{V,t}^2 \E\{ Z^2\} \big],
    \label{eq:SE_lin2}\\
    & \beta_{t+1} = \sqrt{\mu_{V,{t+1}}^2 + \sigma_{V,{t+1}}^2}. \label{eq:defbetat1}
\end{align}
Here, we recall that 
$G \sim \normal(0,1)$  and $Z=\cT(Y)$, for $Y \sim p_{Y|G}(\cdot \mid G)$. From \eqref{eq:SEinit0}, the state evolution iteration is initialized with the following:
\beq
\mu_{V,1} =   \E\{ \cT_L(Y) G \}, \qquad  \sigma_{V,1}^2 =  \frac{1}{\delta}\E\{ \cT_L(Y)^2 \}, \qquad 
\beta_1 = \sqrt{\mu_{V,1}^2 + \sigma_{V,1}^2}.
\label{eq:SEinit}
\eeq

We will show in Lemma \ref{lem:evec_emp_dist} that in the high-dimensional limit, the vector $\bv^t$  in \eqref{eq:GAMP_choice_fg} tends to align with the principal eigenvector of the matrix $\bM_n=  \bA^\sT \bZ(\bZ+\delta\mathbb E\{Z(G^2-1)\}\bI_n)^{-1} \bA$ as $t \to \infty$. In other words, the GAMP iteration is equivalent to a power iteration for $\bM_n$. This equivalence, together with Proposition \ref{prop:GAMP_SE}, allows us to precisely characterize the limiting empirical distribution of the eigenvector of $\bM_n$ in  Lemma \ref{lem:evec_emp_dist}.

We begin with a result characterizing the fixed points of the state evolution recursion in \eqref{eq:SE_lin1}-\eqref{eq:SE_lin2}.

\begin{lemma}
Assume that $\E\{ Z(G^2-1)\} >0$ and that $\delta > \frac{\E\{ Z^2\}}{ (\E\{ Z G^2 \} \, - \,\E\{Z\})^2}$. Then, the state evolution recursion \eqref{eq:SE_lin1}-\eqref{eq:SE_lin2} has three fixed points:  one is $\FP_0 \equiv (\bar{\mu}_V=0, \bar{\sigma}_V^2 =\E\{ Z^2\})$, and the other  two are $\FP_1 \equiv (\tilde{\mu}_V, \tilde{\sigma}_V^2)$ and $\FP_2 \equiv (-\tilde{\mu}_V, \tilde{\sigma}_V^2)$, where 
\beq
\tilde{\mu}_V = \sqrt{\frac{\tilde{\beta}^2(\tilde{\beta}^2 -  \E\{Z^2\})}{\tilde{\beta}^2 + \E\{Z^2 G^2\} - \E\{Z^2\}}}, \qquad  \tilde{\sigma}_V^2 = \frac{\tilde{\beta}^2 \E\{Z^2 G^2\}}{ \tilde{\beta}^2 + \E\{ Z^2 G^2\} - \E\{Z^2\}}, 
\label{eq:FP1}
\eeq
with 
\beq
\tilde{\beta}^2 =  \delta\, ( \E\{ZG^2\} - \E\{Z\})^2.
\label{eq:beta2_def}
\eeq

Furthermore, if the initialization \eqref{eq:SEinit} is such that $\mu_{V,1} >0$, then the recursion  converges to $\FP_1$. If $\mu_{V,1} < 0$, the recursion converges to $\FP_2$. 
\label{lem:fixed_pts}
\end{lemma}
The lemma is proved in Appendix \ref{app:lem_fixed_pts_proof}. We note that Lemma \ref{lem:fixed_pts} (and the subsequent Lemmas \ref{lem:diffs}-\ref{lem:evec_emp_dist} as well) assumes that $\E\{ Z(G^2-1)\} >0$ and $\delta > \frac{\E\{ Z^2\}}{ (\E\{ Z G^2 \} \, - \,\E\{Z\})^2}$. These conditions concern the auxiliary random variable $Z$ (or, equivalently, the auxiliary function $\cT$). In the proof of Theorem \ref{th:W2conv2}, which appears at the end of this section, we will provide a choice of $Z$ (depending on $Z_s$) that fulfills these requirements; see \eqref{eq:deftildeZ}. Let us also point out that the condition $\delta > \frac{\E\{ Z^2\}}{ (\E\{ Z G^2 \} \, - \,\E\{Z\})^2}$ follows from $\psi_{\delta}'(\lambda_{\delta}^*)> 0$, which in turn ensures that $\rho_s>0$. For a discussion of the case $\rho_s=0$, see Remark \ref{rmk:special}.

The next lemma shows that the mean-squared difference between successive AMP iterates vanishes as $t \to \infty$ in the high-dimensional limit.

\begin{lemma}
Assume that $\E\{ Z(G^2-1)\} >0$, $\delta > \frac{\E\{ Z^2\}}{ (\E\{ Z G^2 \} \, - \,\E\{Z\})^2}$, and $\abs{\E\{ \cT_L(Y) G \}} >0$.  Consider the GAMP iteration in \eqref{eq:GAMP_choice_fg} initialized with $\bu^0 = \frac{1}{\delta}\bfone_n$. Then, the following limits hold almost surely:
\begin{align}
\lim_{t \to \infty} \lim_{n \to \infty} \, \frac{1}{n} \| \bu^t - \bu^{t-1} \|^2 =0, \qquad \lim_{t\to \infty} \lim_{d \to \infty} \, \frac{1}{d} \| \bv^{t+1} - \bv^{t} \|^2 =0.
\end{align}
\label{lem:diffs}
\end{lemma}
%%%

%%%%
 The proof of the lemma is given in Appendix \ref{app:proof_lem_diffs}.
 The next result is the main technical lemma needed to prove Theorem \ref{th:W2conv2}.
 It shows that, as $t$ grows large, $\bv^t$ tends to be aligned with the principal eigenvector of the matrix $\bM_n$ defined in \eqref{eq:matZ1} below. Theorem \ref{th:W2conv2} is then obtained from Lemma \ref{lem:evec_emp_dist} by using a suitable choice for $\cT(\cdot)$ in the GAMP iteration \eqref{eq:GAMP_choice_fg}, which ensures that $\bM_n$ is a scaled version of $\bD_n$ in \eqref{eq:Dn_def}.

\begin{lemma}
Let $\bx$ be such that $\|\bx\|^2=d$, $\{\ba_i\}_{1\le i\le n}\sim_{i.i.d.}\normal({\b0}_d,\id_d/d)$, and $\by$ be distributed according to \eqref{eq:defy}. Let $n/d\to \delta$, $G\sim \normal(0, 1)$ and $Z=\mathcal T(Y)$
for $Y\sim p_{Y|G}(\,\cdot\,|\,G)$. Assume that the conditions \textbf{(B1)}-\textbf{(B2)} on p.\pageref{Assump:B1} hold (with $\cT_s(\cdot)$ replaced by $\mathcal T(\cdot)$ in \textbf{(B2)}). Assume also that $Z$ is bounded, $\mathbb P(Z>-1)=1$, $\E\{ Z(G^2-1)\} >0$ and $\delta > \frac{\E\{ Z^2\}}{ (\E\{ Z G^2 \} \, - \,\E\{Z\})^2}$.  Define $Z'=\frac{Z}{Z+\delta\mathbb E\{Z(G^2-1)\}}$ and assume that $Z'$ satisfies the assumptions \textbf{(A1)}, \textbf{(A2)},  \textbf{(A3)} on p. \pageref{assump:A1}.   Define the matrix 
\begin{equation}\label{eq:matZ1}
  \bM_n=  \bA^\sT \bZ(\bZ+\delta\mathbb E\{Z(G^2-1)\}\bI_n)^{-1} \bA,
\end{equation} 
where $\bZ = \diag(\cT(y_1), \ldots, \cT(y_n))$. Let $\hat{\bphi}_1$ be the principal eigenvector of $\bM_n$, let its sign be chosen so that $\<\hat{\bphi}_1, \bx\>\ge0$, and consider the rescaling $\tilde{\bphi}^{(1)}=\sqrt{d}\hat{\bphi}_1$. Also, let $\tbx^{\rm L} = \sqrt{d} \bxl$, where $\bxl$ is the linear estimator defined in \eqref{eq:xlin_def}. 

Then, the following holds almost surely  for any PL(k) function $\psi: \reals \times \reals \times \reals \to \reals$:
\begin{align}
& \lim_{d \to \infty} \,   \frac{1}{d} \sum_{i=1}^d \psi(x_i, \tilde{x}^{\rm L}_i,  \tilde{\varphi}_i^{(1)}) = \E\{ \psi(X, 
\, \mu_{V,1} X + \sigma_{V,1} W_{V,1}
,\,  \tilde{\beta}^{-1}(\tilde{\mu}_{V} X + \tilde{\sigma}_{V} W_{V, \infty})) \} \label{eq:psiX1}. 
\end{align}
Here $X \sim P_X$,  $\tilde{\mu}_{V}, \tilde{\sigma}_{V}, \tilde{\beta}$ are given by \eqref{eq:FP1}-\eqref{eq:beta2_def}, and $\mu_{V,1},\sigma_{V,1}^2$ are given by \eqref{eq:SEinit}. The random variables $(W_{V,1}, W_{V, \infty})$ are independent of $X$, and jointly Gaussian with zero mean and covariance:
\beq
\E\{ W_{V,1}^2 \} = \E\{ W_{V,\infty}^2 \} =1, \qquad 
\E\{W_{V,1} W_{V,\infty}  \} =  \frac{ \tilde{\mu}_{V} \, \E\{ \cT_L(Y) Z G \}}
{ \tilde{\beta} \tilde{\sigma}_{V} \sqrt{ \E\{ \cT_L(Y)^2\} } }.
\label{eq:WV1inf_cov}
\eeq
\label{lem:evec_emp_dist}
\end{lemma}

We first show how Theorem \ref{th:W2conv2} is obtained from Lemma \ref{lem:evec_emp_dist}, and then prove the lemma in the following subsection.

\begin{proof}[\textbf{Proof of Theorem \ref{th:W2conv2}}]
Recall that $\lambda_\delta^*$ is the unique solution of \eqref{eq:uniqsol} for $\lambda>\tau$, where $\tau$ is the supremum of the support of $Z_s$. Define
\begin{equation}\label{eq:deftildeZ}
Z \triangleq \frac{Z_s}{\lambda_\delta^*- Z_s}  = \frac{\cT_s(Y)}{\lambda_\delta^*- \cT_s(Y)}.
\end{equation}
We now verify that $Z$ satisfies the assumptions of Lemma \ref{lem:evec_emp_dist}. As $\lambda_\delta^*>\tau$ and $Z_s$ has bounded support with supremum $\tau$, we have that $Z$ is bounded. Furthermore, $Z$ is a Lipschitz function of $Y$, since $Z_s$ is a Lipschitz function of $Y$ and $Z_s$ is bounded away from $\lambda_\delta^*$. Thus, $Z$ satisfies the condition \textbf{(B2)} (with $\cT_s(\cdot)$ replaced by $\cT(\cdot)$).

Next, note that $\tau >0$ (since $Z_s$ satisfies assumption \textbf{(A3)}). As $\mathbb P(\lambda_\delta^*- Z_s>0)=1$, we have that  $\mathbb P(Z>-1)=1$.

As $\psi_{\delta}'(\lambda_{\delta}^*)> 0$, we have that $\lambda_{\delta}^*>\bar{\lambda}_\delta$, where $\bar{\lambda}_\delta$ denotes the point at which $\psi_\delta$ attains its minimum (see Eq. \eqref{eq:minpsi}). Consequently, since $\lambda_{\delta}^*$ solves \eqref{eq:uniqsol}, we have that
\begin{equation}\label{eq:eqlds}
 \psi_\delta(\lambda_{\delta}^*)=\phi(\lambda_{\delta}^*).   
\end{equation}
As $Z_s$ satisfies assumption \textbf{(A3)}, we have that $\tau>0$, which implies that $\lambda_\delta^*>0$. Thus, by using the definitions \eqref{eq:defphi2}-\eqref{eq:defpsi}, \eqref{eq:eqlds} can be rewritten as
\begin{equation}\label{eq:alphastarnewc}
    \mathbb E\left\{\frac{Z_s(G^2-1)}{\lambda_{\delta}^*-Z_s}\right\}=\frac{1}{\delta}.
\end{equation}
By combining \eqref{eq:deftildeZ} and \eqref{eq:alphastarnewc}, we obtain that
\begin{equation}\label{eq:condZt}
    \mathbb E\{Z(G^2-1)\}=\frac{1}{\delta}>0.
\end{equation}

Finally, we compute the derivative of $\psi_\delta(\lambda)$ at $\lambda_{\delta}^*$, noting that the derivative and the expectation in \eqref{eq:defpsi} can be interchanged due to bounded convergence. We thus obtain that the condition $\psi_{\delta}'(\lambda_{\delta}^*)> 0$ is equivalent to
\begin{equation}
    \label{eq:tilZ_sq}
    \frac{1}{\delta} > \mathbb E\left\{\frac{Z_s^2}{(\lambda_{\delta}^*-Z_s)^2}\right\} = \mathbb E\left\{Z^2\right\},
\end{equation}
where the last equality follows from \eqref{eq:deftildeZ}. From \eqref{eq:condZt} and \eqref{eq:tilZ_sq}, we have $\delta > \frac{\E\{ Z^2\}}{ (\E\{ Z G^2 \} \, - \,\E\{Z\})^2}$.

We also have
\begin{equation}
\label{eq:ZprZs}
  Z'\triangleq  \frac{Z}{Z+\delta\mathbb E\{Z(G^2-1)\}}=\frac{Z}{Z+1}=\frac{Z_s}{\lambda_\delta^*},
\end{equation}
where in the first equality we use \eqref{eq:condZt} and in the second equality we use \eqref{eq:deftildeZ}. Thus, $Z'$ satisfies the assumptions \textbf{(A1)}-\textbf{(A2)}-\textbf{(A3)} (with $\tau/\lambda_\delta^*$ being the supremum of its support), and we can apply Lemma \ref{lem:evec_emp_dist}.

Using \eqref{eq:ZprZs} in \eqref{eq:matZ1}, we see that $\bM_n = \frac{1}{\lambda_\delta^*}\bA^{\sT} \bZ_s \bA$. Therefore the principal eigenvector of $\bM_n$ is equal to $\bxs$. Furthermore, from \eqref{eq:FP1}-\eqref{eq:beta2_def}, we can compute the coefficients $\tilde{\mu}_{V}, \tilde{\sigma}_{V}, \tilde{\beta}$ as
\begin{equation*}
    \begin{split}
    \tilde{\beta}^2 &=  \delta\, ( \E\{ZG^2\} - \E\{Z\})^2=\frac{1}{\delta},\\
    \tilde{\mu}_V &= \sqrt{\frac{\tilde{\beta}^2(\tilde{\beta}^2 -  \E\{Z^2\})}{\tilde{\beta}^2 + \E\{Z^2 G^2\} - \E\{Z^2\}}}=\sqrt{\frac{\frac{1}{\delta}\left(\frac{1}{\delta} -  \E\left\{\frac{Z_s^2}{(\lambda_\delta^*- Z_s)^2}\right\}\right)}{\frac{1}{\delta} + \E\left\{\frac{Z_s^2(G^2-1)}{(\lambda_\delta^*- Z_s)^2}\right\}}},\\
    \tilde{\sigma}_V &= \sqrt{\frac{\tilde{\beta}^2 \E\{Z^2 G^2\}}{ \tilde{\beta}^2 + \E\{ Z^2 G^2\} - \E\{Z^2\}}}= \sqrt{\frac{\frac{1}{\delta} \E\left\{\frac{Z_s^2\cdot G^2}{(\lambda_\delta^*- Z_s)^2}\right\}}{\frac{1}{\delta} + \E\left\{\frac{Z_s^2(G^2-1)}{(\lambda_\delta^*- Z_s)^2}\right\}}}.
    \end{split}
\end{equation*}
By using also \eqref{eq:SEinit}, one can easily verify that 
\begin{equation*}
 \mu_{V, 1}=n_L\rho_L,\qquad \sigma_{V, 1}=n_L\sqrt{1-\rho_L^2},\qquad \tilde{\beta}^{-1}\tilde{\mu}_V=\rho_s,  
\end{equation*}
\begin{equation*}
    \tilde{\beta}^{-1}\tilde{\sigma}_V=\sqrt{1-\rho_s^2}, \qquad \frac{q-\rho_s\rho_L}{\sqrt{(1-\rho_s^2)(1-\rho_L^2)}}=\frac{ \tilde{\mu}_{V} \, \E\{ \cT_L(Y) Z G \}}
{ \tilde{\beta} \tilde{\sigma}_{V} \sqrt{ \E\{ \cT_L(Y)^2\} } },
\end{equation*}
which yields the desired result.
\end{proof}

%%%%%%
%%%%%%
\subsection{Proof of Lemma \ref{lem:evec_emp_dist}} \label{subsec:proof_lem_evec}

Fix $c>0$, and let $\tilde{\bZ}=c \bZ$ and $\tilde{Z}=c Z$. Then,
\begin{equation}
\begin{split}
   \bZ(\bZ+\delta\mathbb E\{Z(G^2-1)\}\bI_n)^{-1} &= \tilde{\bZ}(\tilde{\bZ}+\delta\mathbb E\{\tilde{Z}(G^2-1)\}\bI_n)^{-1}.
\end{split}
\end{equation}
Thus, by inspecting the definition \eqref{eq:matZ1}, one immediately obtains that $\tilde{\bphi}^{(1)}$ does not change if we rescale $\bZ$ and $Z$ by the multiplicative factor $c$. Furthermore, by using the definitions \eqref{eq:FP1}-\eqref{eq:beta2_def}, we have that
\begin{equation}
\begin{split}
        \tilde{\beta}^{-1}\tilde{\mu}_{V}&=\sqrt{\frac{\delta\, ( \E\{ZG^2\} - \E\{Z\})^2 -  \E\{Z^2\}}{\delta\, ( \E\{ZG^2\} - \E\{Z\})^2 + \E\{Z^2 G^2\} - \E\{Z^2\}}} \\
        &= \sqrt{\frac{\delta\, ( \E\{\tilde{Z}G^2\} - \E\{\tilde{Z}\})^2 -  \E\{\tilde{Z}^2\}}{\delta\, ( \E\{\tilde{Z}G^2\} - \E\{\tilde{Z}\})^2 + \E\{\tilde{Z}^2 G^2\} - \E\{\tilde{Z}^2\}}},\\
  \end{split}
\end{equation}
\begin{equation}
\begin{split}
        \tilde{\beta}^{-1}\tilde{\sigma}_{V}&=\sqrt{\frac{\E\{Z^2 G^2\}}{ \delta\, ( \E\{ZG^2\} - \E\{Z\})^2 + \E\{ Z^2 G^2\} - \E\{Z^2\}}}\\
        &=\sqrt{\frac{\E\{\tilde{Z}^2 G^2\}}{ \delta\, ( \E\{\tilde{Z}G^2\} - \E\{\tilde{Z}\})^2 + \E\{ \tilde{Z}^2 G^2\} - \E\{\tilde{Z}^2\}}},
\end{split}
\end{equation}
and that
\begin{equation}
\begin{split}
       \frac{ \tilde{\mu}_{V} \, \E\{ \cT_L(Y) Z G \}}
{ \tilde{\beta} \tilde{\sigma_{V}} \sqrt{ \E\{ \cT_L(Y)^2\} } } &=\frac{\E\{ \cT_L(Y) Z G \}}{\sqrt{\delta}( \E\{ZG^2\} - \E\{Z\})}\sqrt{\frac{\delta\, ( \E\{ZG^2\} - \E\{Z\})^2 - \E\{Z^2\}}{\E\{ Z^2 G^2 \}\E\{ \cT_L(Y)^2\}}}\\
        &=\frac{\E\{ \cT_L(Y) \tilde{Z} G \}}{\sqrt{\delta}( \E\{\tilde{Z}G^2\} - \E\{\tilde{Z}\})}\sqrt{\frac{\delta\, ( \E\{\tilde{Z}G^2\} - \E\{\tilde{Z}\})^2 - \E\{\tilde{Z}^2\}}{\E\{ \tilde{Z}^2 G^2 \}\E\{ \cT_L(Y)^2\}}}.
\end{split}
\end{equation}
Consequently, both the LHS and the RHS of \eqref{eq:psiX1} are unchanged when we rescale $\bZ$ and $Z$ by the multiplicative factor $c$.

The argument above proves that, without loss of generality, we can rescale $\bZ$ and $Z$ by any multiplicative factor $c>0$. To simplify the rest of the argument, it is convenient to assume the normalization condition 
\begin{equation}\label{eq:ncond}
 \mathbb E\{Z(G^2-1)\}=\frac{1}{\delta},   
\end{equation}
under which $\bM_n = \bA^\sT \bZ(\bZ+\bI_n)^{-1} \bA$.

Consider the iteration \eqref{eq:GAMP_choice_fg} with the initialization in \eqref{eq:power_gamp_init}.   Since by hypothesis, $\mu_{V,1}^2 = (\E\{ \cT_L(Y) G  \})^2 >0$, Lemma \ref{lem:fixed_pts} guarantees that the state evolution recursion \eqref{eq:SE_lin2} converges to either fixed point $\FP_1$ or $\FP_2$ as $t \to \infty$. That is, 
\beq
\lim_{t \to \infty} \mu_{V,t}^2 = \tilde{\mu}_V^2, \quad 
\lim_{t \to \infty} \sigma^2_{V,t} = \tilde{\sigma}^2_V, \quad
\lim_{t \to \infty} \beta_t^2 = \tilde{\mu}_V^2 + \tilde{\sigma}^2_V = \tilde{\beta}^2 = \frac{1}{\delta}.
\label{eq:FPspecial}
\eeq
The last equality above follows by combining \eqref{eq:beta2_def} and \eqref{eq:ncond}.

By substituting the expression for $\bu^t$ in \eqref{eq:GAMP_choice_fg} in the $\bv^{t+1}$ update,  the iteration can be rewritten as follows, for $t \geq 2$:
\begin{align}
 \bu^t &= \frac{1}{\sqrt{\delta} \, \beta_t} \big[ \bA \bv^t \, -  \, \bZ \bu^{t-1} \big], \label{eq:GAMPrewrite2} \\
 \bv^{t+1} & = \frac{1}{\sqrt{\delta} \beta_t}\Big[ \big(\bA^\sT \bZ \bA  - \delta \E\{Z\} \, \bI_d \big) \bv^t \, - \, \bA^\sT \bZ^2 \bu^{t-1} \Big]. \label{eq:GAMPrewrite1}
\end{align}
In the remainder of the proof, we will assume that $t \geq 2$. Define
\begin{align}
        \be_1^t &= \bu^{t}-\bu^{t-1},\label{eq:err1}\\
        \be_2^t &= \bv^{t+1}-\bv^{t}.\label{eq:err2}
\end{align}
By combining \eqref{eq:err1} with \eqref{eq:GAMPrewrite2}, we have that
\begin{equation}\label{eq:ut1new}
    \bu^{t-1}=(\bZ+\sqrt{\delta}\beta_t\bI_n)^{-1}\bA\bv^t-\sqrt{\delta}\beta_t(\bZ+\sqrt{\delta}\beta_t\bI_n)^{-1}\be_1^t.
\end{equation}
By substituting the expression for $\bu^{t-1}$ obtained in \eqref{eq:ut1new} into \eqref{eq:GAMPrewrite1}, we have
\begin{equation}\label{eq:GAMPrewrite3}
\begin{split}
    \bv^{t+1}  &= \left(\bA^\sT \bZ(\bZ+\sqrt{\delta}\beta_t\bI_n)^{-1} \bA  - \frac{\sqrt{\delta} \E\{Z\}}{\beta_t} \, \bI_d \right) \bv^t +\bA^\sT \bZ^2(\bZ+\sqrt{\delta}\beta_t\bI_n)^{-1}\be_1^t\\
    &= \left(\bA^\sT \bZ(\bZ+\bI_n)^{-1} \bA  -\delta \E\{Z\} \, \bI_d \right) \bv^t  \\
    & \quad +(1-\sqrt{\delta}\beta_t)\bA^\sT \bZ(\bZ+\bI_n)^{-1}(\bZ+\sqrt{\delta}\beta_t\bI_n)^{-1} \bA\bv^t \\
& \quad   + \delta\E\{Z\}\left(1-\frac{1}{\sqrt{\delta}\beta_t}\right) \bv^t+\bA^\sT \bZ^2(\bZ+\sqrt{\delta}\beta_t\bI_n)^{-1}\be_1^t.
\end{split}
\end{equation}
Let
\begin{equation}\label{eq:GAMPrewrite4}
    \be_3^t =  \left(\bA^\sT \bZ(\bZ+\bI_n)^{-1} \bA  -(\delta \E\{Z\}+1) \, \bI_d \right) \bv^t.
\end{equation}
From \eqref{eq:err2} and \eqref{eq:GAMPrewrite3}, we obtain
\begin{equation}\label{eq:deferr3}
    \begin{split}
        \be_3^t&=\be_2^t-(1-\sqrt{\delta}\beta_t)\bA^\sT \bZ(\bZ+\bI_n)^{-1}(\bZ+\sqrt{\delta}\beta_t\bI_n)^{-1} \bA\bv^t \\
& \quad   - \delta\E\{Z\}\left(1-\frac{1}{\sqrt{\delta}\beta_t}\right) \bv^t-\bA^\sT \bZ^2(\bZ+\sqrt{\delta}\beta_t\bI_n)^{-1}\be_1^t.
    \end{split}
\end{equation}

Let us decompose $\bv^t$ into a component in the direction of $\hat{\bphi}_1$ plus an orthogonal component $\br^t$:
\begin{equation}\label{eq:projvt}
    \bv^t = \xi_t \hat{\bphi}_1+\br^t,
\end{equation}
where $\xi_t=\<\bv^t, \hat{\bphi}_1\>$. 

At this point, the idea is to show that, when $t$ is large, $\br^t$ is small, thus $\bv^t$ tends to be aligned with $\hat{\bphi}_1$. To do so, we prove that, as $n\to\infty$, the largest eigenvalue of the matrix $\bM_n$ defined in \eqref{eq:matZ1} converges almost surely to $\delta \E\{Z\}+1$. Furthermore, we prove that the matrix $\bM_n$ exhibits a spectral gap, in the sense that the second largest eigenvalue of $\bM_n$ converges almost surely to a value  strictly smaller than $\delta \E\{Z\}+1$. Since $\br^t$ is orthogonal to $\hat{\bphi}_1$ and $\bM_n$ has a spectral gap, the norm  of $\be_3^t$ in \eqref{eq:GAMPrewrite4} can be lower bounded by a strictly positive constant times the norm of $\br^t$. Next, using the expression in \eqref{eq:deferr3}, we show that the norm of $\be_3^t$   can be made arbitrarily small by taking $n$ and $t$ sufficiently large. From this, we conclude that $\br^t$ must be small.

We begin by proving that $\bM_n$ has a spectral gap.
\begin{lemma}[Spectral gap for $\bM_n$]
The following holds almost surely:
\begin{align} 
\lim_{n \to \infty} \lambda_1^{\bM_n} & = 1+\delta \mathbb E\{Z\}, \label{eq:lambda1Mlim}
\\
\limsup_{n \to \infty} \lambda_2^{\bM_n} & < 1+\delta \mathbb E\{Z\} - c_1,
\label{eq:lambda2M}
\end{align}
for a numerical constant $c_1 >0$ that does not depend on $n$.
\label{lem:Mn_specgap} 
\end{lemma}
\begin{proof}[Proof of Lemma \ref{lem:Mn_specgap}]
By hypothesis, $Z'=Z/(Z+1)$ satisfies the assumptions \textbf{(A1)}-\textbf{(A2)}-\textbf{(A3)}. Thus, we can use Lemma \ref{lemma:pt} to compute the almost sure limit of the two largest eigenvalues of $\bM_n$, call them $\lambda_1^{\bM_n}\ge \lambda_2^{\bM_n}$.

Let $\tau$ denote the supremum of the support of $Z'$. As $\mathbb P(Z>-1)=1$ and $Z$ has bounded support, we have that $\tau<1$. For $\lambda\in (\tau, \infty)$, define
\begin{equation}\label{eq:defphi2new}
\phi(\lambda) = \lambda \cdot {\mathbb E}\left\{\frac{Z'\cdot G^2}{\lambda-Z'}\right\}=\lambda \cdot {\mathbb E}\left\{\frac{Z\cdot G^2}{(\lambda-1)Z+ \lambda}\right\},
\end{equation}
and 
\begin{equation}\label{eq:defpsinew}
\psi_{\delta}(\lambda) = \lambda\left(\frac{1}{\delta}+{\mathbb E}\left\{\frac{Z'}{\lambda-Z'}\right\}\right)= \lambda\left(\frac{1}{\delta}+{\mathbb E}\left\{\frac{Z}{(\lambda-1)Z+ \lambda}\right\}\right).
\end{equation}
Note that
\begin{equation}\label{eq:cp00}
    \begin{split}
    \phi(1) = \mathbb E\{Z\cdot G^2\},  \qquad \psi_{\delta}(1) =\frac{1}{\delta}+{\mathbb E}\{Z\}.
    \end{split}
\end{equation}
Thus, by using \eqref{eq:ncond}, we have that
\begin{equation}\label{eq:cp0}
    \phi(1)=\psi_{\delta}(1).
\end{equation}
Furthermore, 
\begin{equation}\label{eq:cp1}
    \psi_{\delta}'(\lambda) = \frac{1}{\delta}-{\mathbb E}\left\{\left(\frac{Z'}{\lambda-Z'}\right)^2\right\}= \frac{1}{\delta}-{\mathbb E}\left\{\left(\frac{Z}{(\lambda-1)Z+ \lambda}\right)^2\right\}.
\end{equation}
Thus,
\begin{equation}\label{eq:dpsid1}
 \psi_{\delta}'(1)=\frac{1}{\delta}-{\mathbb E}\{Z^2\}>0,  
\end{equation}
where in the last step we combine the hypothesis $\delta > \frac{\E\{ Z^2\}}{ (\E\{ Z G^2 \} \, - \,\E\{Z\})^2}$ with the normalization condition \eqref{eq:ncond}.

Let $\bar{\lambda}_\delta$ be the point at which $\psi_\delta$ attains its minimum, as defined in \eqref{eq:minpsi}, define $\zeta_\delta$ as in \eqref{eq:defzeta} and let $\lambda_\delta^*$ be the unique solution of \eqref{eq:uniqsol}. Since $\psi_\delta$ is convex, \eqref{eq:cp1} and \eqref{eq:dpsid1} imply that $\bar{\lambda}_\delta<1$. Furthermore, \eqref{eq:cp0} implies that $\lambda_\delta^*=1$. By Lemma \ref{lemma:pt}, we obtain that, as $n\to\infty$,
\begin{equation}\label{eq:lambda12M}
\begin{split}
        \lambda_1^{\bM_n} &\stackrel{\mathclap{\mbox{\footnotesize a.s.}}}{\longrightarrow}\delta \,\zeta_\delta(1),\\
                \lambda_2^{\bM_n} &\stackrel{\mathclap{\mbox{\footnotesize a.s.}}}{\longrightarrow}\delta \,\zeta_\delta(\bar{\lambda}_\delta).
\end{split}
\end{equation}
Note that
\begin{equation}\label{eq:fgeq}
    \delta \,\zeta_\delta(1)=\delta \,\phi(1)= \delta\mathbb E\{Z\cdot G^2\} = 1+\delta \mathbb E\{Z\},
\end{equation}
where the first equality comes from the fact that $\lambda_\delta^*=1$ is the unique solution of \eqref{eq:uniqsol},  while the second  and third equalities follow from \eqref{eq:cp00} and \eqref{eq:cp0}. By combining \eqref{eq:fgeq} and \eqref{eq:lambda12M}, we obtain
$\lambda_1^{\bM_n} \stackrel{\mathclap{\mbox{\footnotesize a.s.}}}{\longrightarrow}1+\delta \mathbb E\{Z\}$.
Furthermore, $\zeta_\delta(\bar{\lambda}_\delta)=\psi_\delta(\bar{\lambda}_\delta)$. As $\bar{\lambda}_\delta<1$ and $\psi_\delta$ is Lipschitz continuous (from \eqref{eq:cp1}), there exists a numerical constant $c_1 >0$ such that
\begin{equation}\label{eq:fgeq2}
\delta\psi_\delta(\bar{\lambda}_\delta)\le 1+\delta \mathbb E\{Z\}-c_1.
\end{equation}
Hence,
$\lambda_2^{\bM_n} \stackrel{\mathclap{\mbox{\footnotesize a.s.}}}{\longrightarrow}\delta \,\psi_\delta(\bar{\lambda}_\delta)\le 1+\delta \mathbb E\{Z\}-c_1$.
\end{proof}

Let us now go back to \eqref{eq:GAMPrewrite4} and combine it with \eqref{eq:projvt}. Then, 
\begin{equation}\label{eq:GAMPrewrite5}
  \left(\bA^\sT \bZ(\bZ+\bI_n)^{-1} \bA  -(\delta \E\{Z\}+1) \, \bI_d \right) \br^t  =\be_3^t+\xi_t(\delta \E\{Z\}+1-\lambda_1^{\bM_n})\hat{\bphi}_1.
\end{equation}
We now prove that, almost surely, for all sufficiently large $n$, the following lower bound on the norm of the LHS of \eqref{eq:GAMPrewrite5} holds:
\begin{equation}\label{eq:lbr1t}
    \left\|\left(\bA^\sT \bZ(\bZ+\bI_n)^{-1} \bA  -(\delta \E\{Z\}+1) \, \bI_d \right) \br^t\right\|_2\ge c_2\|\br^t\|_2,
\end{equation}
where $c_2>0$ is a numerical constant independent of $n, t$.

As the matrix $\bA^\sT \bZ(\bZ+\bI_n)^{-1} \bA  -(\delta \E\{Z\}+1) \, \bI_d$ is symmetric, it can be written in the form $\bQ\bLambda\bQ^\sT$, with $\bQ$ orthogonal and $\bLambda$ diagonal. Furthermore, the columns of $\bQ$ are the eigenvectors of $\bA^\sT \bZ(\bZ+\bI_n)^{-1} \bA  -(\delta \E\{Z\}+1) \, \bI_d$ and the diagonal entries  of $\bLambda$ are the corresponding eigenvalues. As $\br^t$ is orthogonal to $\hat{\bphi}_1$,  we can write
\begin{equation}\label{eq:comprt}
    \left(\bA^\sT \bZ(\bZ+\bI_n)^{-1} \bA  -(\delta \E\{Z\}+1) \, \bI_d \right) \br^t=\bQ\bLambda'\bQ^\sT\br^t,
\end{equation}
where $\bLambda'$ is obtained from $\bLambda$ by changing the entry corresponding to $\lambda_1^{\bM_n}-(\delta \E\{Z\}+1)$ to any other value. For our purposes, it suffices to substitute $\lambda_1^{\bM_n}-(\delta \E\{Z\}+1)$ with $\lambda_2^{\bM_n}-(\delta \E\{Z\}+1)$. Note that
\begin{equation}\label{eq:comprt2}
\begin{split}
        \|\bQ\bLambda'\bQ^\sT\br^t\|_2^2 &\ge \|\br^t\|_2^2 \min_{\bs:\|\bs\|=1}\|\bQ\bLambda'\bQ^\sT\bs\|_2^2\\
        &=\|\br^t\|_2^2\min_{\bs:\|\bs\|=1}\<\bs, \bQ\left(\bLambda'\right)^2\bQ^\sT\bs\>\\
        &=\|\br^t\|_2^2\,\lambda_{\rm min}(\bQ\left(\bLambda'\right)^2\bQ^\sT),
\end{split}
\end{equation}
where $\lambda_{\rm min}(\bQ\left(\bLambda'\right)^2\bQ^\sT)$ denotes the smallest eigenvalue of $\bQ\left(\bLambda'\right)^2\bQ^\sT$ and the last equality follows from the variational characterization of the smallest eigenvalue of a symmetric matrix. Note that
\begin{equation}\label{eq:Lambda}
    \lambda_{\rm min}(\bQ\left(\bLambda'\right)^2\bQ^\sT)=\lambda_{\rm min}\left((\bLambda')^2\right)=\min_{i\in\{2, \ldots, d\}}\left(((\delta \E\{Z\}+1)-\lambda_i^{\bM_n})^2\right).
\end{equation}
By using \eqref{eq:lambda2M}, we obtain that, almost surely, for all sufficiently large $n$,
\begin{equation}\label{eq:Lambda2}
    \min_{i\in\{2, \ldots, d\}}\left(((\delta \E\{Z\}+1)-\lambda_i^{\bM_n})^2\right)\ge \left(\frac{c_1}{2}\right)^2.
\end{equation}
By combining \eqref{eq:comprt}, \eqref{eq:comprt2}, \eqref{eq:Lambda} and \eqref{eq:Lambda2}, we conclude that \eqref{eq:lbr1t} holds. 

Recalling that $\br^t$ satisfies \eqref{eq:GAMPrewrite5}, we will next show that, almost surely,
\begin{equation}\label{eq:stp2}
    \lim_{t\to\infty}\lim_{d\to\infty}\frac{1}{\sqrt{d}}\left\|\be_3^t+\xi_t(\delta \E\{Z\}+1-\lambda_1^{\bM_n})\hat{\bphi}_1\right\|_2=0.
\end{equation}
Combined with \eqref{eq:GAMPrewrite5} and \eqref{eq:lbr1t}, this  implies that 
$ \lim_{t\to\infty}\lim_{d\to\infty} \frac{\| \br^t \|_2}{\sqrt{d}} =  0$ almost surely.

By using the triangle inequality, we have 
\begin{equation}\label{eq:trianglenorm}
\begin{split}
        \left\|\be_3^t+\xi_t(\delta \E\{Z\}+1-\lambda_1^{\bM_n})\hat{\bphi}_1\right\|_2
        &\le \left\|\be_3^t\right\|_2+|\xi_t|\cdot |\delta \E\{Z\}+1-\lambda_1^{\bM_n}|\cdot\left\|\hat{\bphi}_1\right\|_2\\
        &\le \left\|\be_3^t\right\|_2+\|\bv^t\|_2\cdot |\delta \E\{Z\}+1-\lambda_1^{\bM_n}|,
        \end{split}
\end{equation}
where the second inequality uses  $\left\|\hat{\bphi}_1\right\|_2=1$ and that $|\xi_t|= \< \bv^t, \hat{\bphi}_1 \> \le \|\bv^t\|_2$.

We can bound the second term on the RHS of \eqref{eq:trianglenorm} using the result in Proposition \ref{prop:GAMP_SE},  applied with the PL(2) test function $\psi(v) = v^2$. Then, almost surely,
\begin{align}
\lim_{d\to\infty} \,  \frac{1}{d}\|\bv^t\|^2_2 = \E\{V_t^2\}=\beta_t^2.
\label{eq:norm_vt2}
\end{align}
Here we used the definitions of $V_t$ and $\beta_t^2$ from \eqref{eq:Vt_def} and \eqref{eq:defbetat1}.
Recalling from \eqref{eq:FPspecial} that  
$\lim_{t\to\infty}  \beta_t^2 =\frac{1}{\delta}$, the limit in \eqref{eq:norm_vt2} combined with Remark \ref{rem:tn_infty} and the continuous mapping theorem implies that, almost surely,
\begin{equation}\label{eq:limvt}
\lim_{t\to\infty}\lim_{d\to\infty}    \frac{1}{\sqrt{d}}\|\bv^t\|_2=\frac{1}{\sqrt{\delta}}.
\end{equation}
Thus, by using \eqref{eq:lambda1Mlim}, we conclude that, almost surely, 
\begin{equation}\label{eq:secondtrm}
    \lim_{t\to\infty}\lim_{n\to\infty}\frac{1}{\sqrt{d}}\|\bv^t\|_2\cdot |\delta \E\{Z\}+1-\lambda_1^{\bM_n}|=0.
\end{equation}

We now bound the first term on the RHS of \eqref{eq:trianglenorm}. Recalling the definition  of $\be_3^t$ in \eqref{eq:deferr3}, an application of the triangle inequality gives
\begin{equation}\label{eq:deferr3cmp}
    \begin{split}
        \left\|\be_3^t\right\|_2&\le\left\|\be_2^t\right\|_2+|1-\sqrt{\delta}\beta_t|\cdot\left\|\bA^\sT \bZ(\bZ+\bI_n)^{-1}(\bZ+\sqrt{\delta}\beta_t\bI_n)^{-1} \bA\bv^t\right\|_2 \\
& \quad  + \delta|\E\{Z\}|\cdot\left|1-\frac{1}{\sqrt{\delta}\beta_t}\right| \cdot\left\|\bv^t\right\|_2+\left\|\bA^\sT \bZ^2(\bZ+\sqrt{\delta}\beta_t\bI_n)^{-1}\be_1^t\right\|_2\\
&\le\left\|\be_2^t\right\|_2+|1-\sqrt{\delta}\beta_t|\cdot\left\|\bA^\sT \bZ(\bZ+\bI_n)^{-1}(\bZ+\sqrt{\delta}\beta_t\bI_n)^{-1} \bA\right\|_{\rm op}\|\bv^t\|_2 \\
&\quad   + \delta|\E\{Z\}|\cdot\left|1-\frac{1}{\sqrt{\delta}\beta_t}\right| \cdot\left\|\bv^t\right\|_2+\left\|\bA^\sT \bZ^2(\bZ+\sqrt{\delta}\beta_t\bI_n)^{-1}\right\|_{\rm op}\|\be_1^t\|_2,
    \end{split}
\end{equation}
where the second inequality follows from the fact that, given a matrix $\bM$ and a vector $\bv$, $\|\bM\bv\|_2\le \|\bM\|_{\rm op} \|\bv\|_2$.

Let us bound the operator norm of the two matrices appearing in the RHS of \eqref{eq:deferr3cmp}. As the operator norm is sub-multiplicative, we have 
\begin{equation}
    \begin{split}
        \left\|\bA^\sT \bZ(\bZ+\bI_n)^{-1}(\bZ+\sqrt{\delta}\beta_t\bI_n)^{-1} \bA\right\|_{\rm op}&\le \left\|\bZ\right\|_{\rm op}\left\|(\bZ+\bI_n)^{-1}\right\|_{\rm op}\left\|(\bZ+\sqrt{\delta}\beta_t\bI_n)^{-1}\right\|_{\rm op} \left\|\bA\right\|_{\rm op}^2,\\
        \left\|\bA^\sT \bZ^2(\bZ+\sqrt{\delta}\beta_t\bI_n)^{-1}\right\|_{\rm op}&\le \left\|\bZ\right\|_{\rm op}^2\left\|(\bZ+\sqrt{\delta}\beta_t\bI_n)^{-1}\right\|_{\rm op}\left\|\bA\right\|_{\rm op}^2.
    \end{split}
\end{equation}
As $Z$ is bounded, 
the operator norm of $\bZ$ is upper bounded by a numerical constant (independent of $n, t$). The operator norm of $(\bZ+\bI_n)^{-1}$ and $(\bZ+\sqrt{\delta}\beta_t\bI_n)^{-1}$ is also upper bounded by a numerical constant (independent of $n, t$). Indeed, from \eqref{eq:FPspecial} $\beta_t \to 1/\sqrt{\delta}$ as $t \to \infty$,  and the support of $Z$ does not contain points arbitrarily close to $-1$. We also have that, almost surely, for all sufficiently large $n$, the operator norm of $\bA$ is upper bounded by a constant (independent of $n, t$). As a result, we deduce that, almost surely, for all sufficiently large $n, t$, 
\begin{equation}\label{eq:opnormmat}
    \begin{split}
        \left\|\bA^\sT \bZ(\bZ+\bI_n)^{-1}(\bZ+\sqrt{\delta}\beta_t\bI_n)^{-1} \bA\right\|_{\rm op}&\le C,\\
        \left\|\bA^\sT \bZ^2(\bZ+\sqrt{\delta}\beta_t\bI_n)^{-1}\right\|_{\rm op}&\le C,
    \end{split}
\end{equation}
where $C$ is a numerical constant (independent of $n, t$). Furthermore, by Lemma \ref{lem:diffs}, the following limits hold almost surely:
\begin{equation}\label{eq:e12lim}
\begin{split}
    \lim_{t\to\infty}\lim_{n\to\infty}\frac{1}{\sqrt{n}}\left\|\be_1^t\right\|_2&=0,\\    \lim_{t\to\infty}\lim_{n\to\infty}\frac{1}{\sqrt{d}}\left\|\be_2^t\right\|_2&=0.
    \end{split}
\end{equation}
By combining \eqref{eq:FPspecial}, \eqref{eq:limvt}, \eqref{eq:opnormmat} and \eqref{eq:e12lim}, we obtain that, almost surely, each of the terms in the RHS of \eqref{eq:deferr3cmp} vanishes when scaled by the factor $1/\sqrt{n}$, as $t, n\to\infty$. As a result, almost surely,
\begin{equation}\label{eq:lime30}
     \lim_{t\to\infty}\lim_{d\to\infty}\frac{1}{\sqrt{d}}\left\|\be_3^t\right\|_2=0.\\   
\end{equation}
By combining \eqref{eq:trianglenorm}, \eqref{eq:secondtrm} and \eqref{eq:lime30}, we conclude that, almost surely, \eqref{eq:stp2} holds. 

Recall that $\br^t$ satisfies \eqref{eq:GAMPrewrite5}. Thus, by combining the lower bound in \eqref{eq:lbr1t} with the almost sure limit in \eqref{eq:stp2}, we obtain that, almost surely,
\begin{equation}\label{eq:limtrt}
     \lim_{t\to\infty}\lim_{d \to\infty}\frac{1}{\sqrt{d}}\left\|\br^t\right\|_2=0.
\end{equation}
Recalling from \eqref{eq:projvt} that $\br^t$ is the component of $\bv^t$ orthogonal to $\hat{\bphi}_1$, the result in \eqref{eq:limtrt} implies that $\bv^t$ tends to be aligned with $\hat{\bphi}_1$  in the high-dimensional limit. In formulas, by combining \eqref{eq:projvt} with \eqref{eq:limtrt}, we have that, almost surely, 
\begin{equation}
\lim_{t\to\infty}\lim_{n\to\infty}\frac{1}{\sqrt{n}}\left\|\bv^t - \xi_t \hat{\bphi}_1\right\|_2=0.
\end{equation}
Note that
\begin{equation}
    \left\|\bv^t - \xi_t \hat{\bphi}_1\right\|_2^2=\left\|\bv^t\right\|_2^2-\xi_t^2.
\end{equation}
Thus, by using \eqref{eq:limvt}, we obtain that, almost surely,
\begin{equation}
\lim_{t\to\infty}\lim_{n\to\infty}    \frac{1}{\sqrt{d}}|\xi_t|=\frac{1}{\sqrt{\delta}}.
\end{equation}
To obtain the sign of $\xi_t$, we first observe that, by Proposition \ref{prop:GAMP_SE}, almost surely, 
\begin{equation}
    \lim_{d \to\infty}\frac{1}{d}\<\bv^t, \bx\>=\mu_{V, t}.
\end{equation}
As $\mu_{V, 0} = \alpha > 0$ and $\E\{ Z(G^2-1)\} = 1/\delta$, the state evolution iteration \eqref{eq:SE_lin2} implies that $\mu_{V, t} >0$ for all $t \ge 0$.  Using \eqref{eq:projvt} we can write
\beq
\frac{1}{d} \< \bv^t, \bx \> = \frac{\xi_t}{\sqrt{d}} \frac{\< \hat{\bphi}_1, \bx \>}{\sqrt{d}} \, + \, \frac{\< \br^t, \bx \>}{d}.
\eeq
Recall that by hypothesis, $\<\hat{\bphi}_1, \bx\>\ge0$. Moreover, using  \eqref{eq:limtrt} and Cauchy-Schwarz, we have $ \lim_{t \to \infty} \lim_{d \to \infty}\frac{\< \br^t, \bx \>}{\sqrt{d}} = 0$ almost surely. Thus  we deduce that   $\lim_{t \to \infty} \lim_{d \to \infty} \frac{\xi_t}{\sqrt{d}} = + \frac{1}{\sqrt{\delta}}$ almost surely.
 Therefore, 
\begin{equation}\label{eq:lim1p}
\lim_{t\to\infty}\lim_{d\to\infty}\frac{1}{\sqrt{d}}\left\|\sqrt{\delta}\bv^t -  \tilde{\bphi}^{(1)}\right\|_2=0 \quad \text{a.s.},
\end{equation}
with $\tilde{\bphi}^{(1)}=\sqrt{d}\hat{\bphi}_1$. 

At this point, we are ready to prove \eqref{eq:psiX1}. For any $\PL(k)$ function $\psi: \reals^3 \to \reals$, we have that 
\begin{equation}\label{eq:lim2p}
\begin{split}
& \left| \frac{1}{d}\sum_{i=1}^d \psi(x_i, \tilde{x}^{\rm L}_i,  \tilde{\varphi}_i^{(1)}) - \frac{1}{d}\sum_{i=1}^d \psi(x_i, \tilde{x}^{\rm L}_i,  \sqrt{\delta}v_i^t) \right| \le  \frac{1}{d}\sum_{i=1}^d \left|\psi(x_i, \tilde{x}^{\rm L}_i, \tilde{\varphi}_i^{(1)}) -  \psi(x_i, \tilde{x}^{\rm L}_i,  \sqrt{\delta}v_i^t) \right|\\
&\ \le \frac{C}{d}\sum_{i=1}^d |\tilde{\varphi}_i^{(1)}-\sqrt{\delta}v_i^t| \left[ 1+\, \left(\big(\tilde{\varphi}_i^{(1)}\big)^2+  (\tilde{x}^{\rm L}_i)^2 +x_i^2 \right)^{(k-1)/2}+ \,
\left( \delta(v_i^t)^2 + (\tilde{x}^{\rm L}_i)^2 +x_i^2\right)^{(k-1)/2} \right]\\
&  \ \le \frac{C}{d}\sum_{i=1}^d |\tilde{\varphi}_i^{(1)}-\sqrt{\delta}v_i^t| \left[ 1+\, 3^{\frac{k-1}{2}}
 \Big( \abs{\tilde{\varphi}_i^{(1)}}^{k-1} +  \abs{\tilde{x}^{\rm L}_i}^{k-1} + \abs{x_i}^{k-1} +  \abs{\sqrt{\delta} v_i^t}^{k-1} + \abs{\tilde{x}^{\rm L}_i}^{k-1} +
 \abs{x}_i^{k-1}  \Big) \right] \\
&\ \le C'\frac{\|\tilde{\bphi}^{(1)} - \sqrt{\delta}\bv^t\|_2}{\sqrt{d}} 
\left[ 1+ \sum_{i=1}^d \left( 
\frac{ |\tilde{\varphi}_i^{(1)}|^{2(k-1)} +  |\tilde{x}^{\rm L}_i|^{2(k-1)}  +  
 |{x}_i|^{2(k-1)}  +    |\sqrt{\delta} v_i^t|^{2(k-1)}}{d}
\right)  \right]^{1/2},
\end{split}
\end{equation}
where $C,C'$ are universal constants (which may depend on $k$ but not on $d, n$). The inequality in the second line above uses $\psi\in {\rm PL}(k)$, and the third and fourth lines are obtained via the H{\" o}lder and Cauchy-Schwarz inequalities. We now claim that, almost surely,
\begin{equation}\label{eq:PLKphi0}
\lim_{d\to\infty}\sum_{i=1}^d \left( 
\frac{ |\tilde{\varphi}_i^{(1)}|^{2(k-1)} +  |\tilde{x}^{\rm L}_i|^{2(k-1)}  +  
 |{x}_i|^{2(k-1)}  +    |\sqrt{\delta} v_i^t|^{2(k-1)}}{d}
\right)\le C,
\end{equation}
where, from now on, we will use $C$ to denote a generic positive constant that does not depend on $d,n$. If \eqref{eq:PLKphi0} holds, then by using \eqref{eq:lim1p} and \eqref{eq:lim2p}, we deduce that, almost surely,
\begin{equation}
\lim_{t\to\infty}\lim_{d\to\infty} \left| \frac{1}{d}\sum_{i=1}^d \psi(x_i, \tilde{x}^{\rm L}_i, \tilde{\varphi}_i^{(1)}) - \frac{1}{d}\sum_{i=1}^d \psi(x_i, \tilde{x}^{\rm L}_i, \sqrt{\delta}v_i^t) \right| = 0.
\label{eq:phi_v_lim}
\end{equation}

Let us now prove \eqref{eq:PLKphi0}. First, by assumption \textbf{(B1)}, we have that
\begin{equation}\label{eq:PLKphi1}
 \lim_{d \to \infty} \frac{1}{d} \sum_{i} |{x}_i|^{2(k-1)} \leq C.   
\end{equation}
Next, the main technical lemma \cite[Lemma 2]{javanmard2013state}  leading to the state evolution result in Proposition \ref{prop:GAMP_SE} implies that, almost surely, for $t\ge 1$,
\begin{equation}\label{eq:PLKphi2}
\limsup_{d \to \infty} \frac{1}{d} \sum_i | v_i^t|^{2(k-1)} \leq C.
\end{equation}
In particular, this follows from \cite[Lemma 2(e)]{javanmard2013state} (see also \cite[Lemma 1(e)]{BM-MPCS-2011}). Since $\tilde{\bx}^{\rm L} = \bv^1$, we also have that, almost surely,
\begin{equation}\label{eq:PLKphi3}
\limsup_{d \to \infty} \frac{1}{d} \sum_i | \tilde{x}_i^{\rm L}|^{2(k-1)} \leq C.    
\end{equation}
It remains to show that, almost surely,
\begin{equation}\label{eq:PLKphi}
 \limsup_{d \to \infty} \frac{1}{d} \sum_i (\tilde{\varphi}_i^{(1)})^{2(k-1)}\le C.   
\end{equation}
To do so, we use a rotational invariance argument. Let $\bR\in \mathbb R^{d \times d}$ be an orthogonal matrix such that $\bR \bx = \bx$. Then, 
\begin{equation}\label{eq:rot_inv_app}
\langle \bx, \ba_i\rangle =\langle \bR \bx, \bR\ba_i \rangle = \langle \bx, \bR\ba_i\rangle .
\end{equation}
Consequently, we have that
\begin{equation}
\bR\bA^\sT \bZ(\bZ+\bI_n)^{-1} \bA \bR^{\sT}  \stackrel{\mathclap{\text{d}}}{=}\bA^\sT \bZ(\bZ+\bI_n)^{-1} \bA,
\end{equation}
which immediately implies that
\begin{equation}
\bR\tilde{\bphi}^{(1)}  \stackrel{\mathclap{\text{d}}}{=}\tilde{\bphi}^{(1)}.
\end{equation}
Then, we can decompose $\tilde{\bphi}^{(1)}$ as
\begin{equation}\label{eq:dectbphi2}
\tilde{\bphi}^{(1)} = a_d \, \bx + \sqrt{1-a_d^2}\, \bphi^{\perp},
\end{equation}
where $\bphi^{\perp}$ is uniformly distributed over the set of vectors orthogonal to $\bx$ with norm $\sqrt{d}$ and
\begin{equation}\label{eq:defad2}
\frac{1}{d}  \< \bx,  \tilde{\bphi}^{(1)}  \> =  a_d.
\end{equation}
Relating the uniform distribution on the sphere to the normal distribution \cite[Sec.~3.3.3]{vershynin2018high}, we can express $\bphi^\perp$ as follows:
\begin{equation}\label{eq:decbphi2}
\bphi^{\perp} =\sqrt{d} \,  \frac{\bu - \displaystyle\frac{1}{d}\langle \bu, \bx\rangle \bx}{\Big\|\bu - \displaystyle\frac{1}{d}\langle \bu, \bx\rangle \bx\Big\|_2},
\end{equation}
where $\bu\sim \normal(\b0_d, \bI_d)$ and independent of $\bx$. By the law of large numbers, we have the almost sure limits
\begin{equation}\label{eq:ascf}
\begin{split}
\lim_{d\to\infty}\frac{1}{d}\langle \bu, \bx\rangle = 0, \qquad \quad 
\lim_{d\to\infty}\frac{1}{d}\Big\|\bu - \displaystyle\frac{1}{d}\langle \bu, \bx\rangle \bx\Big\|_2^2 = 1. 
\end{split}
\end{equation}
Thus, by combining \eqref{eq:dectbphi2} and \eqref{eq:decbphi2}, we conclude that
\begin{equation}\label{eq:rotinvfin}
    \tilde{\bphi}^{(1)} = c_1 \bx +c_2\bu,
\end{equation}
where the coefficients $c_1$ and $c_2$ can be bounded by universal constants (independent of $n, d$) using \eqref{eq:ascf}. As a result, 
\begin{equation}\label{eq:tU}
    \frac{1}{d} \sum_i (\tilde{\varphi}_i^{(1)})^{2(k-1)}\le 2^{2(k-1)} |c_1|^{2(k-1)} \frac{1}{d} \sum_i |x_i|^{2(k-1)}+2^{2(k-1)} |c_2|^{2(k-1)} \frac{1}{d} \sum_i |u_i|^{2(k-1)}.
\end{equation}
Note that, almost surely,
\begin{equation}\label{eq:aslimU}
\lim_{d\to\infty} \frac{1}{d} \sum_i |u_i|^{2(k-1)}=\E\{U^{2(k-1)}\}\le C,
\end{equation}
where $U\sim \normal(0, 1)$. By combining \eqref{eq:PLKphi1}, \eqref{eq:tU} and \eqref{eq:aslimU}, \eqref{eq:PLKphi} immediately follows. Finally, by combining \eqref{eq:PLKphi1}, \eqref{eq:PLKphi2}, \eqref{eq:PLKphi3} and \eqref{eq:PLKphi}, we deduce that \eqref{eq:PLKphi0} holds.

 We now use Proposition \ref{prop:GAMP_SE}  which guarantees that, almost surely,
\beq
 \abs{ \lim_{d \to \infty} \frac{1}{d} \sum_{i=1}^d \psi(x_i,
\tilde{x}^{\rm L}_i, \sqrt{\delta}v_i^t) - \E\{ \psi(X, \, \mu_{V, 1} X + \sigma_{V,1} W_{V,1} , \, \sqrt{\delta}(\mu_{V, t} X + \sigma_{V,t} W_{V,t})) \}} = 0, \qquad  t \geq 1.
\label{eq:XVi}
\eeq
To conclude the proof of \eqref{eq:psiX1}, we take the limit $t \to \infty$ and use Remark \ref{rem:tn_infty}. For this, we will show that 
\beq
\begin{split}
& \lim_{t \to \infty} \E\{ \psi(X, \, \mu_{V, 1} X + \sigma_{V,1} W_{V,1} , \, \sqrt{\delta}(\mu_{V, t} X + \sigma_{V,t} W_{V,t})) \} \\
& =
\E\{ \psi(X, \, \mu_{V, 1} X + \sigma_{V,1} W_{V,1} , \, \sqrt{\delta}(\tilde{\mu}_{V} X + \tilde{\sigma}_{V} W_{V,\infty})) \},
\end{split}
\label{eq:limExpfinal}
\eeq
where $(W_{V,1}, W_{V,\infty})$ are zero mean jointly Gaussian random variables with covariance given by \eqref{eq:WV1inf_cov}. Using  \eqref{eq:WV_corr_init} and the formulas for $g_0$ and $g_t$ from \eqref{eq:g0def} and \eqref{eq:ft_gt_choice}, we have
\beq
\begin{split}
\E\{ W_{V,1} W_{V,t} \} & = 
\frac{1}{\sigma_{V,1} \sigma_{V,t}}\E\{ \cT_L(Y) Z (\mu_{U,t-1} G  + \sigma_{U,t-1}W_{U,t-1}) \}   \\
& = \frac{\mu_{V,t-1}}{\sqrt{\delta} \beta_{t-1} \sigma_{V,1} \sigma_{V,t}} \E\{ \cT_L(Y) Z G \}.
\end{split}
\eeq
In the second equality above, we used \eqref{eq:SE_lin1}. Using the expression
for $\sigma_{V,1}$ from \eqref{eq:SEinit} and letting  $t \to \infty$, we have 
\beq
\lim_{t \to \infty} \E\{ W_{V,1} W_{V,t} \} = \frac{ \tilde{\mu}_{V} \E\{ \cT_L(Y) Z G \}}
{ \tilde{\beta} \tilde{\sigma}_{V} \sqrt{ \E\{ \cT_L(Y)^2\} } } = \E\{ W_{V,1} W_{V, \infty} \}.
\label{eq:WV1Vinf_conv}
\eeq
Therefore, the sequence of zero mean jointly Gaussian pairs $(W_{V,1}, W_{V,t})_{t \geq 1}$ converges in distribution to the jointly Gaussian pair $(W_{V,1}, W_{V,\infty})$, whose covariance is given by \eqref{eq:WV1inf_cov}. 

To show \eqref{eq:limExpfinal}, we use Lemma \ref{lemma:dumbgen} in Appendix \ref{sec:aux}. We apply this result taking $Q_t$ to be the distribution of   $$(X, \,  \mu_{V,1}X + \sigma_{V,1} W_{V,1}, \,  {\mu}_{V,t}X + {\sigma}_{V,t} W_{V,t} ).$$ Since $\mu_{V,t} \to \tilde{\mu}_V$, $\sigma_{V,t} \to \tilde{\sigma}_V$, the sequence  $(Q_t)_{t \geq 2}$ converges weakly to $Q$, which is the distribution of  $$(X, \,  \mu_{V,1}X + \sigma_{V,1} W_{V,1}, \,  \tilde{\mu}_{V}X + \tilde{\sigma}_{V} W_{V,\infty} ).$$ In our case, $\psi: \reals^3 \to \reals$ is PL($k$),  and therefore
$\psi(a,b,c) \leq C'(1 + \abs{a}^k + \abs{b}^k + \abs{c}^k)$, for all $(a,b,c) \in \reals^3$ for some constant $C'$. Choosing $h(a,b,c) = \abs{a}^k + \abs{b}^k + \abs{c}^k$, we have $\frac{\abs{\psi}}{1 + h} \leq C'$. Furthermore, $\int h \, \de Q_t$ is a linear combination of $\{\mu_{V,t}, \mu_{V,t}^2, \ldots, \mu_{V,t}^k, \sigma_{V,t}, \sigma_{V,t}^2, \ldots, \sigma_{V,t}^k  \}$, with coefficients that do not depend on $t$. The integral $\int h \, \de Q$ has the same form, except that  $\mu_{V,t}, \sigma_{V,t}$ are replaced by $\tilde{\mu}_{V}, \tilde{\sigma}_{V}$, respectively. Since $\mu_{V,t} \to \tilde{\mu}_{V}$ and $\sigma_{V,t} \to \tilde{\sigma}_{V}$, we have that $$\lim_{t \to \infty} \int h \, \de Q_t = \int h \, \de Q.$$ Therefore, by applying Lemma \ref{lemma:dumbgen} in Appendix \ref{sec:aux}, we have that
$$\lim_{t \to \infty} \int \psi \, \de Q_t = \int \psi \, \de Q,$$ which is equivalent to  \eqref{eq:limExpfinal}. This completes the proof of the lemma.

\section*{Acknowledgements}

M. Mondelli would like to thank Andrea Montanari for  helpful discussions. M. Mondelli was partially supported by the 2019 Lopez-Loreta Prize. C. Thrampoulidis was partially supported by an NSF award CIF-2009030. R. Venkataramanan was partially supported by the Alan Turing Institute under the EPSRC grant EP/N510129/1.

%%%%%
{\small{
\bibliographystyle{amsalpha}
\bibliography{all-bibliography}
\addcontentsline{toc}{section}{References}
}}

\appendix

%%%%%%%%%%%%%%%%%%%%%%%%%%%%%%%%%%%%%%%%%%%%%%%
\section{Proof of Lemma \ref{lemma:pt_lin}}\label{sec:proof_lin}

% \begin{proof}
By rotational invariance of the Gaussian measure, we can assume without loss of generality that $\bx=\sqrt{d}\be_1=[\sqrt{d},0,\ldots,0]^\sT$. Let us also denote the first column of the matrix $\bA$ by $\bu\in\R^{n}$ and the remaining $n\times (d-1)$ sub-matrix by $\widetilde\bA$, i.e., 
$
\bA = \begin{bmatrix} \bu & \widetilde\bA \end{bmatrix}.
$
In this notation, each measurement $y_i, i\in[n]$, depends only on the corresponding element $u_i$ of the vector $\bu$. In particular, the random variables $z^{\rm L}_i=\cT(y_i), i\in[n]$, are independent of the sub-matrix $\widetilde\bA$. Furthermore, we may express $\bxl$ as follows:
$$
\bxl = \frac{\sqrt{d}}{n} \begin{bmatrix} \bu^\sT\bzl  \\ \widetilde{\bA}^\sT\bzl \end{bmatrix}\,.
$$
We are now ready to prove \eqref{eq:lin_cor}.
First, we compute the correlation $\langle\bxl, \frac{\bx}{\norm{\bx}_2}\rangle$:
\begin{align}\label{eq:cor_lin}
\langle\bxl, \frac{\bx}{\norm{\bx}_2}\rangle = \frac{1}{n} \sqrt{d} (\bu^\sT \bzl) = \frac{1}{n} \sum_{i\in[n]}\sqrt{d} u_i z_i^{\rm L} \ras \E\left\{GZ_L\right\},
\end{align}
where we have that $\sqrt{d} u_i\simiid\Nn(0,1)$ and the almost sure convergence follows from the law of large numbers.

Second, we compute the norm of the estimator $\|\bxl\|_2$:
\begin{align}
\|\bxl\|_2^2 & = \frac{d}{n^2}(\bu^\sT\bzl)^2 + \frac{d}{n^2}\norm{\widetilde{\bA}^\sT\bzl}_2^2 \nonumber \\
& \eqd \frac{1}{n^2}\Big(\sum_{i\in[n]}\sqrt{d} u_i z_i^{\rm L}\Big)^2 + \frac{1}{n^2}\norm{\|\bzl\|_2\bh}_2^2 \ras \left(\E\left\{GZ_L\right\}\right)^2 + \frac{\E\left\{Z_L^2\right\}}{\delta}\,,\label{eq:norm_lin}
\end{align}
where we have used the following: (i) $\sqrt{d}\widetilde{\bA}^\sT\bzl\eqd \|\bzl \|_2\bh$ with $\bh\sim\Nn(0,\bI_{d-1})$; (ii) $\| \bzl \|_2^2/n\ras\E\left\{Z_L^2\right\}$ and $\| \bh \|_2^2/n\ras1/\delta$, by the law of large numbers.

Combining the above displays completes the proof of the lemma.
%\end{proof}

%%%%%%%%%%%%%%%%%%%%%%%%%%%%%%%%%%%%%%%%%%%%%%%

\section{Proof of Corollary \ref{lem:fopt}} \label{app:fopt}

Consider 
$
F(\theta)=\frac{\theta\rho_L+\rho_s}{\sqrt{1+\theta^2+2\theta q}}.
$
It can be checked that
\begin{align}\label{eq:Fder}
F^\prime(\theta) = \frac{(\rho_L-q\rho_s) - \theta (\rho_s-q\rho_L)}{(1+\theta^2+2\theta q)^{3/2}}.
\end{align}
We consider three cases.

\vp
\underline{Case 1:~$\rho_s=\rho_L q$.}~Here, $F$ is either strictly increasing or strictly decreasing depending on the sign of $\rho_L-q\rho_s$. But, $q\in(-1,1)$, thus, $\rho_s<|\rho_L|\implies \sign(\rho_L-q\rho_s)=\sign(\rho_L)$. Thus, $F$ is maximized at $\tilde\theta\rightarrow\sign(\rho_L)\cdot\infty$ and approaches the value $|\rho_L|$. Moreover, $F(\theta)\leq |\rho_L|$. To conclude with the desired, notice that if $\rho_s=\rho_L q$, then $\theta_*$ and $F(\theta_*)$ defined in the lemma take the values $\theta_*=\sign(\rho_L)\cdot\infty$ and $F(\theta_*)=|\rho_L|$, respectively.

\vp
\underline{Case 2:~$\rho_s>\rho_L q$.}~Here, it can be readily checked from \eqref{eq:Fder} that $F$ is maximized at $\tilde\theta:=\frac{\rho_L-\rho_s q}{\rho_s-\rho_L q}$. Also, a bit of algebra shows that 
$$
F(\tilde\theta)=\sqrt{\frac{\rho_s^2+\rho_L^2-2q\rho_L\rho_s}{1-q^2}}=F(\theta_*).
$$
Thus, $|F(\theta)|$ is maximized either at $\tilde\theta$ or as $\theta$ approaches $\pm\infty$. But, $F(\theta_*)^2-\rho_L^2=\frac{(\rho_s-q\rho_L)^2}{1-q^2}> 0 \Rightarrow F(\tilde\theta)>|\rho_L|$. Hence, $|F(\theta)|$ is indeed maximized at $\tilde\theta=\theta_*$ and attains the value $F(\tilde\theta)=F(\theta_*)$.

\vp
\underline{Case 3:~$\rho_s<\rho_L q$.}~Here, it can be checked from \eqref{eq:Fder} that $F$ is minimized at $\tilde\theta:=\frac{\rho_L-\rho_s q}{\rho_s-\rho_L q}$ and the minimum value is $F(\tilde\theta)=-F(\theta_*).$ Moreover, similar to Case 2 above, $F(\theta_*)=|F(\tilde\theta)|>|\rho_L|$. Thus, again, $|F(\theta)|$ is maximized at $\theta_*$ and taking the value $F(\theta_*)$.

This completes the proof of the result.

%%%%%%%%%%%%%%%%%%%%%%%%%
\section{Optimization of Preprocessing Functions}\label{app:optpre}

In order to state the results in this section, let us define the following functions for $y\in\R$ and $G\sim\mathcal{N}(0,1)$,
\begin{subequations}\label{eq:mus}
\begin{align}
    \mu_0(y) &= \E_G[p_{Y\mid G}(y\mid G)],\\
    \mu_1(y) &= \E_G[G\cdot p_{Y\mid G}(y\mid G)],\\
    \mu_2(y) &= \E_G[G^2\cdot p_{Y\mid G}(y\mid G)].
\end{align}
\end{subequations}
Note that the functions $\mu_0,\mu_1$ and $\mu_2$ only depend on the conditional distribution $p_{Y\mid G}(\, \cdot \, | \,  G)$. Furthermore, let $\cS$ denote the support of the probability measure $Y$ (i.e., the support of $\mu_0(y)$).

\subsection{Linear Estimator}\label{sec:lin_opt}

In terms of the notation in \eqref{eq:mus}, for a preprocessing function $\Tc_L(y)$, we can write 
$$
\E\{GZ_L\} = \int_{\cS} \Tc_L(y) \mu_1(y) \mathrm{d}y\quad\text{and}\quad
\E\{Z_L^2\} = \int_{\cS} \Tc^2_L(y) \mu_0(y) \mathrm{d}y.
$$ 
Thus, it follows from \eqref{eq:lin_cor} in Lemma \ref{lemma:pt_lin} that
\begin{align}\label{eq:rhoL_mus}
|\rho_L| = \left(1+\frac{1}{\delta}\frac{\int_{\cS} \Tc_L^2(y)\mu_0(y)\mathrm{d}y}{\left(\int_{\cS} \Tc_L(y)\mu_1(y)\mathrm{d}y\right)^2}\right)^{-1/2}\,,
\end{align}
 provided $\int_{\cS} \Tc_L(y) \mu_1(y) \mathrm{d}y\neq 0$ and $\E\{|GZ_L|\}<\infty.$

Assume, henceforth, that 
\begin{align}\label{eq:lin_opt_cond}
    0<\int_{\cS} \frac{ \mu_1^2(y) }{\mu_0(y)} \mathrm{d}y\,<\infty.
\end{align}
Then, we will show in this section that 
the optimal preprocessing function for the linear estimator is
\begin{align}\label{eq:Tcl_def}
    \Tc^*_L(y) = \frac{\mu_1(y)}{\mu_0(y)}\,,
\end{align}
and the achieved (optimal) normalized correlation is
\begin{align}\label{eq:rhoL*}
    \rho_L^* = \left(1+\frac{1}{\delta\int_{\cS}\frac{\mu_1^2(y)}{\mu_0(y)}\mathrm{d}y}\right)^{-1/2}\,.
\end{align}

To see this, note from \eqref{eq:rhoL_mus} that $\rho_L^2$ is maximized for the choice of $\Tc_L$ that minimizes the ratio ${\int_{\cS} \Tc_L^2(y)\mu_0(y)\mathrm{d}y}\big/{\left(\int_{\cS} \Tc_L(y)\mu_1(y)\mathrm{d}y\right)^2}$, while ensuring $\int_{\cS} \Tc_L(y) \mu_1(y) \mathrm{d}y\neq0$. Furthermore, by using the Cauchy-Schwarz inequality, we obtain:
\begin{align}
    \left(\int_{\cS} \Tc_L(y)\mu_1(y)\mathrm{d}y\right)^2 = \left(\int_{\cS} \Tc_L(y)\sqrt{\mu_0(y)}\frac{\mu_1(y)}{\sqrt{\mu_0(y)}}\mathrm{d}y\right)^2 \leq \left(\int_{\cS} \Tc_L^2(y){\mu_0(y)}\mathrm{d}y\right) \left(\int_{\cS} \frac{\mu_1^2(y)}{\mu_0(y)}\mathrm{d}y\right)\,.
\end{align}
Rearranging the above and substituting in the expression for $\abs{\rho_L}$ from \eqref{eq:rhoL_mus} yields $\rho_L^2 \leq \left(\rho_L^*\right)^2$, with equality achieved if and only if $\Tc_L(y)=c\cdot \frac{\mu_1(y)}{\mu_0(y)}, \forall y\in\R$ and some constant $c>0$. Clearly, the correlation performance of $\bxl$ is insensitive to scaling $\Tc_L$ by a constant. Thus, we can choose $c=1$ to arrive at \eqref{eq:Tcl_def}. To complete the proof of the claim, note that for the choice in \eqref{eq:Tcl_def}:
$$
\int_{\cS} \Tc_L^*(y) \mu_1(y) \mathrm{d}y = \int_{\cS} \frac{ \mu_1^2(y) }{\mu_0(y)} \mathrm{d}y\,>0,
$$
and 
$$
\E\{|G Z_L|\} \leq\sqrt{\E\{G^2\}}\left(\int_\cS (\cT_L^*(y))^2\mu_0(y)\mathrm{d}y\right)^{1/2} = \sqrt{\E\{G^2\}}\left(\int_\cS \frac{\mu_1^2(y)}{\mu_0(y)}\mathrm{d}y\right)^{1/2} <\infty,
$$
where the last inequalities in the above lines follow from \eqref{eq:lin_opt_cond}.

As a final note, observe that the optimal $\Tc^*_L$ does \emph{not} depend on the sampling ratio $\delta$.

\subsection{Spectral Estimator}\label{sec:spec_opt}
The optimal preprocessing function for the spectral estimator is derived in \cite{luo2019optimal}. For ease of reference, we present here their result in the special case where 
$\inf_y \frac{\mu_2(y)}{\mu_0(y)}>0.$
If this condition does not hold, the idea is to construct a sequence of approximations of the optimal preprocessing function (we refer the reader to \cite{luo2019optimal} for the details).

Assume $\delta\geq \delta^*$, where 
\begin{align}\label{eq:delta_star}
    \delta^*:= \left(\int_{\cS}\frac{(\mu_2(y)-\mu_0(y))^2}{\mu_0(y)} \mathrm{d}y\right)^{-1}
\end{align}
is the threshold for weak recovery of the spectral estimator \cite{mondelli2017fundamental}. For a preprocessing function $\Tc_s(y)$, we have from Lemma \ref{lemma:pt} that 
$$
\rho_s = \left(1+\frac{\int_{\cS} \left(\frac{\Tc_s(y)}{\lambda_\delta^*-\Tc_s(y)}\right)^2\mu_2(y)\mathrm{d}y}{\frac{1}{\delta}-\int_{\cS} \left(\frac{\Tc_s(y)}{\lambda_\delta^*-\Tc_s(y)}\right)^2\mu_0(y)\mathrm{d}y}\right)^{-1/2}\,,
$$
where $\lambda_\delta^*$ is the unique solution to the following equation for $\lambda\geq\tau$:
\begin{align}\label{eq:spec_opt1}
\int_{\cS}\frac{\Tc_s(y)}{\lambda-\Tc_s(y)}\left(\mu_2(y)-\mu_0(y)\right)\mathrm{d}y=\frac{1}{\delta}\,,
\end{align}
and also (cf. $\psi_\delta^\prime(\lambda_\delta^*)\geq0$),
\begin{align}\label{eq:spec_opt2}
\int_{\cS} \left(\frac{\Tc_s(y)}{\lambda_\delta^*-\Tc_s(y)}\right)^2\mu_0(y) \mathrm{d}y \leq \frac{1}{\delta}\,.
\end{align}

Using this characterization, \cite{luo2019optimal} shows that, the optimal preprocessing function for the spectral estimator is
\begin{align}\label{eq:Tcs_def}
    \Tc^*_s(y) = 1-\frac{\mu_0(y)}{\mu_2(y)}\,,
\end{align}
and the achieved (optimal) normalized correlation is
\begin{align}\label{eq:rhos_star_def}
    \rho_s^* = \left(1+\beta_\delta\right)^{-1/2},\quad\text{where}~
    \int_{\cS}\frac{(\mu_2(y)-\mu_0(y))^2}{\mu_0(y)+\mu_2(y)/\beta_\delta}\mathrm{d}y = \frac{1}{\delta}.
\end{align}
As for the linear estimator, the optimal function $\Tc^*_s$ does \emph{not} depend on the sampling ratio $\delta$.

\subsection{Spectral vs Linear}\label{sec:lin_vs_spec}
As mentioned in the introduction, there is no clear winner between the linear and the spectral estimator: the superiority of one method over the other depends on the measurement model and on the sampling ratio. Here, we fix the measurement model (i.e., the stochastic output function $p_{Y|G}(\, \cdot \, | \,  \inp{\bx}{\ba_i})$) and we present an analytic condition that determines which method is superior for any given $\delta>0$ after optimizing both in terms of the preprocessing function. 
\begin{lemma}
Assume that  $\inf_y \frac{\mu_2(y)}{\mu_0(y)}>0$ and \eqref{eq:lin_opt_cond} hold. 
Consider the  function $h:\R_+\rightarrow\R_+$,
\begin{align}
h(t):=    \int_{\cS}\frac{(\mu_2(y)-\mu_0(y))^2}{\mu_0(y)+\mu_2(y)/t}\mathrm{d}y,
\end{align}
and let 
$\gamma_\delta:=\left({\delta\int_{\cS}\frac{\mu_1^2(y)}{\mu_0(y)}\mathrm{d}y}\right)^{-1}$. Then, the following holds:
\begin{align}
    \delta\cdot h( \gamma_\delta ) \lessgtr 1 ~\Longrightarrow~ \rho_s^* \lessgtr \rho_L^*,
\end{align}
where $\rho_L^*$ and $\rho_s^*$ are defined in \eqref{eq:rhoL*} and \eqref{eq:rhos_star_def}, and denote the optimal normalized correlation for the linear and the spectral estimator, respectively.
\end{lemma}
\begin{proof}
It can be checked by direct differentiation and the fact that $\mu_2(y)>0$, that $h(\cdot)$ is strictly increasing. Thus,  from \eqref{eq:rhos_star_def}, we find that $h( \gamma_\delta ) \lessgtr 1/\delta ~\Longrightarrow~ \beta_\delta \gtrless \gamma_\delta$. To conclude the proof, note from \eqref{eq:rhoL*} and \eqref{eq:rhos_star_def} that $\rho_L^*=u(\gamma_\delta)$ and  $\rho_s^*=u(\beta_\delta)$, respectively, where we define $u(t) = (1+t)^{-1/2}$ and $u(\cdot)$ is strictly decreasing.
\end{proof}

\subsection{Combined Estimator}\label{sec:comb_opt}
In the previous sections of Appendix \ref{app:optpre}, we have discussed how to optimally choose $\Tc_L$ and $\Tc_s$  to maximize the correlation of the linear and spectral estimators. This was possible thanks to the asymptotic characterizations in Lemmas \ref{lemma:pt_lin} and \ref{lemma:pt}. Theorem \ref{th:optimality} opens up the possibility to optimally choose $\Tc_L$ and $\Tc_s$ to maximize the correlation achieved by the Bayes-optimal combination $F_*(\bx^{\rm L}, \bx^{\rm s})$. Here, we focus on the special case in which the signal prior is Gaussian and, hence, $F_*(\bx^{\rm L}, \bx^{\rm s})$ is a linear combination of $\bx^{\rm L}$ and  $\bx^{\rm s}$, see Corollary \ref{lem:fopt}. In the rest of this section, we formalize the problem of (optimally) choosing the functions  $\Tc_L$ and $\Tc_s$. 

To begin, note from \eqref{eq:q_def_compact} that $q=\rho_L \cdot\rho_s \cdot s$, where we define
\begin{align}\label{eq:q_def_mus}
    s := 
     \frac{\int_{\cS}{\frac{\Tc_L(y)\mu_1(y)}{1-\frac{1}{\lambda_\delta^*}\Tc_s(y)}\mathrm{d}y}}{\int_{\cS}{\Tc_L(y)\mu_1(y)\mathrm{d}y}} \,.
\end{align}
Furthermore, by using \eqref{eq:q_def_mus} in \eqref{eq:rho_star_def}, we can express the achieved correlation $F(\theta_*)$ of $\bxc(\theta_*)$ as follows
\begin{align}
    F^2(\theta_*) &= \frac{\rho_s^2+\rho_L^2-2\rho_s^2\rho_L^2s}{1-\rho_s^2\rho_L^2s^2} =  \frac{\frac{1}{\rho_s^2}+\frac{1}{\rho_L^2}-2s}{\frac{1}{\rho_s^2\rho_L^2}-s^2} \label{eq:division_zero}\\
    &= \frac{2-2s+\frac{1}{\delta}\frac{\int_{\cS}{\Tc_L^2\mu_0}}{(\int_{\cS}{\Tc_L\mu_1)^2}} + \frac{\int_{\cS}{\left(\frac{\Tct_s}{1-\Tct_s}\right)^2\mu_2}}{\frac{1}{\delta} -\int_{\cS}{\left(\frac{\Tct_s}{1-\Tct_s}\right)^2\mu_0} } }{ \left(1+\frac{1}{\delta}\frac{\int_{\cS}{\Tc_L^2\mu_0}}{(\int_{\cS}{\Tc_L\mu_1)^2}}\right)\left(1+\frac{\int_{\cS}{\left(\frac{\Tct_s}{1-\Tct_s}\right)^2\mu_2}}{\frac{1}{\delta} -\int_{\cS}{\left(\frac{\Tct_s}{1-\Tct_s}\right)^2\mu_0} }\right) -s^2},\label{eq:rho_star_open}
\end{align}
where we have denoted $\Tct(y):=\frac{1}{\lambda_\delta^*}\Tc(y)$ and all integrals are over $y$ (not explicitly written for brevity). Thus, the problem of choosing $\Tc_L$ and $\Tc_s$ can be reformulated as follows:
\begin{align}
    &\max_{\gamma,\Tc_L,\Tct_s}~\gamma\nn\\
    &~\text{s.t.}~~~~\eqref{eq:rho_star_open} \geq \gamma \nn\\
    &~~~~~~~~\int_{\cS}\frac{\Tct_s(y)}{1-\Tct_s(y)}\left(\mu_2(y)-\mu_0(y)\right)\mathrm{d}y=\frac{1}{\delta}\nn\\
    &~~~~~~~~\int_{\cS}\left(\frac{\Tct_s(y)}{1-\Tct_s(y)}\right)^2\mu_0(y) \mathrm{d}y\leq\frac{1}{\delta}\nn\\
    &~~~~~~~~\int_{\cS}\Tc_L(y)\mu_1(y) \mathrm{d}y > 0\nn\,.
\end{align}
The second and the third constraints above follow from \eqref{eq:spec_opt1} and \eqref{eq:spec_opt2}, respectively. These further guarantee that $\rho_s>0$ (so division in \eqref{eq:division_zero} is allowed). Similarly, the last constraint on $\Tc_L$ ensures that $\rho_L^2>0$. 

Though concrete, the formulation above is a difficult function optimization problem. Solving this goes beyond the scope of this paper, but it may be an interesting future direction. Another related question is whether the solution to this problem coincides (or not) with the ``individually" optimal choices in \eqref{eq:Tcl_def} and \eqref{eq:Tcs_def}, respectively.
% For the latter choise the optimal preprocessing functions in \eqref{eq:Tcl_def} and \eqref{eq:Tcs_def}, it holds that
% \begin{align}
%     q = \rho_L^* \cdot  \rho_s^* \cdot\frac{\int{\frac{\mu_1^2(y)\mu_2(y)}{\mu_0^2(y)}\mathrm{d}y}}{\int{\frac{\mu_1^2(y)}{\mu_0(y)}\mathrm{d}y}}
% \end{align}
%%%%%%%%%%%%%%%%%%%%

%%%%%%%%%%%%%%%%%%%%%%%%%%%%%%%%%%%%%%%%%%%%%%%%%%%%%%%%%%%%%%%%
\section{Example: Bayes-optimal Combination for Binary Prior}\label{sec:Bayes_Bern}
%%%%%%%%%%%%%%%%%%%%%%%%%%%%%%%%%%%%%%%%%%%%%%%%%%%%%%%%%%%%%%%%
In this section, we evaluate explicitly the Bayes-optimal estimator $F_*(x_L,x_s) = \E\{ X\,|\,X_L=x_L,X_s=x_s\}$ in \eqref{eq:Bayesopt} for the case where
$X\in\{1, -1\}$ with $P_X(1)=\sfp$ and $P_X(-1)=1-\sfp$.

Using this prior, we obtain
\begin{align}
    F_*(x_L,x_s)=\E\{X\,|\,X_L=x_L,X_s=x_s\} &= \E\{X\,|\,\rho_LX+W_L=x_L,\rho_sX+W_s=x_s\}\nn \\
    &= 2\mathbb P(X=1\,|\,\rho_LX+W_L=x_L,\rho_sX+W_s=x_s) - 1 \nn \\
    &= \frac{2}{1+\frac{1-\sfp}{\sfp} \frac{p_{W_L,W_S}(x_L+\rho_L,x_s+\rho_s)}{p_{W_L,W_S}(x_L-\rho_L,x_s-\rho_s)}} - 1,\label{eq:Bayes_bern}
\end{align}
where the last line follows by Bayes rule and simple algebra. Here, $p_{W_L,W_s}$ denotes the joint density of $(W_L,W_s)$ as predicted by Theorem \ref{th:W2conv2}, i.e.,
\begin{align}\label{eq:pWW}
p_{W_L,W_s}(w_L,w_s) = \frac{1}{C}\exp\Big({-\frac{1}{2}
\begin{bmatrix}
w_L & w_s
\end{bmatrix} \Sigma^{-1} \begin{bmatrix}
w_L \\ w_s
\end{bmatrix}\Big)
},
\end{align}
where $\Sigma=\begin{bmatrix}
1-\rho_L^2 & q-\rho_L\rho_s \\ q-\rho_L\rho_s & 1-\rho_s^2
\end{bmatrix},$
and $C$ is a constant that is irrelevant for our purpose as it cancels in \eqref{eq:Bayes_bern}. Using \eqref{eq:pWW} in \eqref{eq:Bayes_bern} gives an explicit expression for $F_*(x_L,x_s)$.

In Section \ref{sec:num_bayes}, we numerically implement the optimal combined estimator for various measurement models. Specifically, we use the linear and spectral estimators $\bx^{\rm L}$ and $\bx^{\rm s}$ to form the combined estimator
$$
\hat\bx^{\rm mmse} = F_*(\bx^{\rm L},\bx^{\rm s}),
$$
where $F_*$ acts element-wise on the entries of its arguments as specified in \eqref{eq:Bayes_bern}. The  asymptotic correlation of the estimator $\hat\bx^{\rm mmse}$ is given by Theorem \ref{th:optimality} as follows: $\rho_*=\frac{|\E\left\{X\cdot F_*(X_L,X_s)\right\}|}{\sqrt{\E\left\{F_*^2(X_L,X_s)\right\}}}$ (see \eqref{eq:optres0}). Equipped with the explicit expression for $F_*(X_L,X_s)$ in \eqref{eq:Bayes_bern}, we can compute  $\rho_*$ using Monte Carlo averaging over realizations of the triplet $(X,X_L,X_s)$.

%%%%%%%%%%%%%%%%%%
\section{Proof of Lemma \ref{lem:fixed_pts}} \label{app:lem_fixed_pts_proof}
Take $t \to \infty$ in \eqref{eq:SE_lin2}, and let  $\mu_V = \lim_{t \to \infty} \mu_{V,t}, \, 
\sigma_V^2 = \lim_{t \to \infty} \sigma_{V,t}^2$, $\beta^2 = \sigma_V^2 + \mu_V^2$. Then, by solving these equations, we obtain two solutions for the pair $(\mu_V^2, \sigma_V^2)$: one solution gives the fixed point $\FP_0$; and the other solution gives $\FP_1$ and $\FP_2$.
Note that the fixed points $\FP_1$ and $\FP_2$ exist only when $\tilde{\beta}^2 > \E\{Z^2\}$. From \eqref{eq:beta2_def}, this is equivalent to $\delta > \frac{\E\{ Z^2\}}{ (\E\{ Z G^2 \} \, - \,\E\{Z\})^2}$, which is the condition in the statement of the lemma.

Let us define $$\gamma_t^2 \equiv \frac{\mu_{V,t}^2}{\sigma_{V,t}^2}.$$
Using this definition, the two equations in \eqref{eq:SE_lin2} can  be combined to obtain the following recursion in $\gamma_t^2$:
\beq
\gamma_{t+1}^2 = 
\frac{\delta (\E\{ Z G^2\} - \E\{Z\})^2}{\E\{Z^2G^2\} + \E\{ Z^2\}/\gamma_t^2 }.
\label{eq:gamma_iter}
\eeq
Note that $\E\{ Z^2\}>0$. In fact, if $\E\{ Z^2\}=0$, then $\mathbb P(Z=0)=1$ and the condition $\E\{ Z(G^2-1)\} >0$ cannot hold. Thus, the two fixed points of this recursion are $\gamma^2_{\FP_0}=0$, and 
\beq
\gamma^2_{\FP_{12}}=\frac{ \delta(\E\{ ZG^2\} - \E\{Z\})^2 - \E\{Z^2\} }{\E\{ Z^2 G^2 \}} = \frac{\tilde{\beta}^2 - \E\{Z^2\}}{\E\{ Z^2 G^2 \}} = \frac{\tilde{\mu}_V^2}{\tilde{\sigma}_V^2}.
\label{eq:gam_FP12}
\eeq
As discussed above, the fixed point $\gamma^2_{\FP_{12}}$ exists when  $\delta > \frac{\E\{ Z^2\}}{ (\E\{ Z G^2 \} \, - \,\E\{Z\})^2}$. The recursion can be written as $\gamma_{t+1}^2 = f(\gamma_t^2)$, where 
\[ 
f(x) = \frac{\delta (\E\{ZG^2\} - \E\{Z\})^2}{\E\{Z^2G^2\} + \E\{ Z^2\}/x}.
\]
The derivative of $f$ is given by 
\beq
f'(x) = \frac{\delta \, (\E\{ZG^2\} - \E\{Z\})^2\,\E\{Z^2\} }{(\E\{Z^2\} \, +  \,  x \,  \E\{Z^2 G^2\})^2}. 
\label{eq:fpr_x}
\eeq
 
We now argue that, whenever $\gamma_1^2 = \frac{\mu_{V,1}^2}{\sigma_{V,1}^2}$ is strictly positive, the recursion $\gamma_{t+1}^2 = f(\gamma_t^2)$ converges to $\gamma^2_{\FP_{12}}$. We will separately consider the two cases $\gamma_1^2 < \gamma^2_{\FP_{12}}$ and $\gamma_1^2 > \gamma^2_{\FP_{12}}$. 

We consider first the case $\gamma_1^2 < \gamma^2_{\FP_{12}}$. Since $f'(x) >0$, the function $f(x)$ is strictly increasing for $x \geq 0$. 
Therefore, if for any $t \geq 0$ we have  $\gamma_t^2 < \gamma^2_{\FP_{12}}$, then $\gamma_{t+1}^2 = f(\gamma_t^2) < f(\gamma^2_{\FP_{12}}) = \gamma^2_{\FP_{12}}$. Next, we argue that $f(x) > x$ for $0 < x < \gamma^2_{\FP_{12}}$. 
 To show this claim, we first note that $f'(0) >1$ since $\delta > \frac{\E\{ Z^2\}}{ (\E\{ Z G^2 \} \, - \,\E\{Z\})^2}$. Thus, $f(x) > x$ for $x$ sufficiently close to $0$.  If $f(\gamma') \leq \gamma'$ for some $0 < \gamma' < \gamma^2_{\FP_{12}}$, then there exists a fixed point between $0$ and $\gamma'$, as $f(x)$ is continuous. However, this is not possible since $\gamma' < \gamma^2_{\FP_{12}}$ and $\gamma^2_{\FP_{12}}$ is the unique  fixed point $> 0$. As a result, $f(x) > x$ for $0 < x < \gamma^2_{\FP_{12}}$. Hence, if $\gamma_1^2 < \gamma^2_{\FP_{12}}$, then the sequence  $(\gamma_t^2)_{t \geq 1}$ is strictly increasing and bounded above by  $\gamma^2_{\FP_{12}}$. 
Furthermore, by using the uniqueness of the fixed point, one obtains that its supremum is $\gamma^2_{\FP_{12}}$. Therefore, the sequence 
$(\gamma_t^2)_{t \geq 1}$ converges to $\gamma^2_{\FP_{12}}$.

Next, consider the case $\gamma_1^2 > \gamma^2_{\FP_{12}}$. We observe that
\beq
f'(\gamma^2_{\FP_{12}}) =\frac{\E\{Z^2\}}{\delta (\E\{ZG^2\} - \E\{Z\})^2 } < 1
\eeq
since $\delta > \frac{\E\{ Z^2\}}{ (\E\{ Z G^2 \} \, - \,\E\{Z\})^2}$. From \eqref{eq:fpr_x}, we see that $f'(x)$ is strictly decreasing for $x >0$, hence $f'(x) <1$ for $x \geq \gamma^2_{\FP_{12}}$.   Therefore, by the Banach fixed point theorem, the iteration \eqref{eq:gamma_iter} converges to $\gamma^2_{\FP_{12}}$ whenever $\gamma_1^2 \geq \gamma^2_{\FP_{12}}$.   

Finally, we observe from  \eqref{eq:SE_lin2} that:
\beq
\mu_{V, t+1}^2 = \delta (\E\{Z(G^2-1)\})^2 \frac{\gamma_t^2}{1+ \gamma_t^2}, \qquad \sigma_{V, t+1}^2 = \frac{\gamma_t^2}{1+ \gamma_t^2} \E\{Z^2 G^2\} \,
\, + \, \frac{1}{1+ \gamma_t^2} \, \E\{ Z^2 \}. 
\label{eq:mut1_sigt1}
\eeq
Thus, for any initialization such that $\gamma_1^2 > 0$,
 \beq
 \begin{split}
& \lim_{t \to \infty} \mu_{V, t+1}^2  = \delta (\E\{Z(G^2-1)\})^2 \frac{\gamma^2_{\FP_{12}}}{1+ \gamma^2_{\FP_{12}}} = \tilde{\mu}_V^2, \\
& \lim_{t \to \infty} \sigma_{V, t+1}^2 = 
 \frac{\gamma^2_{\FP_{12}}}{1+ \gamma^2_{\FP_{12}}} \E\{Z^2 G^2\}  \, + \, \frac{1}{1+ \gamma^2_{\FP_{12}}} \, \E\{ Z^2 \} = \tilde{\sigma}_V^2,
 \end{split}
\eeq
where we have used the expression for $\gamma^2_{\FP_{12}}$ from \eqref{eq:gam_FP12}. Note that both fixed points $\FP_1$ and $\FP_2$ correspond to the same $(\tilde{\mu}_V^2, \tilde{\sigma}_V^2)$. The assumption that $\E\{ Z(G^2-1)\} >0$ ensures that the sign of $\mu_{V,t+1}$ in \eqref{eq:SE_lin2} remains unchanged and hence the iteration converges to either $\FP_1$ or $\FP_2$ depending  on the sign of $\mu_{V,1}$.
%%%%%%%%%

\section{Proof of Lemma \ref{lem:diffs} }
\label{app:proof_lem_diffs}

We first consider a fixed $t \geq 2$ and write  
\begin{align}
\frac{1}{n} \| \bu^t - \bu^{t-1} \|^2 = \frac{\|\bu^t \|^2}{n}  \, + \,  \frac{\| \bu^{t-1} \|^2}{n} \, - \,  2 \frac{\< \bu^t, \, \bu^{t-1} \>}{n}.      
\end{align}   
Applying Proposition \ref{prop:GAMP_SE} with the PL(2) functions $\psi(a) = a^2$ (for the first two terms) and $\psi(a, b) = ab$ (for the last term), we obtain
\beq
\begin{split}
& \lim_{n \to \infty } \frac{1}{n} \| \bu^t - \bu^{t-1} \|^2  \,   \stackrel{\text{a.s.}}{=} \,  
\E\Big\{ (\mu_{U,t} G + \sigma_{U,t} W_{U,t} )^2 \Big\} \, + \, 
\E\Big\{ (\mu_{U,t-1} G + \sigma_{U,t-1} W_{U,{t-1}} )^2 \Big\}  \\
& \hspace{1.5in} - 2 \E\Big\{ (\mu_{U,t} G + \sigma_{U,t} W_{U,t}) (\mu_{U,t-1} G + \sigma_{U,t-1} W_{U,{t-1}} ) \Big\} \\
&= \mu_{U,t}^2 \, + \, \sigma_{U,t}^2 \, + \, \mu_{U,t-1}^2 \, + \, \sigma_{U,t-1}^2
\, - \, 2 \mu_{U,t} \mu_{U,t-1} \, - \, 2 \sigma_{U,t} \sigma_{U,t-1} \E\{W_{U,t-1} W_{U,t}\}.
\end{split}
\label{eq:ut1ut_diff}
\eeq
Similarly,  we have for any $t \geq 1$
\beq
\begin{split}
& \lim_{d \to \infty } \frac{1}{d} \| \bv^{t+1} - \bv^{t} \|^2   
\,  \stackrel{\text{a.s.}}{=} \, \mu_{V,t+1}^2 + \sigma_{V,t+1}^2  +  \mu_{V,t}^2  +  \sigma_{V,t}^2
 -  2 \mu_{V,t} \mu_{V,t+1}  -  2 \sigma_{V,t+1} \sigma_{V,t} \E\{W_{V,t+1} W_{V,t}\}.
\end{split}
\label{eq:vt1vt_diff}
\eeq

Since  $\abs{\E\{ \cT_L(Y) G \}} >0$, the initialization  $\mu_{V,1}$ of the state evolution recursion in \eqref{eq:SEinit}  is strictly non-zero.  Therefore Lemma \ref{lem:fixed_pts} guarantees that the state evolution recursion \eqref{eq:SE_lin1}-\eqref{eq:SE_lin2} converges to either the fixed point $\FP_1$ or to $\FP_2$ depending on the sign of $\mu_{V,1}$. Without loss of generality assume that $\mu_{V,1} >0$, so that the recursion converges to $\FP_1$. (The argument for $\mu_{V,1}<0$ is identical.)
\beq
\lim_{t \to \infty} \mu_{V,t} = \tilde{\mu}_V, \quad 
\lim_{t \to \infty} \sigma_{V,t}^2 = \tilde{\sigma}_V^2, \qquad
\lim_{t \to \infty} \mu_{U,t} = \frac{\tilde \mu_V}{\sqrt{\delta} \tilde \beta}, \quad 
\lim_{t \to \infty} \sigma_{U,t}^2 = \frac{\tilde \sigma_V^2}{\delta \tilde \beta^2}.
\label{eq:musigmaUV_lims}
\eeq
Hence,  the desired result immediately follows from \eqref{eq:ut1ut_diff}, \eqref{eq:vt1vt_diff}  and Remark \ref{rem:tn_infty}, if we show that $\E\{W_{V,t+1} W_{V,t}\} \to 1$ and $\E\{W_{U,t-1} W_{U,t}\} \to 1$ as $t \to \infty$.

 Taking $r=(t-1)$ in \eqref{eq:WV_corr} and using the formula for $g_t(\cdot, \cdot)$ from \eqref{eq:ft_gt_choice}, we obtain 
\begin{equation}
\begin{split}
   \E \{ W_{V,t+1} W_{V,t} \} \sigma_{V, t+1} \sigma_{V,t} & = 
  \delta \E \left\{ Z^2 (\mu_{U,t} G + \sigma_{U,t} W_{U,t}) (\mu_{U,t-1} G + \sigma_{U,t-1} W_{U,t-1}) \right\} \\
   &   = \delta \left( \E\{ Z^2 G^2\} \mu_{U,t} \mu_{U,t-1} \, + \, \E\{Z^2\} 
   \E\{ W_{U,t} W_{U,t-1}\} \sigma_{U,t} \sigma_{U,t-1} \right).
\end{split}
\label{eq:WVt1t0}
\end{equation}
Next, taking $r=(t-1)$ in \eqref{eq:WU_corr} and using the formula for 
$f_t(\cdot)$ from \eqref{eq:ft_gt_choice}, we get
\beq
\begin{split}
&\E\{ W_{U,t} W_{U,t-1}\} \sigma_{U,t} \sigma_{U,t-1} \\
& = \frac{1}{\delta} \E\left\{ \left( \frac{\mu_{V,t} X + \sigma_{V,t} W_{V,t}}{\beta_t}  - X \sqrt{\delta} \mu_{U,t}\right) \left( \frac{\mu_{V,t-1} X + \sigma_{V,t-1} W_{V,t-1}}{\beta_{t-1}}  - X \sqrt{\delta} \mu_{U,t-1}\right) \right\}  \\
& =\frac{1}{\delta}  \frac{\sigma_{V,t} \, \sigma_{V,t-1}}{ \beta_t  \,\beta_{t-1}}\E\{ W_{V,t} W_{V,t-1}\}. 
\end{split}
\label{eq:WUtt1}
\eeq
Here, the last equality is obtained by noting from \eqref{eq:SE_lin1} that $\sqrt{\delta} \mu_{U,t} = \frac{\mu_{V,t}}{\beta_t}$, hence the coefficients on $X$ cancel.
Combining \eqref{eq:WVt1t0} and \eqref{eq:WUtt1}, we obtain
\beq
\begin{split}
\E \{ W_{V,t+1} W_{V,t} \} = \frac{\delta \E\{ Z^2 G^2\} \mu_{U,t} \, \mu_{U,t-1}}{\sigma_{V, t+1} \sigma_{V,t}} \, + \, \frac{E\{Z^2\} \sigma_{V,t} \, \sigma_{V,t-1}}{\beta_t \, \beta_{t-1} \, \sigma_{V, t+1} \sigma_{V,t}} \E\{ W_{V,t} W_{V,t-1}\}.
\end{split}
\label{eq:WVt1t}
\eeq

For brevity, we write this iteration as  $w_{t+1} =  a_t \,  + \, b_t w_t$, where
\beq
w_{t} = \E\{ W_{V,t} W_{V,t-1} \}, \qquad a_t = \frac{\delta \E\{ Z^2 G^2\} \mu_{U,t} \, \mu_{U,t-1}}{\sigma_{V, t+1} \sigma_{V,t}},
\qquad b_t=\frac{E\{Z^2\} \sigma_{V,t} \, \sigma_{V,t-1}}{\beta_t \, \beta_{t-1} \, \sigma_{V, t+1} \sigma_{V,t}}.
\eeq
The iteration is initialized with $w_1 = \E\{W_{V,1} W_{V,0}\}=0$.  Note that, as $t \to \infty$, the sequences $a_t$ and $b_t$ converge to well-defined limits determined by 
\eqref{eq:musigmaUV_lims}. By using the sub-additivity of $\limsup$, we have that 
\beq
\limsup_{t\to \infty} w_{t+1} = \limsup_{t \to \infty} \, (a_t  +  b_t w_t) \le \lim_{t \to \infty} a_t + \lim_{t \to \infty} b_t \limsup_{t\to \infty} w_t.
\label{eq:wt_limsup}
\eeq
Rearranging and using the limits from \eqref{eq:musigmaUV_lims}, we obtain
\beq \label{eq:limsuple}
\limsup_{t\to \infty} w_t \le \frac{\E\{Z^2 G^2 \}}{(\tilde \beta^2 - \E\{Z^2\})} \frac{\tilde \mu_V^2}{\tilde \sigma_V^2} = 1,
\eeq
where the last equality follows from \eqref{eq:FP1}. By using the super-additivity of $\liminf$, we also have that 
\beq
\liminf_{t\to \infty} w_{t+1} = \liminf_{t \to \infty} \, (a_t  +  b_t w_t) \ge \lim_{t \to \infty} a_t + \lim_{t \to \infty} b_t \liminf_{t\to \infty} w_t,
\label{eq:wt_liminf}
\eeq
which leads to 
\beq \label{eq:liminfge}
\liminf_{t\to \infty} w_t \ge \frac{\E\{Z^2 G^2 \}}{(\tilde \beta^2 - \E\{Z^2\})} \frac{\tilde \mu_V^2}{\tilde{\sigma}_V^2} = 1.
\eeq
By combining \eqref{eq:limsuple} and \eqref{eq:liminfge}, we conclude that $\lim_{t \to \infty} \E\{W_{V,t+1} W_{V,t}\} =1$. By using \eqref{eq:WUtt1} and \eqref{eq:musigmaUV_lims}, we also have that $\lim_{t \to \infty} \E\{W_{U,t-1} W_{U,t}\} =1$, which completes the proof.

\section{An Auxiliary Lemma} \label{sec:aux}

\begin{lemma}[Lemma 4.5 in \cite{dumbgen2011}]\label{lemma:dumbgen}
Let $(Q_t)_{t \geq 1}$ be a sequence of distributions converging weakly to some distribution $Q$, and let $h$ be a non-negative continuous function such that
\begin{equation}
\lim_{t \to \infty} \int h \, \de Q_t = \int h \, \de Q.    
\end{equation}
Then, for any continuous function $\psi$ such that $\abs{\psi}/(1 + h)$ is bounded,
\begin{equation}
 \lim_{t \to \infty} \int \psi \, \de Q_t = \int \psi \, \de Q .
\end{equation}
\end{lemma}

%%%%%%%%%%%%%%%%%%%%%%%%%%%%%%%%%%%%%%%%%%%%%%%

\end{document}